\documentclass[journal]{IEEEtran}
\usepackage{dblfloatfix}
\usepackage{hyperref}
\usepackage{amssymb}
\usepackage{amsmath}
\usepackage{amsthm}
\newtheorem{theorem}{Theorem}
\newtheorem{lemma}{Lemma}

\newtheorem{corollary}{Corollary}
\usepackage{enumerate}
\usepackage{enumitem}
\usepackage{changepage}
\usepackage{tabularx}
\usepackage{arydshln}
\usepackage{xcolor}
\usepackage{multicol}
\usepackage{booktabs}
\usepackage{ragged2e}
\usepackage{multirow}
\usepackage{graphicx}
\usepackage{pdflscape}
\usepackage{url}   
\usepackage{pifont}
\usepackage{pgfplots}
\usepackage{pgfplotstable}
\usepgfplotslibrary{groupplots}
\usepgfplotslibrary{statistics}
\pgfplotsset{compat=1.18}
\definecolor{FedAvgColor}{RGB}{31,119,180}
\definecolor{DittoColor}{RGB}{255,127,14}
\definecolor{FedDFColor}{RGB}{140,86,75}
\definecolor{FedDistillColor}{RGB}{44,160,44}
\definecolor{FedMDColor}{RGB}{148,103,189}
\definecolor{FedKDNASColor}{RGB}{214,39,40}
\definecolor{LocalKDColor}{RGB}{127,127,127}
\definecolor{darkblue}{rgb}{0.0,0.0,0.55}
\newcommand{\Bubble}[4]{%
  \addplot+[
    only marks,
    mark=*,
    mark size=#4,
    mark options={
      draw=#1,
      fill=#1,
      line width=0.8pt
    },
    forget plot
  ] coordinates {(#2,#3)};
}

\pgfplotsset{
  FedAvgStyle/.style     = {ybar, bar width=0.18, mark=none, fill=FedAvgColor,     draw=black, draw opacity=0.5},
  DittoStyle/.style      = {ybar, bar width=0.18, mark=none, fill=DittoColor,      draw=black, draw opacity=0.5},
  FedDFStyle/.style      = {ybar, bar width=0.18, mark=none, fill=FedDFColor,      draw=black, draw opacity=0.5},
  FedDistillStyle/.style = {ybar, bar width=0.18, mark=none, fill=FedDistillColor, draw=black, draw opacity=0.5},
  FedMDStyle/.style      = {ybar, bar width=0.18, mark=none, fill=FedMDColor,      draw=black, draw opacity=0.5},
  FedKDNASStyle/.style   = {ybar, bar width=0.18, mark=none, fill=FedKDNASColor,   draw=black, draw opacity=0.5,
                             nodes near coords, point meta=y,
                             every node near coord/.append style={
                               font=\tiny, rotate=90, anchor=west,
                               /pgf/number format/fixed,
                               /pgf/number format/precision=2}},
  LocalKDStyle/.style    = {ybar, bar width=0.18, mark=none, fill=LocalKDColor,    draw=black, draw opacity=0.5},
}

\newcommand{\AddThreeRegimePlotShift}[9]{%
  \pgfplotstablegetelem{#2}{#1}\of#3 \let\myValA\pgfplotsretval
  \pgfplotstablegetelem{#2}{#1}\of#4 \let\myValB\pgfplotsretval
  \pgfplotstablegetelem{#2}{#1}\of#5 \let\myValC\pgfplotsretval
  \addplot+[#6, mark=none, bar width=0.18] coordinates {
    (#7,\myValA)(#8,\myValB)(#9,\myValC)
  };%
}
\newcommand{\DrawRegimeBrackets}{%
  \draw[decorate, decoration={brace, amplitude=3pt, mirror}, thin]
    ([yshift=-8pt]axis cs:-0.75,0) -- ([yshift=-8pt]axis cs:0.75,0)
    node[midway, below=4pt, font=\scriptsize]{IID};
  \draw[decorate, decoration={brace, amplitude=3pt, mirror}, thin]
    ([yshift=-8pt]axis cs:1.85,0) -- ([yshift=-8pt]axis cs:3.35,0)
    node[midway, below=4pt, font=\scriptsize]{Dirichlet};
  \draw[decorate, decoration={brace, amplitude=3pt, mirror}, thin]
    ([yshift=-8pt]axis cs:4.05,0) -- ([yshift=-8pt]axis cs:5.55,0)
    node[midway, below=4pt, font=\scriptsize]{Shards};
}

\pgfplotstableread[col sep=comma]{figures/fedkdnasResults/metrics_iid_cifar10_mobilenet_v2.csv}\MobiIID
\pgfplotstableread[col sep=comma]{figures/fedkdnasResults/metrics_dirichlet_cifar10_mobilenet_v2.csv}\MobiDir
\pgfplotstableread[col sep=comma]{figures/fedkdnasResults/metrics_shards_cifar10_mobilenet_v2.csv}\MobiSha
\pgfplotstableread[col sep=comma]{figures/fedkdnasResults/metrics_iid_cifar10_shufflenet_v2_x0_5.csv}\ShufIID
\pgfplotstableread[col sep=comma]{figures/fedkdnasResults/metrics_dirichlet_cifar10_shufflenet_v2_x0_5.csv}\ShufDir
\pgfplotstableread[col sep=comma]{figures/fedkdnasResults/metrics_shards_cifar10_shufflenet_v2_x0_5.csv}\ShufSha

\pgfplotstableread[col sep=comma]{figures/fedkdnasResults/metrics_iid_cifar100_mobilenet_v2.csv}\MobiIIDCif
\pgfplotstableread[col sep=comma]{figures/fedkdnasResults/metrics_dirichlet_cifar100_mobilenet_v2.csv}\MobiDirCif
\pgfplotstableread[col sep=comma]{figures/fedkdnasResults/metrics_shards_cifar100_mobilenet_v2.csv}\MobiShaCif
\pgfplotstableread[col sep=comma]{figures/fedkdnasResults/metrics_iid_cifar100_shufflenet_v2_x0_5.csv}\ShufIIDCif
\pgfplotstableread[col sep=comma]{figures/fedkdnasResults/metrics_dirichlet_cifar100_shufflenet_v2_x0_5.csv}\ShufDirCif
\pgfplotstableread[col sep=comma]{figures/fedkdnasResults/metrics_shards_cifar100_shufflenet_v2_x0_5.csv}\ShufShaCif

\pgfplotstableread[col sep=comma]{figures/fedkdnasResults/metrics_iid_mnist_mlp.csv}\MobiIIDMn
\pgfplotstableread[col sep=comma]{figures/fedkdnasResults/metrics_dirichlet_mnist_mlp.csv}\MobiDirMn
\pgfplotstableread[col sep=comma]{figures/fedkdnasResults/metrics_shards_mnist_mlp.csv}\MobiShaMn
\pgfplotstableread[col sep=comma]{figures/fedkdnasResults/metrics_iid_mnist_simplecnn.csv}\ShufIIDMn
\pgfplotstableread[col sep=comma]{figures/fedkdnasResults/metrics_dirichlet_mnist_simplecnn.csv}\ShufDirMn
\pgfplotstableread[col sep=comma]{figures/fedkdnasResults/metrics_shards_mnist_simplecnn.csv}\ShufShaMn

\pgfplotstableread[col sep=comma]{figures/fedkdnasResults/metrics_iid_fmnist_mlp.csv}\MobiIIDFm
\pgfplotstableread[col sep=comma]{figures/fedkdnasResults/metrics_dirichlet_fmnist_mlp.csv}\MobiDirFm
\pgfplotstableread[col sep=comma]{figures/fedkdnasResults/metrics_shards_fmnist_mlp.csv}\MobiShaFm
\pgfplotstableread[col sep=comma]{figures/fedkdnasResults/metrics_iid_fmnist_simplecnn.csv}\ShufIIDFm
\pgfplotstableread[col sep=comma]{figures/fedkdnasResults/metrics_dirichlet_fmnist_simplecnn.csv}\ShufDirFm
\pgfplotstableread[col sep=comma]{figures/fedkdnasResults/metrics_shards_fmnist_simplecnn.csv}\ShufShaFm

\pgfplotstableread[col sep=comma]{figures/fedkdnasResults/metrics_iid_emnist_mlp.csv}\MobiIIDEm
\pgfplotstableread[col sep=comma]{figures/fedkdnasResults/metrics_dirichlet_emnist_mlp.csv}\MobiDirEm
\pgfplotstableread[col sep=comma]{figures/fedkdnasResults/metrics_shards_emnist_mlp.csv}\MobiShaEm
\pgfplotstableread[col sep=comma]{figures/fedkdnasResults/metrics_iid_emnist_simplecnn.csv}\ShufIIDEm
\pgfplotstableread[col sep=comma]{figures/fedkdnasResults/metrics_dirichlet_emnist_simplecnn.csv}\ShufDirEm
\pgfplotstableread[col sep=comma]{figures/fedkdnasResults/metrics_shards_emnist_simplecnn.csv}\ShufShaEm

\usepackage{pstricks}
\usepackage{pstricks-add}
\usepackage{tabularx}
\usepackage{booktabs}
\usepackage{algorithm}
\usepackage{algpseudocode}
\usepackage{comment}
\usepackage{makecell}
\usepackage{adjustbox}
\usepackage{tikz}
\usepackage[most]{tcolorbox}
\usepackage{tabularx} 
\usepackage{array}    
\usepackage{array, xcolor}
\usepackage{multirow}
\usepackage{caption}
\usepackage{subcaption}
\usepackage{pdflscape}
\usepackage{rotating}
\usepackage{adjustbox}
\usepackage{amsmath, amssymb, amsfonts}
\usepackage{float}
\usepackage{mathtools}           
\usepackage{bm}
\usepackage{todonotes}
\usepackage{microtype}
\usepackage{xspace}

\newcommand{\thickmidrule}{\specialrule{0.9pt}{1pt}{1pt}}

\usepackage[table]{xcolor}
\usepackage{dblfloatfix}
\usepackage{pdflscape}   
\usepackage{placeins}   
\usepackage{tikz}          
\usetikzlibrary{matrix}
\usetikzlibrary{decorations.pathreplacing,calc}
\usepackage{caption}

\definecolor{Top1}{HTML}{2E7D32}
\definecolor{Top2}{HTML}{66BB6A}
\definecolor{Top3}{HTML}{D9F2D9}
\definecolor{Bottom1}{HTML}{C62828}
\definecolor{Bottom2}{HTML}{EF9A9A}
\definecolor{Bottom3}{HTML}{F6D6D6}

\renewcommand{\arraystretch}{1.0}

\usepackage{ragged2e}
\newcolumntype{L}[1]{>{\RaggedRight\arraybackslash}p{#1}} 

\DeclareGraphicsExtensions{.pdf,.png,.jpg}
\newcolumntype{Y}{>{\raggedright\arraybackslash}X} 
\newcolumntype{C}{>{\Centering\arraybackslash}X}
\usepackage{etoolbox}

\makeatletter
\newcommand*{\algrule}[1][\algorithmicindent]{%
  \makebox[#1][l]{
    \hspace*{.2em}
    \vrule height .75\baselineskip depth .25\baselineskip
  }
}
\newcount\ALG@printindent@tempcnta
\def\ALG@printindent{%
    \ifnum \theALG@nested>0
    \ifx\ALG@text\ALG@x@notext
    \else
    \unskip
    \ALG@printindent@tempcnta=1
    \loop
    \algrule[\csname ALG@ind@\the\ALG@printindent@tempcnta\endcsname]%
    \advance \ALG@printindent@tempcnta 1
    \ifnum \ALG@printindent@tempcnta<\numexpr\theALG@nested+1\relax
    \repeat
    \fi
    \fi
}
\patchcmd{\ALG@doentity}{\noindent\hskip\ALG@tlm}{\ALG@printindent}{}{\errmessage{failed to patch}}
\patchcmd{\ALG@doentity}{\item[]\nointerlineskip}{}{}{}
\makeatother
\newlength\epaisLigne 
\definecolor{mygreen}{rgb}{0,0.5,0}
\definecolor{mygray}{rgb}{0.5,0.5,0.5}
\definecolor{mymauve}{rgb}{0.58,0,0.82}
\definecolor{Bitcoin}{RGB}{206, 62, 200}
\definecolor{EthPoW}{RGB}{6, 132, 44}
\definecolor{EthPoA}{RGB}{21, 191, 225}
\definecolor{Monero}{RGB}{243, 185, 31}
\definecolor{Hyperledger}{RGB}{127, 140, 140}
\definecolor{Quorum}{RGB}{49, 87, 206}
\definecolor{Multichain}{RGB}{0, 0, 0}
\definecolor{Core}{RGB}{244,181,180}
\definecolor{Ecosystem}{RGB}{152,155,248}
\definecolor{Operation}{RGB}{243,169,94}
\definecolor{Application}{RGB}{133,199,227}
\definecolor{Performance}{RGB}{197,217,193}
\definecolor{Nanayakkara}{RGB}{240,195,175}
\definecolor{Labazova}{RGB}{171,207,168}
\definecolor{Polge}{RGB}{111,144,182}
\definecolor{Grabe}{RGB}{120,120,120}
\definecolor{Six}{RGB}{237,177,32}
\definecolor{Yang2021}{RGB}{159, 224, 211}
\definecolor{Buyukozkan2021}{RGB}{194,157,227}
\definecolor{Nanayakkara}{RGB}{66,73,73}
\definecolor{Labazova}{RGB}{66,73,73}
\definecolor{Polge}{RGB}{66,73,73}
\definecolor{Grabe}{RGB}{66,73,73}
\definecolor{Six}{RGB}{66,73,73}
\definecolor{Yang2021}{RGB}{66,73,73}
\definecolor{Buyukozkan2021}{RGB}{66,73,73}
\definecolor{CMAPSSTCN}{RGB}{54,195,210}
\definecolor{CMAPSSTFT}{RGB}{227, 132, 198}
\definecolor{CWRUFCN}{RGB}{72,128,181}
\definecolor{CWRURESNET}{RGB}{239,143,63}
\definecolor{EcoatingTST}{RGB}{150,116,188}
\definecolor{EcoatingCNNLSTM}{RGB}{142,102,91}
\definecolor{HydraulicRESNET}{RGB}{92,166,76}
\definecolor{HydraulicLSTM}{RGB}{198,69,63}
\graphicspath{{./figures/}{.}{./../}}
\makeatletter   

\newcommand{\Rmnum}[1]{\expandafter\@slowromancap\romannumeral #1@}
\makeatother  
\makeatletter
\newif\if@borderstar
\def\bordermatrix{\@ifnextchar*{%
  \@borderstartrue\@bordermatrix@i}{\@borderstarfalse\@bordermatrix@i*}%
}
\def\@bordermatrix@i*{\@ifnextchar[{%
  \@bordermatrix@ii}{\@bordermatrix@ii[()]}
}
\def\@bordermatrix@ii[#1]#2{%
  \begingroup
    \m@th\@tempdima8.75\p@\setbox\z@\vbox{%
      \def\cr{\crcr\noalign{\kern 2\p@\global\let\cr\endline }}%
      \ialign {$##$\hfil\kern 2\p@\kern\@tempdima & \thinspace %
      \hfil $##$\hfil && \quad\hfil $##$\hfil\crcr\omit\strut %
      \hfil\crcr\noalign{\kern -\baselineskip}#2\crcr\omit %
      \strut\cr}}%
    \setbox\tw@\vbox{\unvcopy\z@\global\setbox\@ne\lastbox}%
    \setbox\tw@\hbox{\unhbox\@ne\unskip\global\setbox\@ne\lastbox}%
    \setbox\tw@\hbox{%
      $\kern\wd\@ne\kern -\@tempdima\left\@firstoftwo#1%
        \if@borderstar\kern2pt\else\kern -\wd\@ne\fi%
      \global\setbox\@ne\vbox{\box\@ne\if@borderstar\else\kern 2\p@\fi}%
      \vcenter{\if@borderstar\else\kern -\ht\@ne\fi%
        \unvbox\z@\kern-\if@borderstar2\fi\baselineskip}%
        \if@borderstar\kern-2\@tempdima\kern2\p@\else\,\fi\right\@secondoftwo#1 $%
    }\null \,\vbox{\kern\ht\@ne\box\tw@}%
  \endgroup
}
\makeatother
\makeatletter
\newcommand\footnoteref[1]{\protected@xdef\@thefnmark{\ref{#1}}\@footnotemark}
\makeatother
\hypersetup{colorlinks,breaklinks,
            linkcolor=darkblue,
            urlcolor=darkblue,
            anchorcolor=darkblue,
            citecolor=darkblue,
            pdfnewwindow=true}

\definecolor{colDir}{HTML}{1F77B4}  
\definecolor{colSha}{HTML}{FF7F0E} 
\tikzset{
  droplbl/.style={rotate=90, font=\fontsize{4}{5}\selectfont,
                  fill=white, inner sep=0.4pt, outer sep=0pt},
  iidlbl/.style={font=\fontsize{6}{7}\selectfont, above, inner sep=1pt},
}
\begin{document}

\title{Optimized Federated Knowledge Distillation with Distributed Neural Architecture Search}

\author{Chaimaa~Medjadji, Sylvain~Kubler, Yves~Le~Traon, Guilain~Leduc, Sadi~Alawadi, and Feras~M.~Awaysheh%
\thanks{Chaimaa Medjadji, Sylvain Kubler, Yves Le Traon, and Guilain Leduc are with the Interdisciplinary Centre for Security, Reliability and Trust (SnT), University of Luxembourg, Luxembourg. E-mail: \{chaimaa.medjadji, sylvain.kubler, yves.letraon, guilain.leduc\}@uni.lu.}%
\thanks{Sadi Alawadi is with Blekinge Institute of Technology, Sweden. E-mail: sadi.alawadi@bth.se.}%
\thanks{Feras M. Awaysheh is with ADSLabs, Umea University, Sweden. E-mail: feras.awaysheh@umu.se.}%
}

\maketitle

\begin{abstract}
Federated Learning (FL) enables collaborative model training without centralizing data. However, real-world deployments must simultaneously address statistical heterogeneity across client data (non-IID), system heterogeneity in device capabilities, and communication efficiency. Existing FL approaches mitigate these challenges through improved aggregation, personalization, or knowledge distillation, but they almost universally assume a fixed client architecture, limiting adaptability to heterogeneous data complexity and hardware constraints. This architectural constraint often leads to suboptimal trade-offs between accuracy and efficiency in real-world FL systems.
This work introduces FedKD-NAS, a distillation-driven FL framework that combines client-side neural architecture selection with distillation of server-coordinated knowledge. Each client autonomously selects a lightweight model under accuracy--resource constraints. It then trains it locally using a hybrid objective combining supervised learning and knowledge distillation, and shares only predictions on a public reference set. The server then aggregates and smooths these predictions, optionally combining them with a teacher model, to produce stable distillation targets for the next round.
Extensive evaluation on six datasets against six representative FL baselines (FedAvg, Ditto, FedMD, FedDF, FedDistill, Local-KD) demonstrates that FedKD-NAS consistently achieves superior Pareto efficiency, improving accuracy by up to 15\% under non-IID conditions, reducing client CPU usage by approximately 28\%, and decreasing communication overhead by up to $44\times$ while maintaining lightweight logit-based communication.
\end{abstract}


\section{Introduction}\label{sec:introduction}
Federated Learning (FL) is a foundational paradigm for distributed machine learning (ML) that enables clients to jointly train a shared model without centralizing data, addressing privacy and regulatory constraints in sensitive domains such as healthcare, finance, mobile systems and industrial IoT~\cite{mcmahan2017communication}. Despite its promise, real-world FL deployments face a threefold challenge and corresponding trade-off, between: (i) \emph{statistical heterogeneity}, as client data are typically non-IID due to differences in user behavior or sensing conditions; (ii) \emph{system heterogeneity}, as clients vary widely in compute, memory, energy, and availability; and (iii) \emph{communication cost}, which often dominates runtime and energy in bandwidth- or power-constrained settings~\cite{konevcny2016federated, sattler2019robust}. These factors are tightly coupled because non-IID data can induce client drift and unstable convergence~\cite{zhao2018federated, li2020fedprox, karimireddy2020scaffold}, while heterogeneous resources and limited communication constrain how much computation and coordination each client can afford.

To mitigate these effects, different methods have been introduced, including aggregation-based methods \cite{li2020fedprox,karimireddy2020scaffold,wang2020tackling} that modify the federated optimization procedure to improve robustness under non-IID data, personalized FL methods \cite{dinh2020pfedme,li2021ditto} that introduce client-specific components or objectives to better align with local distributions, and knowledge distillation and ensemble-based methods \cite{li2019fedmd,lin2020ensemble} that exchange prediction-level information rather than averaging parameters directly, improving stability under heterogeneity. Despite these advances, existing methods share a common implicit assumption that the client model architecture is fixed and independent of both data distribution and system constraints~\cite{mcmahan2017communication,li2020fedprox}. This assumption, inherited from centralized learning where architectures are selected offline and kept fixed, becomes problematic in federated settings with widely varying compute, memory, and data complexity across clients. A single model may be sufficient for some clients yet under-parameterized for others, while a design targeting worst-case heterogeneity can exceed the resource budgets of constrained devices. 
Existing approaches attempt to mitigate heterogeneity at the optimization level but ignore a more fundamental issue: the mismatch between model capacity and client-specific constraints, which cannot be resolved without adapting the model architecture itself~\cite{he2020fednas}. Recent works (e.g., \cite{he2020fednas, ilhan2022rafl, zhu2021fedgen, song2024feddistill}) integrate neural architecture search (NAS) or knowledge distillation (KD) into FL, but typically rely on supernets or parameter sharing that require joint architecture training across clients, adding significant compute overhead and limiting scalability in large federations.

We argue that \emph{model architecture should be a first-class optimization variable in FL}: a single static model is rarely optimal across heterogeneous clients, making architectural adaptability key to deployable FL. In this regard, resource-aware NAS studies \cite{tan2019mnasnet, wu2019fbnet, tan2019efficientnet} show that tailoring model capacity to hardware and task constraints can improve efficiency without sacrificing accuracy. 
To address these limitations, we propose FedKD-NAS, a FL framework that treats model architecture, rather than only model parameters, as a first-class optimization variable, by enabling lightweight client-side architecture adaptation coupled with knowledge distillation.

Each client selects an appropriate model from a constrained search space and aligns with the federation via distillation targets exchanged at the prediction level. Unlike supernet or parameter-sharing approaches, FedKD-NAS avoids weight sharing and global architecture synchronization, reducing overhead and improving scalability while mitigating representational mismatch and client drift. In summary, our contributions are:
\begin{itemize}
    \item 
    A decentralized NAS mechanism enabling clients to adapt architectures without global coordination;
    \item 
    A distillation-based collaboration protocol that eliminates parameter sharing of heterogeneous models;
    \item 
    A stabilized knowledge transfer mechanism that mitigates client drift under non-IID data;
    \item 
    A unified efficiency framework demonstrating consistent Pareto improvements across accuracy, communication, and resource usage state-of-the-art.

\end{itemize}

The paper is organized as follows: Section~\ref{Sec:SoTa} provides the background and reviews the literature. Section~\ref{sec:FedKD-NAS} introduces FedKD-NAS, while Section~\ref{sec:convergence} presents its convergence analysis. Section~\ref{sec:Exp} reports the experimental setup and results, followed by discussion~\ref{sec:discussion}, limitations of the framework~\ref{sec:limitations} and concluding remarks~\ref{sec:conclusion}.

\section{Background \& Related Work}\label{Sec:SoTa}
Section~\ref{Sec:SoTa:background} reviews the core concepts underpinning this work: federated optimization, model compression and knowledge distillation, neural architecture search, and federated distillation. Section~\ref{sec:SoTa:Literature} then surveys the related literature in FL through the lens of these topics.

\subsection{Background}\label{Sec:SoTa:background}
\subsubsection{Federated optimization}\label{Sec:SoTa:background:Fo}
We adopt the standard federated optimization setting with $K$ clients. Each client $k$ holds a local dataset $\mathcal{D}_k$ of size $n_k$, and the aggregate dataset size is $n=\sum_{k=1}^{K} n_k$. The global objective is to minimize the empirical loss on client $k$, denoted by $F(\mathbf{w})$:
\begin{equation}
\min_{\mathbf{w}}
\; F(\mathbf{w}) = \sum_{k=1}^{K} \frac{n_k}{n} F_k(\mathbf{w})
\end{equation}

In the widely adopted Federated Averaging (FedAvg) algorithm~\cite{mcmahan2017communication}, each client performs multiple steps of local stochastic gradient descent, after which the server aggregates model parameters as:
\begin{equation}
\mathbf{w}^{(r+1)}
= \sum_{k=1}^{K} \frac{n_k}{n} \mathbf{w}_k^{(r)}
\end{equation}

Under non-IID data, local updates can drift from the global optimum and hinder convergence~\cite{zhao2018federated, li2020fedprox, karimireddy2020scaffold}. Moreover, transmitting full model parameters is often prohibitive in resource-constrained settings~\cite{konevcny2016federated, sattler2019robust}. These issues motivate the compression- and architecture-aware extensions introduced next.

\subsubsection{Model compression \& KD}\label{Sec:SoTa:background:model_compression}
Model compression reduces compute, memory, and communication costs, especially in FL where models are repeatedly exchanged under heterogeneous client constraints. However, parameter-level methods such as quantization~\cite{jacob2018quantization} and pruning~\cite{deng2020model} are often architecture-specific and ill-suited to heterogeneous model sizes~\cite{gou2021knowledge}. KD \cite{hinton2015distilling} provides a flexible alternative to parameter compression by transferring knowledge from a teacher to a student through alignment of softened output distributions: given teacher logits $\mathbf{z}_T$ and student logits $\mathbf{z}_S$, the objective minimizes their divergence via the distillation loss, formalized by:
\begin{equation}
\label{eq:kdloss}
\mathcal{L}_{\text{KD}} =
\mathrm{KL}\left(
\sigma(\mathbf{z}_T / T)
\;\|\;
\sigma(\mathbf{z}_S / T)
\right)
\end{equation}

where $T$ denotes the temperature parameter and $\sigma(\cdot)$ denotes the softmax function.
Standard KD assumes a fixed student architecture chosen upfront; selecting it in a resource-aware way for each client motivates NAS, described next.

\subsubsection{Neural Architecture Search (NAS)}\label{Sec:SoTa:background:nas}
NAS automates the design of neural networks by searching a predefined space under joint performance and resource constraints. The generic NAS objective is defined as:
\begin{equation}
\min_{\alpha \in \mathcal{A}}
\; \mathcal{L}(\alpha)
+ \beta \cdot \mathrm{Cost}(\alpha)
\end{equation}

where $\alpha$ denotes a candidate architecture, $\mathcal{A}$ is the search space, and $\mathrm{Cost}(\cdot)$ encodes constraints such as latency, memory, or energy usage~\cite{tan2019mnasnet, wu2019fbnet, tan2019efficientnet}.

Resource-aware NAS shows that adapting architectures to deployment constraints can improve efficiency with minimal accuracy loss~\cite{tan2019mnasnet, wu2019fbnet, tan2019efficientnet}. In FL, it enables clients to choose capacities aligned with local resources and data complexity, but collaboration across heterogeneous architectures then requires prediction-level coordination, making federated distillation essential.

\subsubsection{Federated distillation}\label{Sec:SoTa:background:fd}
Federated distillation replaces weight aggregation with prediction exchange: clients compute logits on a shared public/proxy dataset and upload them to the server, which averages ${\mathbf{z}k}{k=1}^{K}$ to form a consensus target,
\begin{equation}
\mathbf{z}{\text{cons}}^{(r)}=\frac{1}{K}\sum{k=1}^{K}\mathbf{z}_k^{(r)} .
\end{equation}

The server broadcasts this consensus for distillation. Since communication depends only on the reference set size and number of classes (not model parameters), heterogeneous NAS-selected architectures can collaborate without shared weights, assuming an appropriate reference dataset~\cite{li2019fedmd,lin2020ensemble}.

These observations highlight a key gap: while FL, KD, and NAS individually address different aspects of heterogeneity, their joint integration remains underexplored, particularly in a scalable and resource-aware manner.

\begin{table}[t]
\centering
\footnotesize
\setlength{\tabcolsep}{4pt}
\begin{tabular}{lcccccccc}
\toprule
\textbf{Method} &
\textbf{Year} &
\textbf{Type} &
\rotatebox{90}{\textbf{Communication}} &
\rotatebox{90}{\textbf{Statistical heterogeneity}} &
\rotatebox{90}{\textbf{System heterogeneity}} &
\rotatebox{90}{\textbf{Public data}} &
\rotatebox{90}{\textbf{Teacher}} &
\rotatebox{90}{\textbf{Code availability}} \\
\midrule

FedAvg~\cite{mcmahan2017communication} & 2017 & PA & H & \ding{55} & \ding{55} & \ding{55} & \ding{55} & \ding{51} \\
FedProx~\cite{li2020fedprox} & 2020 & PA & H & \ding{51} & \ding{55} & \ding{55} & \ding{55} & \ding{51} \\
Ditto~\cite{li2021ditto} & 2021 & PF & H & \ding{51} & \ding{55} & \ding{55} & \ding{55} & \ding{51} \\
Local-KD~\cite{gou2021knowledge} & 2021 & KD & M & \ding{55} & \ding{51} & \ding{55} & \ding{51} & \ding{51} \\
FD~\cite{jeong2018communication} & 2018 & KD & L & \ding{55} & \ding{51} & \ding{55} & \ding{55} & \ding{55} \\
FedMD~\cite{li2019fedmd} & 2019 & KD & L & \ding{51} & \ding{51} & \ding{51} & \ding{55} & \ding{51} \\
Cronus~\cite{chang2019cronus} & 2019 & KD & L & \ding{51} & \ding{51} & \ding{51} & \ding{55} & \ding{55} \\
FedDF~\cite{lin2020ensemble} & 2020 & EM & L & \ding{55} & \ding{51} & \ding{51} & \ding{51} & \ding{51} \\
FedAUX~\cite{sattler2021fedaux} & 2021 & EM & L & \ding{51} & \ding{51} & \ding{51} & \ding{51} & \ding{55} \\
MHAT~\cite{hu2021mhat} & 2021 & KD & L & \ding{51} & \ding{51} & \ding{51} & \ding{51} & \ding{55} \\
FedGEMS~\cite{cheng2021fedgems} & 2021 & KD & L & \ding{51} & \ding{51} & \ding{51} & \ding{51} & \ding{55} \\
CFD~\cite{sattler2021cfd} & 2021 & KD & L & \ding{55} & \ding{51} & \ding{51} & \ding{55} & \ding{55} \\
FedGen~\cite{zhu2021fedgen} & 2021 & KD & L & \ding{51} & \ding{51} & \ding{55} & \ding{51} & \ding{51} \\
FedZKT~\cite{zhang2021fedzkt} & 2021 & KD & L & \ding{51} & \ding{51} & \ding{55} & \ding{51} & \ding{55} \\
FedFTG~\cite{zhang2022fedftg} & 2022 & KD & L & \ding{51} & \ding{51} & \ding{55} & \ding{51} & \ding{51} \\
DAFKD~\cite{wang2023dafkd} & 2023 & KD & L & \ding{51} & \ding{51} & \ding{55} & \ding{51} & \ding{55} \\
FedET~\cite{liu2023fedet} & 2023 & KD & L & \ding{51} & \ding{51} & \ding{55} & \ding{51} & \ding{55} \\
FedNAS~\cite{he2020fednas} & 2020 & NAS & H & \ding{51} & \ding{55} & \ding{55} & \ding{55} & \ding{51} \\
SPIDER~\cite{mushtaq2021spider} & 2021 & NAS & H & \ding{51} & \ding{55} & \ding{55} & \ding{55} & \ding{55} \\
RaFL~\cite{ilhan2022rafl} & 2022 & NAS+KD & L & \ding{51} & \ding{51} & \ding{51} & \ding{51} & \ding{55} \\
FedDistill~\cite{song2024feddistill} & 2024 & KD & L & \ding{51} & \ding{51} & \ding{51} & \ding{51} & \ding{51} \\
AdaptFL~\cite{adaptfl2024} & 2024 & NAS+KD & L & \ding{51} & \ding{51} & \ding{51} & \ding{51} & \ding{55} \\
HAPFNAS~\cite{yang2025hapfnas} & 2025 & NAS+KD & H & \ding{51} & \ding{51} & \ding{55} & \ding{55} & \ding{55} \\
\textbf{FedKD-NAS} & \textbf{2026} & \textbf{NAS+KD} & \textbf{L} & \textbf{\ding{51}} & \textbf{\ding{51}} & \textbf{\ding{51}} & \textbf{\ding{51}} & \textbf{\ding{51}} \\

\bottomrule
\end{tabular}
\caption{Comparison of representative FL methods across key dimensions. Column ``Type'' indicates whether the method relies on parameter aggregation (PA), KD, ensemble-based distillation (EM), or personalized FL (PF). Column ``Communication'' reflects the relative communication overhead induced by the method: Low (L), Medium (M), and High (H). Among the compared methods, \textbf{FedKD-NAS} simultaneously supports heterogeneous client architectures and heterogeneous data distributions while keeping communication overhead low through logit-based distillation and resource-aware client-side architecture selection.}
\label{tab:method_comparison_horizontal}
\end{table}

\subsection{Related Works}\label{sec:SoTa:Literature}
Our literature review focuses on the intersection of the four core concepts introduced in Sections~\ref{Sec:SoTa}. \tablename~\ref{tab:method_comparison_horizontal} summarizes representative FL frameworks categorized according to the following properties:
\begin{itemize}
\item \textbf{Type:} whether the proposed method relies on parameter aggregation (A), ensemble-based distillation (E), personalized FL (P), KD, NAS or combination of the two (NAS+KD);
\item \textbf{Communication:} the relative communication cost induced by the method (e.g., parameter exchange vs.\ logit sharing). A three-level cost scale is defined: (L) Low (less than 0.1× model size); (M) Medium (between 0.1× and 1× model size); (H) High (greater than 1× model size);
\item \textbf{Statistical heterogeneity:} whether the method explicitly addresses heterogeneous client data distributions;
\item \textbf{System heterogeneity:} whether heterogeneous client architectures or device capabilities are supported;
\item \textbf{Public data:} whether a public or proxy dataset is required for collaborative knowledge transfer;
\item \textbf{Teacher model:} whether a centralized or server-side teacher model is used during distillation (Yes, No, or Optional).
\end{itemize}

Early federated learning systems rely on parameter aggregation across participating clients, with Federated Averaging (FedAvg) serving as the canonical baseline for collaborative model training~\cite{mcmahan2017communication}. FedAvg aggregates locally trained models using weighted parameter averaging but often suffers from degraded convergence when client data distributions are heterogeneous or non-IID~\cite{zhao2018federated,kairouz2021advances}. Several extensions were proposed to improve robustness under heterogeneous conditions, such as FedProx (proximal regularization to limit local-global model divergence during training~\cite{li2020fedprox})~\cite{li2020fedprox}. Personalized FL addresses client variability by allowing each client to maintain a personalized model while still benefiting from global collaboration, as in Ditto, which jointly optimizes personalized and global models through regularization to improve robustness under heterogeneous client distributions~\cite{li2021ditto}. Although such methods improve optimization stability, they rely on exchanging full model parameters and assume homogeneous architectures across clients.

To reduce communication overhead and support heterogeneous models, a growing body of work replaces parameter aggregation with KD, allowing clients to exchange predictions instead of model parameters~\cite{hinton2015distilling,gou2021knowledge}. These FD methods decouple communication cost from model size and enable collaboration across heterogeneous architectures. Early approaches aggregate class-wise logits to produce global soft targets~\cite{jeong2018communication}, while FedMD enables heterogeneous clients to distill knowledge from averaged predictions on a shared public dataset~\cite{li2019fedmd}. Subsequent work explores iterative proxy-data distillation, collaborative co-distillation across rounds~\cite{chang2019cronus,anil2018large,gou2021knowledge}, and server-side ensemble distillation, as in FedDF~\cite{lin2020ensemble}. More recent methods improve robustness under heterogeneous data distributions, for example through lightweight class-wise logit exchange in FedDistill~\cite{song2024feddistill}, public data or stronger teacher guidance~\cite{sattler2021fedaux,hu2021mhat,cheng2021fedgems}, and communication-efficient compression schemes such as CFD~\cite{sattler2021cfd}. While many distillation-based methods rely on a shared public dataset, several studies explore data-free alternatives using generative models for knowledge transfer. FedGen introduces a server-side generator to synthesize features without public data~\cite{zhu2021fedgen}, and later works extend this idea to improve robustness under heterogeneous client distributions, including FedZKT, FedFTG, DAFKD, and FedET~\cite{zhang2021fedzkt,zhang2022fedftg,wang2023dafkd,liu2023fedet}. However, these methods often add training complexity and instability due to generative modeling.

Beyond communication and data challenges, FL must also address system heterogeneity across edge devices with widely varying compute and memory capacities~\cite{kairouz2021advances}. Resource-aware FL adapts training or model complexity to device capabilities, and KD is particularly suitable because it decouples model architecture from communication~\cite{gou2021knowledge,li2019fedmd}. However, many methods still rely on manually chosen architectures and do not explicitly consider device constraints. Recent work therefore combines NAS with FL to automate architecture adaptation in heterogeneous settings, including FedNAS and SPIDER, and more recent resource-aware distillation frameworks such as RaFL and AdaptFL~\cite{he2020fednas,mushtaq2021spider,ilhan2022rafl,adaptfl2024}. More recently, HAPFNAS~\cite{yang2025hapfnas} combines one-shot federated NAS with ensemble distillation to train a server-side supernet under joint resource and statistical heterogeneity.  However, it still relies on centralized supernet training and weight-based communication, inducing high per-round overhead and limiting scalability to large federations.

Despite this progress in FL over the past decade, many methods still rely on shared supernets or focus on predictive performance without explicitly considering deployment constraints, as summarized in \tablename~\ref{tab:method_comparison_horizontal}. In particular, the interplay between architecture selection, distillation target stability, and device-level resource budgets remains underexplored. To address this, we propose \textbf{FedKD-NAS}, a distillation-driven FL framework that combines lightweight local architecture selection under explicit resource constraints with server-coordinated KD. Unlike supernet-based methods such as RaFL and AdaptFL (see \tablename~\ref{tab:method_comparison_horizontal}), FedKD-NAS performs architecture selection locally while using a persistent server teacher to stabilize distillation targets across rounds~\cite{ilhan2022rafl,adaptfl2024}.

\begin{figure*}[t]
\centering
\includegraphics[width=0.9\textwidth]{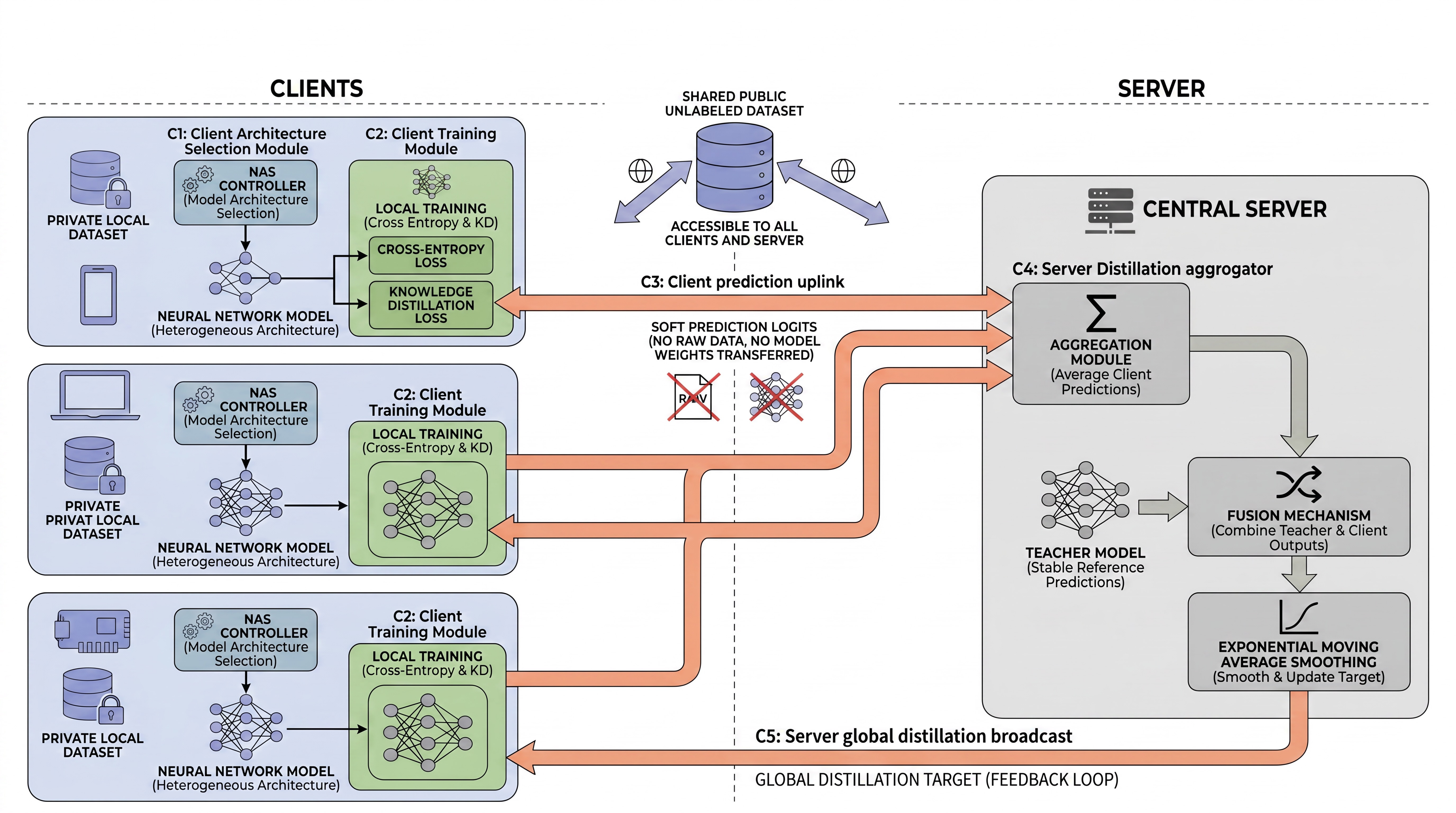}
\caption{Overview of the proposed FedKD-NAS architecture. At each communication round, each client independently selects a resource-adapted architecture via the NAS controller and trains its local student model using a combination of supervised loss and distillation from the global target. Clients transmit only soft predictions on the shared public dataset to the server. The server aggregates these predictions, fuses them with teacher guidance, and broadcasts a smoothed distillation target to all clients for the next round.}
\label{fig:NASKD}
\end{figure*}

\section{FedKD-NAS Framework}\label{sec:FedKD-NAS}
This section introduces the proposed \textbf{FedKD-NAS} framework. Section~\ref{sec:FedKD-NAS:Problem} formalizes the the system and problem. Section~\ref{sec:FedKD-NAS:Archi} describes the architecture and building blocks underlying FedKD-NAS.

\subsection{System scenario and Notation}\label{sec:FedKD-NAS:Problem}
Consider a FL system composed of a set of devices that collaboratively train a machine learning model. The system consists of a central server and $N$ participating clients, denoted by $\mathcal{C}=\{1,\dots,N\}$.  Each client $i \in \mathcal{C}$ owns a private local dataset $D_i^{\mathrm{priv}}=\{(x_k^i,y_k^i)\}_{k=1}^{n_i}$ drawn from a client-specific data distribution $\mathcal{P}_i$ which may vary from one client to another, inducing statistical heterogeneity (non-IID). Neither raw data nor model parameters are shared at any point during training.

In addition to private data, all participants have access to a small shared \emph{public reference dataset} $D^{\mathrm{pub}}=\{x_j\}_{j=1}^{M}$ of $M$ unlabeled samples drawn from the same or a related domain. This dataset serves exclusively as a common inference surface for prediction exchange and KD, and contains no sensitive information. It may be obtained from open data repositories, domain-related public sources, or synthetic generation.

Each client $i$ maintains a local student model $f_i(\cdot\,,\theta_i,\mathcal{A}_i)$, parameterized by weights $\theta_i$ and a neural architecture $\mathcal{A}_i$ drawn from a predefined search space $\mathcal{A}$. The server maintains a fixed teacher model $f_T(\cdot\,,\theta_T)$, pre-trained on $D^{\mathrm{pub}}$, which provides a stable global knowledge reference throughout training. Client distributions $\{\mathcal{P}_i\}$ are assumed to be heterogeneous in general, reflecting the non-IID conditions prevalent in real federated deployments. The federated procedure runs for $R$ communication rounds, indexed by $r \in \{1,\dots,R\}$.

\subsection{FedKD-NAS Architecture}\label{sec:FedKD-NAS:Archi}
FedKD-NAS targets two main failure modes in FL: \emph{(i) representation mismatch}, caused by resource heterogeneity, where a single shared architecture is either too large for constrained devices or too limited to capture complex local data; and \emph{(ii) client drift}, caused by non-IID data, where local objectives diverge across clients and lead to incompatible solutions that degrade global performance. To address these challenges simultaneously, FedKD-NAS integrates five main components, as illustrated in \figurename~\ref{fig:NASKD} and listed below:
\begin{enumerate}[label=C\arabic*]
    \item \textbf{Client Architecture Selection Module (NAS Controller)}: it selects a resource-adapted architecture from the search space $\mathcal{A}$ for each client, balancing predictive performance against local computational constraints;
    \item \textbf{Client Training Module}: it trains the selected student architecture locally using a hybrid objective that combines supervised cross-entropy on private data with KD from a globally shared target;
    \item \textbf{Client Prediction Uplink}: transmits each client's soft predictions on $D^{\mathrm{pub}}$ to the server, replacing costly weight transmission with a lightweight, architecture-independent communication primitive;
    \item \textbf{Server Distillation Aggregator}: it aggregates the received client predictions and fuses them with the teacher model's outputs to form a globally consistent knowledge signal;
    \item \textbf{Server Global Distillation Broadcast}: it applies temporal smoothing to the aggregated target and broadcasts it to all clients as the distillation supervision for the next round.
\end{enumerate}

The key design principle of FedKD-NAS is to decouple model capacity adaptation from global coordination, while maintaining alignment across heterogeneous clients through prediction-level knowledge transfer. This novel design enables each client to operate at its optimal capacity without introducing parameter incompatibility. The global federated procedure coordinating these components is summarized in Algorithm~\ref{alg:global}. The five components are described in Sections~\ref{sec:FedKD-NAS:Archi:NAS} to \ref{sec:FedKD-NAS:Archi:server_agg}.

\begin{algorithm}[t]
\caption{FedKD-NAS: Global Federated Procedure}
\label{alg:global}
\footnotesize
\begin{algorithmic}[1]
\Require Clients~$\mathcal{C}$, $D^{\mathrm{pub}}$, $R$, $\gamma$
\State Pre-train teacher model $f_T$ on $D^{\mathrm{pub}}$
\State Initialize global distillation target: $\tilde{\mathbf{Z}}^{(0)} \gets \{f_T(x_j)\}_{j=1}^{M}$ \Comment{Warm-start from teacher predictions}
\For{$r \gets 1$ \textbf{to} $R$}
    \State Broadcast $\tilde{\mathbf{Z}}^{(r-1)}$ to all clients \hfill \textit{(C5)}
    \ForAll{clients $i \in \mathcal{C}$ \textbf{in parallel}}
        \State Select architecture $\mathcal{A}_i^\star$ via NAS controller \hfill \textit{(C1, Algo.~\ref{alg:nas})}
        \State Train local student model using $\mathcal{L}_i^{(r)}$ \hfill \textit{(C2, Algo.~\ref{alg:client})}
        \State Compute and transmit predictions $\mathbf{P}_i^{(r)}$ on $D^{\mathrm{pub}}$ to server \hfill \textit{(C3)}
    \EndFor
    \State Aggregate predictions and update $\tilde{\mathbf{Z}}^{(r)}$ \hfill \textit{(C4, Algo.~\ref{alg:server})}
\EndFor
\end{algorithmic}
\end{algorithm}

\subsubsection{C1: Client Architecture Selection Module}\label{sec:FedKD-NAS:Archi:NAS}
A key enabler of FedKD-NAS is the ability of each client to operate with an architecture tailored to its local data complexity and available computational resources, rather than being constrained to a fixed global model. This is achieved through a lightweight, resource-aware NAS procedure executed independently by each client at every round.

\paragraph{Search space}
The search space $\mathcal{A}$ is constructed by varying the teacher model's structural parameters (incl., depth, width, kernel sizes, and expansion ratios) yielding a family of sub-networks that span a broad range of capacity and efficiency trade-offs. Each candidate architecture $a \in \mathcal{A}$ is associated with a resource profile
\[
    \mathrm{cost}(a) = \bigl(\mathrm{FLOPs}(a),\,\mathrm{Params}(a),\,\mathrm{Mem}(a),\,\mathrm{Latency}(a)\bigr).
\]
When direct hardware profiling is unavailable, these quantities are approximated from observed CPU usage and peak memory consumption during a few local training epochs.

\paragraph{Architecture selection}
Client~$i$ selects its architecture at round~$r$ by solving the following performance-efficiency trade-off:
\begin{equation}
\label{eq:arch_select}
a_i^{(r)} \in \arg\max_{a \in \mathcal{A}_i^{(r)}}
\Bigl(\widehat{\mathrm{Perf}}_i(a) - \lambda\,\widehat{\mathrm{Cost}}_i(a)\Bigr),
\end{equation}
where $\widehat{\mathrm{Perf}}_i(a)$ denotes a proxy validation accuracy obtained after brief training on a local held-out split, $\widehat{\mathrm{Cost}}_i(a)$ is a normalized resource consumption proxy, and $\lambda \geq 0$ controls the accuracy-efficiency trade-off. To keep selection overhead minimal, only a small random subset $\mathcal{A}_i^{(r)} \subset \mathcal{A}$ is evaluated per round. The full NAS procedure is detailed in Algorithm~\ref{alg:nas}.

\begin{algorithm}[t]
\caption{Resource-Aware NAS Controller}\label{alg:nas}
\footnotesize
\begin{algorithmic}[1]
\Require $\mathcal{S}$, $\mathcal{B}_i$, $D^{\mathrm{pub}}$, $\lambda$
\State Initialize candidate set $\mathcal{A} \subset \mathcal{S}$
\For{$t \gets 1$ \textbf{to} $T$}
    \State Sample a subset of candidate architectures $\{\mathcal{A}_k\} \subset \mathcal{A}$
    \ForAll{candidates $\mathcal{A}_k$}
        \State Train $\mathcal{A}_k$ briefly on $D^{\mathrm{pub}}$
        \State Evaluate proxy accuracy $\mathrm{Acc}(\mathcal{A}_k)$ and resource cost $\mathrm{Cost}(\mathcal{A}_k)$
    \EndFor
    \State $\mathcal{A}_i^\star \gets \arg\max_{\mathcal{A}_k}\bigl(\mathrm{Acc}(\mathcal{A}_k) - \lambda \cdot \mathrm{Cost}(\mathcal{A}_k)\bigr)$
    \State Refine search space $\mathcal{S}$ around $\mathcal{A}_i^\star$
\EndFor
\State \Return $\mathcal{A}_i^\star$
\end{algorithmic}
\end{algorithm}

\subsubsection{C2: Client Training Module}\label{sec:FedKD-NAS:Archi:client_train}
Once an architecture $\mathcal{A}_i^\star$ is selected, client $i$ trains its student model $f_i(\cdot\,,\theta_i,\mathcal{A}_i^\star)$ for $E$ local epochs using a hybrid objective that integrates both supervised learning and KD. Specifically, the local loss at round $r$ is defined as:
\begin{equation}
\label{eq:client_obj}
\mathcal{L}_i^{(r)} =
\underbrace{\mathcal{L}_{\mathrm{CE}}\!\left(f_i(x,\theta_i),\, y\right)}_{\text{supervised term}}
+ \,\alpha\,\underbrace{\mathcal{L}_{\mathrm{KD}}\!\left(f_i(x,\theta_i),\, \tilde{\mathbf{Z}}^{(r)}\right)}_{\text{distillation term}},
\end{equation}
where $\alpha \geq 0$ controls the relative weight of the distillation term, and $\tilde{\mathbf{Z}}^{(r)}$ is the global distillation target broadcast by the server at round~$r$. The distillation term is computed as the Kullback-Leibler divergence between temperature-scaled distributions:
\begin{equation}
\label{eq:kd_loss}
\mathcal{L}_{\mathrm{KD}} = \mathrm{KL}\!\left(
\mathrm{softmax}\!\left(\tfrac{\tilde{\mathbf{Z}}^{(r)}}{T}\right)
\,\middle\|\,
\mathrm{softmax}\!\left(\tfrac{f_i(x,\theta_i)}{T}\right)
\right),
\end{equation}
with temperature $T > 0$ controlling the softness of the target distribution. Higher values of $T$ produce softer probability distributions, which encode richer inter-class relational information beneficial for knowledge transfer. Each client solves:
\begin{equation}
\label{eq:client_objective}
\min_{\theta_i} \,\, \mathcal{L}_i^{(r)},
\end{equation}
using stochastic gradient descent (SGD) or Adam over $E$ local epochs. The distillation term operates at the prediction level, enabling functional alignment across clients with arbitrarily heterogeneous architectures, without requiring any constraint on parameter shapes or model sizes. This component is detailed in Algorithm~\ref{alg:client} 

\begin{algorithm}[t]
\caption{Client-Side Training with NAS+D}
\label{alg:client}
\footnotesize
\begin{algorithmic}[1]
\Require Client~$i$, $D_i^{\mathrm{priv}}$, $D^{\mathrm{pub}}$, $\tilde{\mathbf{Z}}^{(r)}$, $E$
\State Select architecture: $\mathcal{A}_i^\star \gets \textsc{NASController}(\mathcal{S}, \mathcal{B}_i, D^{\mathrm{pub}})$ \hfill (Algo.~\ref{alg:nas})
\State Initialize student model $f_i(\cdot\,,\theta_i, \mathcal{A}_i^\star)$
\For{$e \gets 1$ \textbf{to} $E$}
    \State Sample minibatch $(x, y) \sim D_i^{\mathrm{priv}}$
    \State Compute hybrid loss $\mathcal{L}_i^{(r)}$ using Eq.~(\ref{eq:client_obj}) and Eq.~(\ref{eq:kd_loss})
    \State Update $\theta_i$ via gradient descent on $\mathcal{L}_i^{(r)}$
\EndFor
\State Compute prediction logits on public dataset: $\mathbf{P}_i^{(r)} \gets \bigl\{f_i(x_j,\theta_i)\bigr\}_{j=1}^{M}$
\State \Return $\mathbf{P}_i^{(r)}$
\end{algorithmic}
\end{algorithm}

\subsubsection{C3: Client Prediction Uplink}\label{sec:FedKD-NAS:Archi:Comm}
Rather than transmitting model weights, each client communicates its knowledge to the server by uploading its soft prediction tensor on the shared public dataset. Formally, at round~$r$, client~$i$ sends:
\begin{equation}
\mathbf{P}_i^{(r)} = \bigl\{f_i(x_j,\theta_i^{(r)},\mathcal{A}_i)\bigr\}_{j=1}^{M} \in \mathbb{R}^{M \times C},
\end{equation}
where $C$ is the number of output classes. 

Transmitting float32 logits costs about $4MC$ bytes per client per round, far less than sending full model parameters for moderately sized networks. Crucially, this cost is independent of client architecture, making the protocol naturally compatible with the heterogeneous model sizes induced by the NAS controller..

\subsubsection{C4 \& C5: Server Distillation Aggregator \& Global Distillation Broadcast}\label{sec:FedKD-NAS:Archi:server_agg}
Upon receiving the prediction tensors from all clients, the server executes a two-stage process: it first constructs a globally consistent distillation target by aggregating client and teacher knowledge (C4), then broadcasts this target back to all clients (C5).

\paragraph{Client prediction aggregation}
The server computes an ensemble of client predictions by averaging the received logits:
\begin{equation}
\label{eq:aggregation}
\mathbf{P}_{\mathrm{agg}}^{(r)} = \frac{1}{N}\sum_{i=1}^{N}\mathbf{P}_i^{(r)}.
\end{equation}
When client predictions are expected to be noisy or adversarially corrupted, robust alternatives such as coordinate-wise median or trimmed mean may be substituted without altering the remainder of the pipeline.

\begin{algorithm*}[t]
\caption{Server-Side Knowledge Aggregation and Broadcast}
\label{alg:server}
\begin{algorithmic}[1]
\Require Client prediction tensors $\{\mathbf{P}_i^{(r)}\}_{i=1}^{N}$, teacher predictions $\mathbf{P}_T$, mixing coefficient $\beta^{(r)}$, EMA coefficient $\gamma$, previous target $\tilde{\mathbf{Z}}^{(r-1)}$
\State Aggregate client predictions: $\,\mathbf{P}_{\mathrm{agg}}^{(r)} \gets \tfrac{1}{N}\sum_{i=1}^{N}\mathbf{P}_i^{(r)}$
\State Fuse with teacher guidance: $\,\mathbf{Z}^{(r)} \gets \beta^{(r)}\,\mathbf{P}_T + \bigl(1-\beta^{(r)}\bigr)\,\mathbf{P}_{\mathrm{agg}}^{(r)}$
\State Apply EMA smoothing: $\,\tilde{\mathbf{Z}}^{(r)} \gets \gamma\,\tilde{\mathbf{Z}}^{(r-1)} + (1-\gamma)\,\mathbf{Z}^{(r)}$
\State Broadcast $\tilde{\mathbf{Z}}^{(r)}$ to all clients
\State \Return $\tilde{\mathbf{Z}}^{(r)}$
\end{algorithmic}
\end{algorithm*}

\paragraph{Teacher-client knowledge fusion}
To prevent the global target from drifting toward the potentially biased consensus of client predictions, the server interpolates between the aggregated client logits and the teacher model's reference predictions $\mathbf{P}_T = \{f_T(x_j,\theta_T)\}_{j=1}^{M}$:
\begin{equation}
\label{eq:target}
\mathbf{Z}^{(r)} = \beta^{(r)}\,\mathbf{P}_T + \bigl(1-\beta^{(r)}\bigr)\,\mathbf{P}_{\mathrm{agg}}^{(r)},
\end{equation}
where $\beta^{(r)} \in [0,1]$ controls the relative influence of the teacher at round $r$. This interpolation mechanism provides a stable knowledge anchor, particularly in early rounds when client models are still under-trained.

\paragraph{Temporal smoothing and broadcast}
To further suppress round-to-round fluctuations in the distillation target, the server applies an exponential moving average (EMA):
\begin{equation}
\label{eq:ema}
\tilde{\mathbf{Z}}^{(r)} = \gamma\,\tilde{\mathbf{Z}}^{(r-1)} + (1-\gamma)\,\mathbf{Z}^{(r)},
\end{equation}
where $\gamma \in [0,1]$ is the EMA decay coefficient. The smoothed target $\tilde{\mathbf{Z}}^{(r)}$ is then broadcast to all clients as the global distillation supervision for round $r+1$. The full server-side procedure is summarized in Algorithm~\ref{alg:server}.

Overall, the computational overhead of FedKD-NAS consists of lightweight local NAS exploration and standard local training, while the communication cost scales with the size of the public dataset rather than the model parameter count. This design supports scalability to large and heterogeneous federations by decoupling communication from model size. Accordingly, FedKD-NAS is guided by three key principles: (i) local autonomy through client-specific architecture adaptation, (ii) global alignment through prediction-level knowledge transfer, and (iii) stability through teacher-guided aggregation and temporal smoothing.

\section{FedKD-NAS Convergence Analysis}\label{sec:convergence}
\label{sec:convergence}
This section proves convergence guarantees for FedKD-NAS, which differs from classical FL in two ways that prevent a direct use of standard FedAvg analysis~\cite{mcmahan2017communication}: (i) clients exchange soft predictions on $D^{\mathrm{pub}}$ instead of model parameters, and (ii) clients may use different NAS-selected architecture. Therefore, there is no single shared parameter space and the objective changes across rounds.
We analyze FedKD-NAS as stochastic optimization over round-wise objectives ${\mathcal{L}_i^{(r)}}$ (Eq.~\ref{eq:client_obj}), driven by the time-varying distillation target $\tilde{\mathbf{Z}}^{(r)}$ built by the server through teacher-client fusion and EMA smoothing (Eq.~\eqref{eq:target}-\eqref{eq:ema}). 
The analysis proceeds in three parts: EMA-target stability (Lemma~\ref{lem:ema}), a round-averaged stationarity bound (Theorem~\ref{thm:main}), and convergence implications.

At each round~$r$, each client~$i$ executes $E$ local SGD steps with learning rate $\eta > 0$:
\begin{equation}
\label{eq:local_sgd}
\theta_{i,t+1}^{(r)} = \theta_{i,t}^{(r)} - \eta\, g_{i,t}^{(r)}, \qquad t = 0, \dots, E-1,
\end{equation}
where $g_{i,t}^{(r)}$ is the stochastic gradient of $\mathcal{L}_i^{(r)}$ (Eq.~\eqref{eq:client_obj}) computed from mini-batches drawn from $D_i^{\mathrm{priv}}$ and $D^{\mathrm{pub}}$. All notations and model definitions follow the ones introduced in section~\ref{sec:FedKD-NAS:Problem}.

\subsection{Assumptions}\label{subsec:assumptions}

We use standard nonconvex stochastic-optimization assumptions adapted to prediction-based federated distillation.

\begin{itemize}

     \item  \textbf{A1: $L$-Smoothness,} For each client $i$ and round $r$, the local objective $\mathcal{L}_i^{(r)}$ is differentiable and its gradient is $L$-Lipschitz continuous, as shown in Equation \ref{eq:a1}. This is a standard assumption used in previous studies ~\cite{mcmahan2017communication,li2019convergence,karimireddy2020scaffold}.
\begin{equation}
\label{eq:a1}
\|\nabla \mathcal{L}_i^{(r)}(\theta) - \nabla \mathcal{L}_i^{(r)}(\theta')\|
\le L \|\theta - \theta'\|, \quad \forall\, \theta,\theta'.
\end{equation}

    \item  \textbf{A2: Unbiased Stochastic Gradients with Bounded Variance,} As shown in Equation \ref{eq:a2}:
\begin{equation}
\label{eq:a2}
\mathbb{E}\big[g_{i,t}^{(r)} \mid \theta_{i,t}^{(r)}\big]
= \nabla \mathcal{L}_i^{(r)}(\theta_{i,t}^{(r)}),
\end{equation}
and its variance is uniformly bounded: there exists $\sigma^2 > 0$, as shown in Equation \ref{eq:a2_1}. The bounded variance conditions used based on several studies~\cite{ghadimi2013stochastic,reddi2021adaptive}.
\begin{equation}
\label{eq:a2_1}
\mathbb{E}\Big[\|g_{i,t}^{(r)} - \nabla \mathcal{L}_i^{(r)}(\theta_{i,t}^{(r)})\|^2 \,\Big|\, \theta_{i,t}^{(r)}\Big] \le \sigma^2.
\end{equation}

    \item  \textbf{A3: Lower-Bounded Objective}, there exists a finite constant $\mathcal{L}_{\inf} \in \mathbb{R}$ (see Equation \ref{eq:a3}), ensuring well-posedness of the optimization problem.
\begin{equation}
\label{eq:a3}
\mathcal{L}_i^{(r)}(\theta) \ge \mathcal{L}_{\inf} \quad \forall\, i, r, \theta,
\end{equation}

\item  \textbf{A4: Bounded Logit Targets on the Public Set}, there exists $B > 0$ such that for all rounds $r$,
\begin{equation}
\|\mathbf{Z}^{(r)}\|_F \le B, \qquad \|\tilde{\mathbf{Z}}^{(r)}\|_F \le B.
\end{equation}
This is natural in knowledge distillation, where logits are computed on a finite public dataset and temperature scaling ensures bounded network outputs~\cite{hinton2015distilling,anil2018large}.

\item  \textbf{A5: Lipschitz Dependence of KD Gradients on the Target}, there exists $L_Z > 0$ such that for all $\theta$ and any two targets $\mathbf{U}, \mathbf{V} \in \mathbb{R}^{M \times C}$,
\begin{equation}
\left\|\nabla_\theta \mathcal{L}_{\mathrm{KD}}(\theta, \mathbf{U}) - \nabla_\theta \mathcal{L}_{\mathrm{KD}}(\theta, \mathbf{V})\right\| \le L_Z \|\mathbf{U} - \mathbf{V}\|_F.
\end{equation}
This follows from the smoothness of the softmax function and the KL divergence under bounded logits and temperature scaling~\cite{hinton2015distilling,lin2020ensemble}.

\item  \textbf{A6: Controlled Architecture Switching}, We assume that changes in the selected architecture do not lead to unbounded perturbations in the gradients. Specifically, there exists a constant $\Delta_A \ge 0$, as shown in Equation \ref{eq:a6}.
\begin{equation}
\label{eq:a6}
\mathbb{E}\!\left[\left\|\nabla \mathcal{L}_i^{(r)}(\theta)\big|_{a_i^{(r)}} - \nabla \mathcal{L}_i^{(r)}(\theta)\big|_{a_i^{(r-1)}}\right\|^2\right] \le \Delta_A^2.
\end{equation}
This assumption is similar to the bounded client-drift condition in heterogeneous federated optimization settings~\cite{li2020federated, wang2020tackling}. In practice, the bound is typically small when the NAS controller operates within a restricted, resource-constrained architecture space and avoids large oscillations between consecutive communication rounds.

\end{itemize}

\subsection{The EMA Distillation Target Stability}
\label{subsec:ema_stability}

We analyze the temporal variability of the distillation target. To this end, we define the raw target drift as
$\Delta_Z^{(r)} := \|\mathbf{Z}^{(r)} - \mathbf{Z}^{(r-1)}\|_F$ and the corresponding smoothed target drift $\Delta_{\tilde{Z}}^{(r)} := \|\tilde{\mathbf{Z}}^{(r)} - \tilde{\mathbf{Z}}^{(r-1)}\|_F$.

\begin{lemma}\textbf{\textit{EMA Drift Bound}},\label{lem:ema}
Let $\tilde{\mathbf{Z}}^{(r)}$ be defined according to Eq.~\eqref{eq:ema},  with initialization $\tilde{\mathbf{Z}}^{(0)} = \mathbf{Z}^{(0)}$. Then
\begin{equation}
\Delta_{\tilde{Z}}^{(r)} = (1-\gamma)\|\mathbf{Z}^{(r)} - \tilde{\mathbf{Z}}^{(r-1)}\|_F \le (1-\gamma)\sum_{k=0}^{r-1} \gamma^k \Delta_Z^{(r-k)}.
\end{equation}
In particular, if $\Delta_Z^{(r)} \le \delta_Z$ for all $r$, then $\Delta_{\tilde{Z}}^{(r)} \le (1 - \gamma^r)\delta_Z \le \delta_Z$, so that instantaneous target drift is effectively low-pass filtered by the EMA coefficient $\gamma$.
\end{lemma}

\begin{proof}
See ~\ref{app:proof_lemma}.
\end{proof}

Lemma~\ref{lem:ema} formalizes the stabilizing role of EMA smoothing. In particular, high-frequency fluctuations in the raw target $\mathbf{Z}^{(r)}$ are attenuated in the smoothed target $\tilde{\mathbf{Z}}^{(r)}$, with stronger attenuation achieved as $\gamma$ increases.

\begin{lemma}\label{lem:multiclient}\textbf{\textit{Multi-Client Aggregation Error}},
Let $\mathbf{Z}^{(r)} = \frac{1}{K}\sum_{k=1}^{K} z_k^{(r)}$ denote the average logit across $K$ clients at round $r$, where $z_k^{(r)} \in \mathbb{R}^{M \times C}$ represents the client $k$ prediction matrix. Under Assumption~A4 (bounded logits), the mean-squared aggregation error decomposes as:
\begin{equation}
  \mathbb{E}\!\left[\left\|\mathbf{Z}^{(r)} - \mathbf{Z}^*\right\|_F^2\right]
  \,\leq\,
  \frac{\sigma_z^2}{K} + \kappa_z^2,
  \label{eq:agg_error}
\end{equation}
where $\mathbf{Z}^* = \mathbb{E}[\mathbf{Z}^{(r)}]$ is the expected
aggregation. Here, $\sigma_z^2 = \frac{1}{K}\sum_{k=1}^K
\mathbb{E}[\|z_k^{(r)} - \mathbb{E}[z_k^{(r)}]\|_F^2]$ is the mean
per-client logit variance, while $\kappa_z^2 = \frac{1}{K}\sum_{k=1}^K
\|\mathbb{E}[z_k^{(r)}] - \mathbf{Z}^*\|_F^2$ quantifies the inter-client bias arising from data heterogeneity

As a result, the smoothed aggregation error satisfies:
\begin{equation}
  \mathbb{E}\!\left[\left(\Delta_{\tilde{Z}}^{(r)}\right)^2\right]
  \,\leq\,
  (1-\gamma)^2\!\left(\frac{\sigma_z^2}{K} + \kappa_z^2\right),
  \label{eq:smoothed_agg}
\end{equation}
so that increasing $K$ reduces the variance component at a rate $\mathcal{O}(1/K)$, while $\kappa_z^2$ reflects irreducible inter-client heterogeneity independent of $K$.
\end{lemma}

\begin{proof}
See~\ref{proof_multi}.
\end{proof}

\subsection{Convergence to Stationary Points}
\label{subsec:main_theorem}

We have stated the main convergence guarantee for FedKD-NAS, expressed in terms of the expected squared gradient norm averaged over all local iterates.

\begin{theorem}\textbf{\textit{Convergence of FedKD-NAS}},
\label{thm:main}
Assume A1-A6 (See subsection \ref{subsec:assumptions}) and let $\eta \le 1/L$. Let $\{\theta_{i,t}^{(r)}\}$ be the sequence of local iterates generated by Eq.~\eqref{eq:local_sgd} over $R$ rounds of $E$ local steps each, for a total of $T = RE$ client updates. Then, for any fixed client $i$,
\begin{equation*}
\label{eq:main_bound}
\begin{split}
\frac{1}{T}\sum_{r=1}^{R}\sum_{t=0}^{E-1} \mathbb{E}\!\left[\left\|\nabla \mathcal{L}_i^{(r)}(\theta_{i,t}^{(r)})\right\|^2\right]
\,\le\, 
& \underbrace{\frac{2\bigl(\mathbb{E}[\mathcal{L}_i^{(1)}(\theta_{i,0}^{(1)})] - \mathcal{L}_{\inf}\bigr)}{\eta T}}_{\text{(a) optimization}} \\
& + \underbrace{L\eta\sigma^2}_{\text{(b) noise}}
+ \underbrace{\frac{2}{T}\sum_{r=2}^{R}\Delta_A^2}_{\text{(d) arch.\ switching}}
\\
& + \underbrace{2\alpha^2 L_Z^2 \cdot \frac{1}{T}\sum_{r=2}^{R} \mathbb{E}\!\left[(\Delta_{\tilde{Z}}^{(r)})^2\right]}_{\text{(c) target drift}}.
\end{split}
\end{equation*}

In particular, if (i) $\eta = \Theta(T^{-1/2})$, (ii) $\sup_r \mathbb{E}[(\Delta_{\tilde{Z}}^{(r)})^2] < \infty$, and (iii) $\sum_{r=2}^{R} \Delta_A^2 = o(T)$, then
\begin{equation}
\label{eq:rate}
\begin{split}
    \min_{1 \le r \le R,\, 0 \le t < E}\, \mathbb{E}\!\left[\left\|\nabla \mathcal{L}_i^{(r)}(\theta_{i,t}^{(r)})\right\|^2\right]
\,= 
& \, \mathcal{O}\!\left(T^{-1/2}\right)
+  \mathcal{O}\!\left(\frac{1}{T}\sum_{r=2}^{R} \Delta_A^2\right) \\
& +\mathcal{O}\!\left(\frac{1}{T}\sum_{r=2}^{R} \mathbb{E}\!\left[(\Delta_{\tilde{Z}}^{(r)})^2\right]\right).
\end{split}
\end{equation}

Hence, FedKD-NAS converges to a first-order stationary point up to additive terms that capture, respectively, distillation target drift~(c) and architecture switching~(d).
\end{theorem}

\begin{proof}
The proof proceeds in three steps, the full derivation is given in ~\ref{app:proof_main}.

\textbf{Step 1 (Per-step descent).}
By $L$-smoothness (A1) and the unbiasedness and bounded variance of $g_{i,t}^{(r)}$ (A2), a single SGD step yields the standard descent inequality:
\begin{equation}
\mathbb{E}\!\left[\mathcal{L}^{(r)}(\theta^+)\right]
\le \mathbb{E}\!\left[\mathcal{L}^{(r)}(\theta)\right]
- \frac{\eta}{2}\mathbb{E}\!\left[\|\nabla \mathcal{L}^{(r)}(\theta)\|^2\right]
+ \frac{L\eta^2}{2}\sigma^2,
\end{equation}
where we use $\eta \le 1/L$ to absorb the curvature term.

\begin{table*}[t]
  \centering
  \footnotesize
\caption{Evaluation metrics used to evaluate the benchmarked FL baselines.}
\label{Tab:Metrics}
  \begin{tabular}{l l l p{9.5cm}}
    \toprule
    \multicolumn{2}{l}{\textbf{Metric}} & \textbf{Equation} & \textbf{Variable description} \\
    \midrule
     $Acc$ & Accuracy &
      $\frac{1}{N}\sum_{i=1}^{N}\mathbf{1}\left(\hat{y}_i = y_i\right)$ &
      $N$ is the number of test samples, $y_i$ is the ground-truth label,
      $\hat{y}_i$ is the predicted label, $\mathbf{1}(\cdot)$ is the indicator
      function. \\
     $\mathcal{L}$ & Loss &
      $\frac{1}{N}\sum_{i=1}^{N}\ell\left(y_i,\hat{y}_i\right)$ &
      $\ell(\cdot)$ denotes the loss function (e.g., cross-entropy). \\
    $CPU$ & CPU usage &
      $\frac{1}{K}\sum_{k=1}^{K} u_k^{\text{CPU}}$ &
      $K$ is the number of clients, $u_k^{\text{CPU}}$ is the CPU usage of
      client $k$. \\
    $MEM$ & Client memory usage &
      $\frac{1}{K}\sum_{k=1}^{K} u_k^{\text{MEM}}$ &
      $u_k^{\text{MEM}}$ is the memory usage of client $k$. \\
    $COM$ & Communication volume & $\sum_{r=1}^{R}\sum_{k=1}^{K}\left(B_{k,r}^{\uparrow}+B_{k,r}^{\downarrow}\right)$ &
      Communication volume exchanged during training. \(R\): number of communication rounds, \(B_{k,r}^{\uparrow}\) and \(B_{k,r}^{\downarrow}\): bytes uploaded and downloaded by client \(k\) in round \(r\). \\
    $RES$ & Resource Efficiency &
      $\frac{1}{2}\frac{\text{CPU}}{\max(\text{CPU})}+
       \frac{1}{2}\frac{\text{MEM}}{\max(\text{MEM})}$ &
      - \\
    $CES$ & Communication Efficiency &
      $\frac{1}{\text{Comm}/\max(\text{Comm})}$ &
      - \\
    $PQS$ & Performance Quality &
      $0.7\cdot\text{Acc}_{\text{norm}}+0.3\cdot\frac{\min(\mathcal{L})}{\mathcal{L}}$ &
      $\text{Acc}_{\text{norm}}$ denotes min--max normalized accuracy.\\
    $UES$ & Unified Efficiency &
      $\text{UES}=\frac{\text{PQS}\cdot\text{CES}}{\text{RES}}$ & - \\
    \bottomrule
  \end{tabular}
\end{table*}

\textbf{Step 2 (Inter-round mismatch).}
Because $\mathcal{L}_i^{(r)}$ varies across rounds through both $\tilde{\mathbf{Z}}^{(r)}$ and $a_i^{(r)}$, the telescoping sum does not close directly. We bound the inter-round gradient mismatch by adding and subtracting the KD gradient at the previous target and architecture, then applying A5 and A6:
\begin{equation}
\|\nabla \mathcal{L}_i^{(r)}(\theta) - \nabla \mathcal{L}_i^{(r-1)}(\theta)\|^2
\le 2\alpha^2 L_Z^2 (\Delta_{\tilde{Z}}^{(r)})^2 + 2\Delta_A^2.
\end{equation}

\textbf{Step 3 (Telescoping and averaging).}
Summing the per-step descent inequalities over all $T$ updates, invoking the lower bound A3 to close the telescope, and dividing by $\eta T / 2$ yields the bound~\eqref{eq:main_bound}. 
Setting $\eta = \Theta(T^{-1/2})$ and applying the conditions on target drift and architecture switching then establishes the rate~\eqref{eq:rate}.
\end{proof}

\begin{corollary}[Multi-Client Convergence and Scalability]
\label{cor:multiclient}
Under the conditions of Theorem~\ref{thm:main} and
Lemma~\ref{lem:multiclient}, substituting bound~\eqref{eq:smoothed_agg}
into term~(c) of~\eqref{eq:main_bound} gives:
\begin{equation}
  \frac{1}{T}\sum_{r=2}^{R}
  \mathbb{E}\!\left[(\Delta_{\tilde{Z}}^{(r)})^2\right]
  \,\leq\,
  (1-\gamma)^2\!\left(\frac{\sigma_z^2}{K} + \kappa_z^2\right).
\end{equation}
The overall convergence rate therefore becomes:
\begin{equation}
\begin{split}
    \min_{r,t}\,\mathbb{E}\!\left[\|\nabla\mathcal{L}_i^{(r)}
  (\theta_{i,t}^{(r)})\|^2\right] = 
  & \mathcal{O}\!\left(T^{-1/2}\right)
  + \mathcal{O}\!\left(\frac{(1-\gamma)^2\sigma_z^2}{K}\right) \\
  & + \mathcal{O}\!\left((1-\gamma)^2\kappa_z^2\right)
  + \mathcal{O}\!\left(\frac{\Delta_A^2 R}{T}\right).
\end{split}
\end{equation}
Two scaling properties follow directly. First, the variance term $(1-\gamma)^2\sigma_z^2/K$ decreases as $\mathcal{O}(1/K)$, meaning that \emph{aggregating more clients improves teacher signal quality}, in contrast to FedAvg where increasing $K$ under non-IID data worsens gradient aggregation due to client drift. Second, the bias $\kappa_z^2$ is irreducible in $K$ and reflects genuine inter-client distribution heterogeneity, it is minimised by the NAS mechanism, which aligns local architectures to local distributions and thereby reduces the divergence between $\mathbb{E}[z_k^{(r)}]$ and $\mathbf{Z}^*$.
\end{corollary}

\section{Experimental Results}\label{sec:Exp}
This section is organized as follows. Section~\ref{sec:Exp:Setup} describes the experimental setup and evaluation metrics. Section~\ref{sec:Exp:ResultsOnAcademic} presents the results obtained on curated benchmark datasets, while Section~\ref{sec:Exp:ResultsOnReal} evaluates the methods on datasets that better reflect real-world application settings.

\subsection{Experimental Setting}\label{sec:Exp:Setup}
We conduct an empirical evaluation on six datasets: MNIST~\cite{mnist},FMNIST~\cite{fmnist}, EMNIST~\cite{emnist}, CASA~\cite{CASA}, CIFAR10 and CIFAR100~\cite{cifar100}, whose main characteristics are summarized in \tablename~\ref{Tab:Datasets}. Each dataset is evaluated under three partitioning settings: (i)~IID, (ii)~Dirichlet-based non-IID ($\alpha{=}0.1$), and (iii)~shard-based extreme label skew, enabling systematic robustness analysis as heterogeneity increases. To reflect realistic heterogeneous deployments, client models employ lightweight, resource-aware architectures adapted to each dataset: for MNIST, FMNIST, and EMNIST, clients use either a LeNet5 or ResNet18; for CIFAR10 and CIFAR100, they use MobileNetV2 and ShuffleNetV2-x0.5; and for CASA, a DeepConvLSTM model is used to capture temporal dependencies. On the server side, KD is guided by pre-trained teacher models selected among LeNet5, ResNet18, MobileNetV2, and ShuffleNetV2, depending on the dataset modality.This heterogeneous client–teacher configuration enables FedKD-NAS to combine NAS-driven client adaptation with expressive server-side supervision.

\begin{table}[t]
  \centering
  \footnotesize
  \setlength{\tabcolsep}{3pt}
  \caption{Summary of datasets used in our empirical study. \textbf{Type} denotes the data modality. \textbf{Target} specifies the prediction label. \textbf{Inst$^\dagger$} represents the total number of samples. \textbf{Difficulty} reflects relative learning complexity considering class count, data variability, and feature richness.}
  \label{Tab:Datasets}
  \resizebox{\columnwidth}{!}{
  \begin{tabular}{llllll}
    \toprule
    \textbf{Dataset} & \textbf{Type} & \textbf{Target} & \textbf{Inst$^\dagger$} & \textbf{Difficulty} \\
    \midrule
    MNIST     & Image (grayscale)    & Digit (0--9)         & 70\,000   & Low \\
    FMNIST    & Image (grayscale)    & Clothing category    & 70\,000   & Medium \\
    EMNIST    & Image (grayscale)    & Character (47 cls)   & 814\,255  & Medium-High \\
    CIFAR10   & Image (RGB)          & Object (10 cls)      & 60\,000   & Medium \\
    CIFAR100  & Image (RGB)          & Object (100 cls)     & 60\,000   & High \\
    CASA      & Time-series (sensor) & Activity type        & 13\,956\,534  & Very High \\
    \bottomrule
  \end{tabular}
  }
\end{table}

\begin{table*}[!htbp]
\centering
\caption{CIFAR10 results over 100~communication rounds across three data distributions (IID, Dirichlet, Shards): performance and resources utilisation metrics for MobileNetV2 and ShuffleNetV2.}
\label{tab:cifar10_combined_raw}
\scriptsize
\begin{adjustwidth}{-1.5cm}{-1.5cm}
\centering
\setlength{\tabcolsep}{2.5pt}   
\renewcommand{\arraystretch}{1.0} 
\resizebox{\textwidth}{!}{
\begin{tabular}{l*{15}{c}}
\toprule
\multicolumn{16}{c}{\textbf{MobileNetV2}} \\
\midrule
\multirow{2}{*}{Method}
& \multicolumn{5}{c}{IID}
& \multicolumn{5}{c}{Dirichlet}
& \multicolumn{5}{c}{Shards} \\
\cmidrule(lr){2-6}\cmidrule(lr){7-11}\cmidrule(lr){12-16}
& Acc$\uparrow$ & Loss$\downarrow$ & CPU$\downarrow$ & RAM$\downarrow$ & Comm$\downarrow$
& Acc$\uparrow$ & Loss$\downarrow$ & CPU$\downarrow$ & RAM$\downarrow$ & Comm$\downarrow$
& Acc$\uparrow$ & Loss$\downarrow$ & CPU$\downarrow$ & RAM$\downarrow$ & Comm$\downarrow$ \\
\midrule
FedAvg      & \textbf{0.7697} & 0.6566 & 47.0968 & 1077.1211 & 71\,573\,824 & 0.5134 & 1.4549 & 47.1893 & 1111.5518 & 71\,573\,824 & 0.5129 & 1.4409 & 47.1081 & 1080.5068 & 71\,573\,824 \\
Ditto       & 0.7568 & \textbf{0.6340} & 47.1282 & 1205.4883 & 71\,573\,824 & 0.3497 & 1.3631 & 47.1159 & 1232.3281 & 71\,573\,824 & 0.3214 & 1.4372 & 47.0089 & 1151.0547 & 71\,573\,824 \\
FedDF       & 0.4985 & 1.3816 & 47.2916 & 1110.8105 & 71\,978\,304 & 0.4440 & 1.6929 & 47.0609 & 1137.2842 & 71\,978\,304 & 0.4418 & 1.7477 & 47.0881 & 1113.6807 & 71\,978\,304 \\
FedDistill  & 0.6995 & 1.0890 & 46.9686 & 1054.0361 & \textbf{1\,617\,920} & 0.5994 & 1.4335 & 47.1102 & 1100.9980 & \textbf{1\,617\,920} & 0.6150 & 1.3891 & 47.1072 & 1072.4082 & \textbf{1\,617\,920} \\
FedMD       & 0.7187 & 0.8472 & 46.8021 & 1064.3936 & \textbf{1\,617\,920} & 0.5863 & 1.2504 & 47.1800 & 1192.2227 & \textbf{1\,617\,920} & 0.6218 & 1.1860 & 46.9156 & 1062.3096 & \textbf{1\,617\,920} \\
FedKD-NAS   & 0.6920 & 0.9774 & \textbf{33.9926} & \textbf{971.9180} & \textbf{1\,617\,920} & \textbf{0.6671} & \textbf{1.0714} & \textbf{35.2611} & \textbf{911.7109} & \textbf{1\,617\,920} & \textbf{0.6542} & \textbf{1.0793} & \textbf{34.3094} & \textbf{947.5234} & \textbf{1\,617\,920} \\
Local-KD    & 0.7026 & 0.9747 & 47.1696 & 1188.6553 & 71\,573\,824 & 0.2757 & 2.1734 & 47.2457 & 1356.5654 & 71\,573\,824 & 0.3309 & 1.9622 & 47.2867 & 1208.3408 & 71\,573\,824 \\
\midrule
\multicolumn{16}{c}{\textbf{ShuffleNetV2}} \\
\midrule
FedAvg      & 0.6788 & 0.9131 & 42.5618 & \textbf{860.3389} & 11\,265\,344 & 0.4961 & 1.7080 & 42.8577 & \textbf{863.2227} & 11\,265\,344 & 0.4460 & 1.6002 & 42.6718 & \textbf{862.4404} & 11\,265\,344 \\
Ditto       & 0.6527 & 0.8855 & 42.5863 & 895.1680 & 11\,265\,344 & 0.3375 & 1.5513 & 42.4086 & 897.3008 & 11\,265\,344 & 0.3034 & 1.7208 & 42.4416 & 895.4395 & 11\,265\,344 \\
FedDF       & 0.5096 & 1.3640 & 42.8215 & 865.1416 & 11\,669\,824 & 0.3825 & 1.9432 & 42.5055 & 866.5195 & 11\,669\,824 & 0.4514 & 1.7162 & 42.6431 & 864.9658 & 11\,669\,824 \\
FedDistill  & 0.6200 & 1.1413 & 42.9043 & 892.6270 & \textbf{1\,617\,920} & 0.5628 & 1.3477 & 42.7883 & 936.8057 & \textbf{1\,617\,920} & 0.5662 & 1.3501 & 42.7873 & 893.8965 & \textbf{1\,617\,920} \\
FedMD       & 0.6512 & 0.9985 & 42.5022 & 890.6963 & \textbf{1\,617\,920} & 0.5516 & 1.3485 & 42.8513 & 894.9854 & \textbf{1\,617\,920} & 0.5335 & 1.3703 & 42.7477 & 892.2266 & \textbf{1\,617\,920} \\
FedKD-NAS   & \textbf{0.7360} & \textbf{0.8536} & \textbf{33.6243} & 944.5820 & \textbf{1\,617\,920} & \textbf{0.6562} & \textbf{1.0426} & \textbf{34.9672} & 942.7852 & \textbf{1\,617\,920} & \textbf{0.6801} & \textbf{0.9888} & \textbf{35.0826} & 967.5703 & \textbf{1\,617\,920} \\
Local-KD    & 0.6359 & 1.1046 & 43.6111 & 1008.1592 & 11\,265\,344 & 0.3009 & 2.2856 & 43.6030 & 1040.6748 & 11\,265\,344 & 0.3418 & 2.0688 & 43.7823 & 1021.2637 & 11\,265\,344 \\
\bottomrule
\end{tabular}
}
\end{adjustwidth}
\end{table*}

Six FL baselines are compared to FedKD-NAS, namely FedAvg, Ditto, FedDF, FedDistill, FedMD, and Local-KD (see \tablename~\ref{tab:method_comparison_horizontal} for references). These baselines were selected to cover the main families of FL methods that are directly relevant to our setting. FedAvg is the most widely used baseline in FL, serving as a reference point for centralized-style aggregation. Ditto extends FedAvg with personalization and is commonly used to evaluate robustness under heterogeneity. FedDF represents server-side ensemble distillation, while FedDistill, FedMD, and Local-KD are logit-based knowledge distillation approaches that reduce communication by exchanging predictions instead of full model parameters. Among these, FedDistill and FedMD are widely regarded as strong baselines for communication-efficient FL. Importantly, these methods span both parameter-averaging and distillation-based paradigms, enabling a structured comparison between communication-heavy and communication-efficient approaches. Existing work combining neural architecture search (NAS) with FL is still limited and typically focuses on architecture optimization under standard aggregation, without integrating knowledge distillation. To the best of our knowledge, there are no established baselines that jointly combine NAS with logit-based distillation in a federated setting, which motivates the design of FedKD-NAS. Existing methods combining NAS with FL, such as RaFL, AdaptFL, and HAPFNAS~\cite{ilhan2022rafl,adaptfl2024,yang2025hapfnas}, are not included in the experimental comparison as none of 
these methods provide publicly available implementations. These frameworks involve complex multi-component pipelines (supernet construction, federated weight-sharing protocols, architecture evaluation strategies) whose design details significantly affect performance, making faithful re-implementation from paper descriptions alone unreliable and potentially unfair to the original methods. Our baseline selection instead covers the two principal FL paradigms: parameter aggregation (FedAvg, Ditto, Local-KD) and logit-level distillation (FedDistill, FedMD, FedDF); using established, reproducible implementations, enabling a controlled evaluation of FedKD-NAS's core contribution.

Nine metrics defined in \tablename~\ref{Tab:Metrics} are used to evaluate and compare the FL methods. A first set covers predictive performance ($Acc$, $\mathcal{L}$), communication overhead ($Comm$), and client-side resource consumption ($Mem$, $CPU$). A second set focuses on multi-objective efficiency, derived from raw measurements collected during training: \emph{(i)~Resource Efficiency Score (RES):} assesses client-side computational cost as a normalized combination of CPU and memory usage; lower RES indicates lower resource consumption. \emph{(ii)~Performance Quality Score (PQS):} jointly captures predictive accuracy and convergence quality via a weighted combination of min-max normalized accuracy ($\times 0.7$) and inverse normalized loss ($\times 0.3$). The weighting prioritizes accuracy as the primary evaluation criterion in classification tasks, while incorporating loss to reflect convergence behavior; higher PQS reflects better overall model quality.

\emph{(ii)~PQS:} jointly captures predictive accuracy and convergence quality via a weighted combination of min-max normalized accuracy ($\times 0.7$) and inverse normalized loss ($\times 0.3$). The weighting prioritizes accuracy as the primary evaluation criterion in classification tasks, while incorporating loss to reflect convergence behavior
\emph{(iii)~Communication Efficiency Score (CES):} measures the inverse normalized communication cost; a higher CES corresponds to fewer bytes communicated.
\emph{(iv)~Unified Efficiency Score (UES):} aggregates all three dimensions into a single scalar; a higher UES indicates superior performance relative to communication and computational cost. All system metrics reflect client-side training costs only, excluding teacher-side computation.

All reported results are averaged over 10 independent runs with different random seeds, low variance was observed across repetitions for all methods and metrics, confirming the stability of the findings.

\subsection{Experimental results on benchmark datasets}\label{sec:Exp:ResultsOnAcademic}
Sections~\ref{sec:Exp:Results:ACC_L}-\ref{sec:Exp:Results:resource} analyze the results obtained for the seven benchmarked FL methods, including FedKD-NAS. Given the large number of dataset-model-distribution combinations, this analysis focuses on CIFAR10, the most representative and widely used FL benchmark, reported in \tablename~\ref{tab:cifar10_combined_raw} and
\tablename~\ref{tab:all_combined_quality}. Raw metrics for MNIST, FMNIST, EMNIST, and CIFAR100 are reported in
\tablename~\ref{tab:all_combined_raw}; composite scores are in \tablename~\ref{tab:all_combined_quality}. While a summary across all datasets is provided and discussed in Section~\ref{sec:discussion}.

\begin{figure*}[t]
  \centering

  \begin{subfigure}[t]{0.98\textwidth}
    \centering
    \includegraphics[width=\textwidth]{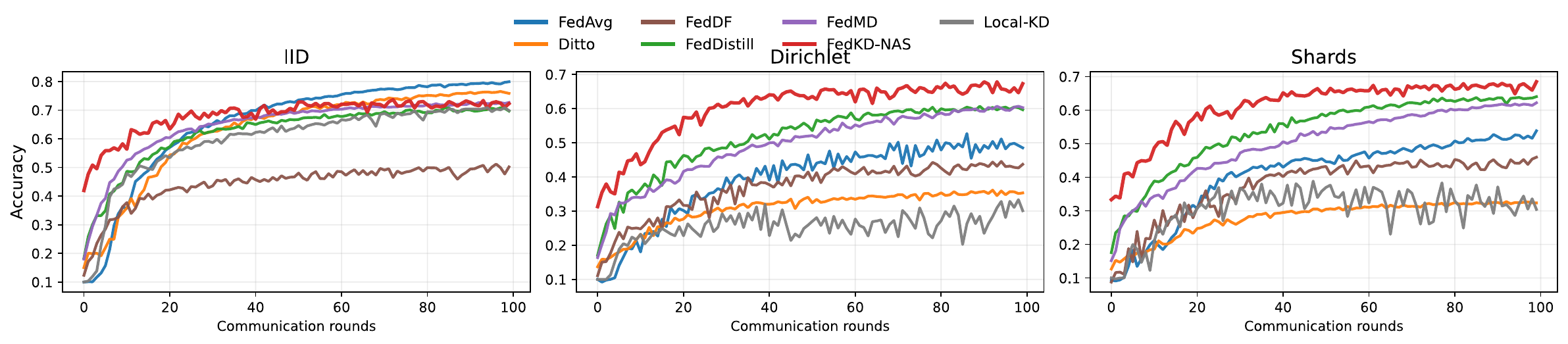}
    \caption{MobileNetV2}
  \end{subfigure}

  \vspace{0.5em}

 \begin{subfigure}[t]{0.98\textwidth}
    \centering
    \begin{adjustbox}{clip,trim=0 0 0 20pt}
      \includegraphics[width=\textwidth]{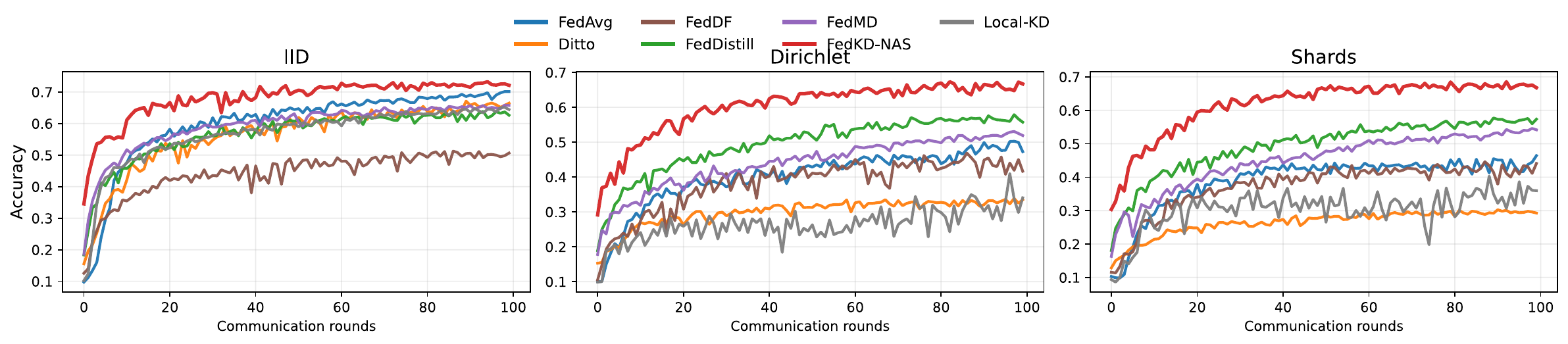}
    \end{adjustbox}
    \caption{ShuffleNetV2}
  \end{subfigure}

  \caption{Accuracy over 100 communication rounds on CIFAR10 under IID, Dirichlet ($\alpha=0.1$), and Shards partitioning. Top: MobileNetV2. Bottom: ShuffleNetV2. FedKD-NAS (red) converges faster than all competitors and maintains the widest accuracy gap under non-IID settings, where its advantage grows with the degree of heterogeneity.}
  \label{fig:cifar10_acc_curves}
\end{figure*}

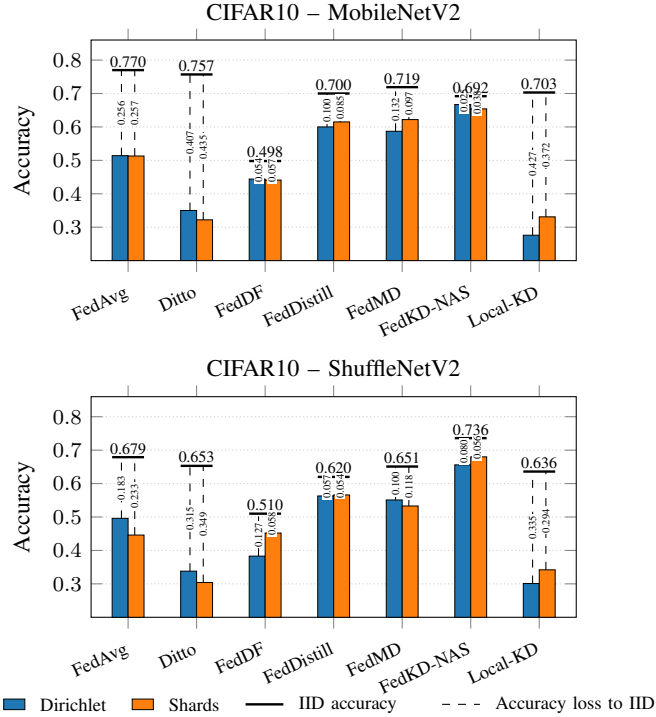
\begin{figure}[t]
\centering
\begin{tikzpicture}
\begin{groupplot}[
  group style={group size=1 by 2, vertical sep=1.8cm},
  symbolic x coords={FedAvg,Ditto,FedDF,FedDistill,FedMD,FedKD-NAS,Local-KD},
  xtick=data,
  width=8cm,
  height=4.5cm,
  ymin=0.20, ymax=0.86,
  ytick={0.3,0.4,0.5,0.6,0.7,0.8},
  ylabel={Accuracy},
  ylabel style={font=\small},
  yticklabel style={font=\footnotesize},
  xticklabel style={font=\scriptsize, rotate=25, anchor=north east},
  ymajorgrids=true,
  xmajorgrids=false,
  grid style={densely dotted, gray!40},
  ybar=0pt,
  every axis plot/.append style={
        bar width=6pt
    },
  enlarge x limits=0.09,
  title style={font=\small, yshift=-2pt},
]

\nextgroupplot[title={CIFAR10 {--} MobileNetV2}]

\addplot[fill=colDir, draw=black, line width=0.2pt] coordinates {
  (FedAvg,0.514) (Ditto,0.350) (FedDF,0.444)
  (FedDistill,0.600) (FedMD,0.587) (FedKD-NAS,0.667) (Local-KD,0.276)
};
\addplot[fill=colSha, draw=black, line width=0.2pt] coordinates {
  (FedAvg,0.513) (Ditto,0.322) (FedDF,0.441)
  (FedDistill,0.615) (FedMD,0.622) (FedKD-NAS,0.654) (Local-KD,0.331)
};

\draw[black,line width=0.9pt] ([xshift=-6pt]axis cs:FedAvg,0.770)--([xshift=6pt]axis cs:FedAvg,0.770);
\draw[black,line width=0.9pt] ([xshift=-6pt]axis cs:Ditto,0.757)--([xshift=6pt]axis cs:Ditto,0.757);
\draw[black,line width=0.9pt] ([xshift=-6pt]axis cs:FedDF,0.498)--([xshift=6pt]axis cs:FedDF,0.498);
\draw[black,line width=0.9pt] ([xshift=-6pt]axis cs:FedDistill,0.700)--([xshift=6pt]axis cs:FedDistill,0.700);
\draw[black,line width=0.9pt] ([xshift=-6pt]axis cs:FedMD,0.719)--([xshift=6pt]axis cs:FedMD,0.719);
\draw[black,line width=0.9pt] ([xshift=-6pt]axis cs:FedKD-NAS,0.692)--([xshift=6pt]axis cs:FedKD-NAS,0.692);
\draw[black,line width=0.9pt] ([xshift=-6pt]axis cs:Local-KD,0.703)--([xshift=6pt]axis cs:Local-KD,0.703);

\draw[black,thin,dashed] ([xshift=-2.5pt]axis cs:FedAvg,0.514)--([xshift=-2.5pt]axis cs:FedAvg,0.770);
\draw[black,thin,dashed] ([xshift=-2.5pt]axis cs:Ditto,0.350)--([xshift=-2.5pt]axis cs:Ditto,0.757);
\draw[black,thin,dashed] ([xshift=-2.5pt]axis cs:FedDF,0.444)--([xshift=-2.5pt]axis cs:FedDF,0.498);
\draw[black,thin,dashed] ([xshift=-2.5pt]axis cs:FedDistill,0.600)--([xshift=-2.5pt]axis cs:FedDistill,0.700);
\draw[black,thin,dashed] ([xshift=-2.5pt]axis cs:FedMD,0.587)--([xshift=-2.5pt]axis cs:FedMD,0.719);
\draw[black,thin,dashed] ([xshift=-2.5pt]axis cs:FedKD-NAS,0.667)--([xshift=-2.5pt]axis cs:FedKD-NAS,0.692);
\draw[black,thin,dashed] ([xshift=-2.5pt]axis cs:Local-KD,0.276)--([xshift=-2.5pt]axis cs:Local-KD,0.703);

\draw[black,thin,dashed] ([xshift=2.5pt]axis cs:FedAvg,0.513)--([xshift=2.5pt]axis cs:FedAvg,0.770);
\draw[black,thin,dashed] ([xshift=2.5pt]axis cs:Ditto,0.322)--([xshift=2.5pt]axis cs:Ditto,0.757);
\draw[black,thin,dashed] ([xshift=2.5pt]axis cs:FedDF,0.441)--([xshift=2.5pt]axis cs:FedDF,0.498);
\draw[black,thin,dashed] ([xshift=2.5pt]axis cs:FedDistill,0.615)--([xshift=2.5pt]axis cs:FedDistill,0.700);
\draw[black,thin,dashed] ([xshift=2.5pt]axis cs:FedMD,0.622)--([xshift=2.5pt]axis cs:FedMD,0.719);
\draw[black,thin,dashed] ([xshift=2.5pt]axis cs:FedKD-NAS,0.654)--([xshift=2.5pt]axis cs:FedKD-NAS,0.692);
\draw[black,thin,dashed] ([xshift=2.5pt]axis cs:Local-KD,0.331)--([xshift=2.5pt]axis cs:Local-KD,0.703);

\node[droplbl] at ([xshift=-2.5pt]axis cs:FedAvg,0.642)     {0.256};
\node[droplbl] at ([xshift=-2.5pt]axis cs:Ditto,0.554)      {0.407};
\node[droplbl] at ([xshift=-2.5pt]axis cs:FedDF,0.471)      {0.054};
\node[droplbl] at ([xshift=-2.5pt]axis cs:FedDistill,0.650) {0.100};
\node[droplbl] at ([xshift=-2.5pt]axis cs:FedMD,0.653)      {0.132};
\node[droplbl] at ([xshift=-2.5pt]axis cs:FedKD-NAS,0.680)  {0.025};
\node[droplbl] at ([xshift=-2.5pt]axis cs:Local-KD,0.490)   {0.427};

\node[droplbl] at ([xshift=2.5pt]axis cs:FedAvg,0.642)      {0.257};
\node[droplbl] at ([xshift=2.5pt]axis cs:Ditto,0.540)       {0.435};
\node[droplbl] at ([xshift=2.5pt]axis cs:FedDF,0.470)       {0.057};
\node[droplbl] at ([xshift=2.5pt]axis cs:FedDistill,0.658)  {0.085};
\node[droplbl] at ([xshift=2.5pt]axis cs:FedMD,0.671)       {0.097};
\node[droplbl] at ([xshift=2.5pt]axis cs:FedKD-NAS,0.673)   {0.038};
\node[droplbl] at ([xshift=2.5pt]axis cs:Local-KD,0.517)    {0.372};

\node[iidlbl] at (axis cs:FedAvg,0.770)     {0.770};
\node[iidlbl] at (axis cs:Ditto,0.757)      {0.757};
\node[iidlbl] at (axis cs:FedDF,0.498)      {0.498};
\node[iidlbl] at (axis cs:FedDistill,0.700) {0.700};
\node[iidlbl] at (axis cs:FedMD,0.719)      {0.719};
\node[iidlbl] at (axis cs:FedKD-NAS,0.692)  {0.692};
\node[iidlbl] at (axis cs:Local-KD,0.703)   {0.703};

\nextgroupplot[title={CIFAR10 {--} ShuffleNetV2}]

\addplot[fill=colDir, draw=black, line width=0.2pt] coordinates {
  (FedAvg,0.496) (Ditto,0.338) (FedDF,0.383)
  (FedDistill,0.563) (FedMD,0.551) (FedKD-NAS,0.656) (Local-KD,0.301)
};
\addplot[fill=colSha, draw=black, line width=0.2pt] coordinates {
  (FedAvg,0.446) (Ditto,0.304) (FedDF,0.452)
  (FedDistill,0.566) (FedMD,0.533) (FedKD-NAS,0.680) (Local-KD,0.342)
};

\draw[black,line width=0.9pt] ([xshift=-6pt]axis cs:FedAvg,0.679)--([xshift=6pt]axis cs:FedAvg,0.679);
\draw[black,line width=0.9pt] ([xshift=-6pt]axis cs:Ditto,0.653)--([xshift=6pt]axis cs:Ditto,0.653);
\draw[black,line width=0.9pt] ([xshift=-6pt]axis cs:FedDF,0.510)--([xshift=6pt]axis cs:FedDF,0.510);
\draw[black,line width=0.9pt] ([xshift=-6pt]axis cs:FedDistill,0.620)--([xshift=6pt]axis cs:FedDistill,0.620);
\draw[black,line width=0.9pt] ([xshift=-6pt]axis cs:FedMD,0.651)--([xshift=6pt]axis cs:FedMD,0.651);
\draw[black,line width=0.9pt] ([xshift=-6pt]axis cs:FedKD-NAS,0.736)--([xshift=6pt]axis cs:FedKD-NAS,0.736);
\draw[black,line width=0.9pt] ([xshift=-6pt]axis cs:Local-KD,0.636)--([xshift=6pt]axis cs:Local-KD,0.636);

\draw[black,thin,dashed] ([xshift=-2.5pt]axis cs:FedAvg,0.496)--([xshift=-2.5pt]axis cs:FedAvg,0.679);
\draw[black,thin,dashed] ([xshift=-2.5pt]axis cs:Ditto,0.338)--([xshift=-2.5pt]axis cs:Ditto,0.653);
\draw[black,thin,dashed] ([xshift=-2.5pt]axis cs:FedDF,0.383)--([xshift=-2.5pt]axis cs:FedDF,0.510);
\draw[black,thin,dashed] ([xshift=-2.5pt]axis cs:FedDistill,0.563)--([xshift=-2.5pt]axis cs:FedDistill,0.620);
\draw[black,thin,dashed] ([xshift=-2.5pt]axis cs:FedMD,0.551)--([xshift=-2.5pt]axis cs:FedMD,0.651);
\draw[black,thin,dashed] ([xshift=-2.5pt]axis cs:FedKD-NAS,0.656)--([xshift=-2.5pt]axis cs:FedKD-NAS,0.736);
\draw[black,thin,dashed] ([xshift=-2.5pt]axis cs:Local-KD,0.301)--([xshift=-2.5pt]axis cs:Local-KD,0.636);

\draw[black,thin,dashed] ([xshift=2.5pt]axis cs:FedAvg,0.446)--([xshift=2.5pt]axis cs:FedAvg,0.679);
\draw[black,thin,dashed] ([xshift=2.5pt]axis cs:Ditto,0.304)--([xshift=2.5pt]axis cs:Ditto,0.653);
\draw[black,thin,dashed] ([xshift=2.5pt]axis cs:FedDF,0.452)--([xshift=2.5pt]axis cs:FedDF,0.510);
\draw[black,thin,dashed] ([xshift=2.5pt]axis cs:FedDistill,0.566)--([xshift=2.5pt]axis cs:FedDistill,0.620);
\draw[black,thin,dashed] ([xshift=2.5pt]axis cs:FedMD,0.533)--([xshift=2.5pt]axis cs:FedMD,0.651);
\draw[black,thin,dashed] ([xshift=2.5pt]axis cs:FedKD-NAS,0.680)--([xshift=2.5pt]axis cs:FedKD-NAS,0.736);
\draw[black,thin,dashed] ([xshift=2.5pt]axis cs:Local-KD,0.342)--([xshift=2.5pt]axis cs:Local-KD,0.636);

\node[droplbl] at ([xshift=-2.5pt]axis cs:FedAvg,0.588)     {0.183};
\node[droplbl] at ([xshift=-2.5pt]axis cs:Ditto,0.496)      {0.315};
\node[droplbl] at ([xshift=-2.5pt]axis cs:FedDF,0.447)      {0.127};
\node[droplbl] at ([xshift=-2.5pt]axis cs:FedDistill,0.592) {0.057};
\node[droplbl] at ([xshift=-2.5pt]axis cs:FedMD,0.601)      {0.100};
\node[droplbl] at ([xshift=-2.5pt]axis cs:FedKD-NAS,0.696)  {0.080};
\node[droplbl] at ([xshift=-2.5pt]axis cs:Local-KD,0.469)   {0.335};

\node[droplbl] at ([xshift=2.5pt]axis cs:FedAvg,0.563)      {0.233};
\node[droplbl] at ([xshift=2.5pt]axis cs:Ditto,0.479)       {0.349};
\node[droplbl] at ([xshift=2.5pt]axis cs:FedDF,0.481)       {0.058};
\node[droplbl] at ([xshift=2.5pt]axis cs:FedDistill,0.593)  {0.054};
\node[droplbl] at ([xshift=2.5pt]axis cs:FedMD,0.592)       {0.118};
\node[droplbl] at ([xshift=2.5pt]axis cs:FedKD-NAS,0.708)   {0.056};
\node[droplbl] at ([xshift=2.5pt]axis cs:Local-KD,0.489)    {0.294};

\node[iidlbl] at (axis cs:FedAvg,0.679)     {0.679};
\node[iidlbl] at (axis cs:Ditto,0.653)      {0.653};
\node[iidlbl] at (axis cs:FedDF,0.510)      {0.510};
\node[iidlbl] at (axis cs:FedDistill,0.620) {0.620};
\node[iidlbl] at (axis cs:FedMD,0.651)      {0.651};
\node[iidlbl] at (axis cs:FedKD-NAS,0.736)  {0.736};
\node[iidlbl] at (axis cs:Local-KD,0.636)   {0.636};

\end{groupplot}

\node[anchor=north] at ([yshift=-0.8cm]current axis.south) {%
  \begin{tikzpicture}[baseline=0pt]
    \fill[colDir, draw=black, line width=0.2pt] (0,0) rectangle (0.28,0.14);
    \node[right, font=\scriptsize] at (0.36,0.07) {Dirichlet};
    \fill[colSha, draw=black, line width=0.2pt] (1.7,0) rectangle (1.98,0.14);
    \node[right, font=\scriptsize] at (2.06,0.07) {Shards};
    \draw[black, line width=0.9pt] (3.2,0.07)--(3.7,0.07);
    \node[right, font=\scriptsize] at (3.78,0.07) {IID accuracy};
    \draw[black, thin, dashed] (5.8,0.07)--(6.3,0.07);
    \node[right, font=\scriptsize] at (6.38,0.07) {Accuracy loss to IID};
  \end{tikzpicture}
};

\end{tikzpicture}
\caption{Accuracy drop from IID to non-IID settings on CIFAR10. FedKD-NAS shows the smallest degradation (${\le}5$\%); Ditto and Local-KD drop by more than 35--43\%, highlighting the benefit of NAS-based client adaptation under distribution shift.}
\label{fig:cifar10_acc_drop}
\end{figure}

\begin{figure}[!t]
\centering

\begin{tabular}{cccccccc}
  \raisebox{0.2ex}{\tikz\draw[fill=FedAvgColor,    draw=black](0,0)rectangle(0.18,0.18);} & \small FedAvg     &
  \raisebox{0.2ex}{\tikz\draw[fill=DittoColor,     draw=black](0,0)rectangle(0.18,0.18);} & \small Ditto      &
  \raisebox{0.2ex}{\tikz\draw[fill=FedDFColor,     draw=black](0,0)rectangle(0.18,0.18);} & \small FedDF      &
  \raisebox{0.2ex}{\tikz\draw[fill=FedDistillColor,draw=black](0,0)rectangle(0.18,0.18);} & \small FedDistill \\
  \raisebox{0.2ex}{\tikz\draw[fill=FedMDColor,     draw=black](0,0)rectangle(0.18,0.18);} & \small FedMD      &
  \raisebox{0.2ex}{\tikz\draw[fill=FedKDNASColor,  draw=black](0,0)rectangle(0.18,0.18);} & \small FedKD-NAS  &
  \raisebox{0.2ex}{\tikz\draw[fill=LocalKDColor,   draw=black](0,0)rectangle(0.18,0.18);} & \small Local-KD   \\
\end{tabular}

\begin{tikzpicture}
\begin{axis}[
    name=mainaxis,
    width  = 0.9\linewidth,
    height = 3.5cm,
    xmin=-1.4, xmax=4.8,
    xtick=\empty,
    clip=false,
    axis y line       = left,
    axis x line       = bottom,
    ymin=0, ymax=82000000,
    ylabel            = {Comm.\ Cost (bytes)},
    ylabel style      = {font=\scriptsize},
    yticklabel style  = {font=\scriptsize},
    enlarge x limits  = false,
    every axis plot/.append style={mark=none},
    bar width=0.18,
]
\addplot+[FedAvgStyle]     coordinates {(-0.75, 71573824)};
\addplot+[DittoStyle]      coordinates {(-0.50, 71573824)};
\addplot+[FedDFStyle]      coordinates {(-0.25, 71978304)};
\addplot+[FedDistillStyle] coordinates {( 0.00,  1617920)};
\addplot+[FedMDStyle]      coordinates {( 0.25,  1617920)};
\addplot+[FedKDNASStyle]   coordinates {( 0.50,  1617920)};
\addplot+[LocalKDStyle]    coordinates {( 0.75, 71573824)};
\draw[decorate, decoration={brace, amplitude=4pt, mirror}, thin]
  ([yshift=-10pt]axis cs:-0.75,0) -- ([yshift=-10pt]axis cs:0.75,0)
  node[midway, below=5pt, font=\scriptsize]{MobileNetV2};
\end{axis}

\begin{axis}[
    at={(mainaxis.south west)},
    anchor=south west,
    width  = 0.9\linewidth,
    height = 3.5cm,
    xmin=-1.4, xmax=4.8,
    xtick=\empty,
    clip=false,
    axis y line       = right,
    axis x line       = none,
    ymin=0, ymax=14000000,
    ylabel            = {Comm.\ Cost (bytes)},
    ylabel style      = {font=\scriptsize},
    yticklabel style  = {font=\scriptsize},
    enlarge x limits  = false,
    every axis plot/.append style={mark=none},
    bar width=0.18,
]
\addplot+[FedAvgStyle]     coordinates {(2.45, 11265344)};
\addplot+[DittoStyle]      coordinates {(2.70, 11265344)};
\addplot+[FedDFStyle]      coordinates {(2.95, 11669824)};
\addplot+[FedDistillStyle] coordinates {(3.20,  1617920)};
\addplot+[FedMDStyle]      coordinates {(3.45,  1617920)};
\addplot+[FedKDNASStyle]   coordinates {(3.70,  1617920)};
\addplot+[LocalKDStyle]    coordinates {(3.95, 11265344)};
\draw[decorate, decoration={brace, amplitude=4pt, mirror}, thin]
  ([yshift=-10pt]axis cs:2.45,0) -- ([yshift=-10pt]axis cs:3.95,0)
  node[midway, below=5pt, font=\scriptsize]{ShuffleNetV2};
\end{axis}
\end{tikzpicture}

\caption{Communication cost per round on CIFAR-10. Since this metric depends
only on model architecture and not on data distribution, values are identical
across IID, Dirichlet ($\alpha{=}0.1$), and Shards (one bar per method per
architecture is therefore shown). The left $y$-axis corresponds to MobileNetV2
and the right $y$-axis to ShuffleNetV2, as their communication budgets differ
by roughly $6\times$. FedKD-NAS shares the logit-only communication budget with
FedDistill and FedMD, achieving a ${\approx}44\times$ reduction over FedAvg on
MobileNetV2 and ${\approx}7\times$ on ShuffleNetV2.}
\label{fig:cifar10_comm}
\end{figure}
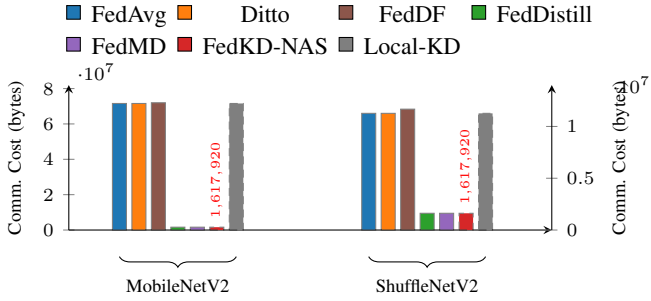

\subsubsection{Predictive Performance ($Acc$, $\mathcal{L}$)}
\label{sec:Exp:Results:ACC_L}

\tablename~\ref{tab:cifar10_combined_raw} presents the final metric evaluation values obtained with the seven FL methods over 100~communication rounds for MobileNetV2 and ShuffleNetV2 on CIFAR10 under IID, Dirichlet, and Shards partitioning. The results show that FedKD-NAS consistently outperforms all baselines under non-IID settings (i.e., Dirichlet, Shards), surpassing the best distillation baseline FedDistill by 7-10\% (MobileNetV2) and 9-11\% (ShuffleNetV2), and the median competitor by 18-19\% and 22-23\%, respectively; \figurename~\ref{fig:cifar10_acc_curves} further shows that FedKD-NAS converges faster and maintains a consistently wider accuracy gap over all baselines as data heterogeneity increases. Under IID, FedAvg leads on MobileNetV2 (0.7697 vs.\ 0.6920), but this advantage disappears entirely under
heterogeneous data settings. The same trend holds across all other benchmark datasets (see \tablename~\ref{tab:all_combined_raw}), with FedKD-NAS ranking first under non-IID partitioning in every case; on MNIST (LeNet5) it gains $+3$-$4\%$ over FedAvg under Dirichlet and Shards, on FMNIST (LeNet5) it outperforms the best non-NAS baseline by $+3$-$5\%$ under non-IID, and on EMNIST (LeNet5) it leads by $+5$-$7\%$ over FedAvg despite the 47-class complexity. On CIFAR100, FedAvg achieves the highest raw accuracy, but FedKD-NAS remains the strongest within the resource-efficient distillation family across all distributions. 

Regarding convergence, FedKD-NAS also reaches the lowest and most stable final loss under non-IID conditions on CIFAR10 (approx.\ 1.07-1.08 on MobileNetV2, vs.\ 1.4-1.75 for the best competing distillation methods), while FedDF and Local-KD remain unstable above 1.5-2.5 at round~100 and FedDistill exhibits sharp loss spikes under Dirichlet with ShuffleNetV2; this advantage in loss stability transfers uniformly to all smaller benchmarks, where FedKD-NAS consistently delivers the lowest non-IID loss within the distillation family (\tablename~\ref{tab:all_combined_raw}).


\begin{tcolorbox}[colback=gray!5!white,colframe=gray!50!black!50,title=Takeaway] 
FedKD-NAS outperforms all baselines under non-IID settings, with $+$3-11\% gains over the strongest distillation baselines. Under IID settings, it ranks 2\textsuperscript{nd} behind FedAvg, but the gap vanishes as heterogeneity increases.
\end{tcolorbox}

\definecolor{pastelblue}{RGB}{221,235,247}

\begin{table*}[p]
\centering
\caption{Composite quality results across CIFAR10, CIFAR100, FMNIST, EMNIST, and MNIST under three data distributions (IID, Dirichlet, Shards). Evaluation of MobileNetV2, ShuffleNetV2 for CIFAR datasets; LeNet5 and ResNet18 for the others, using RES, PQS, CES, and UES. RES is lower-better ($\downarrow$); PQS, CES, and UES are higher-better ($\uparrow$).}
\label{tab:all_combined_quality}
\setlength{\tabcolsep}{2.0pt}
\renewcommand{\arraystretch}{0.82}
\small
\begin{adjustbox}{max width=\linewidth, max totalheight=0.9\textheight, keepaspectratio}
\begin{tabular}{lll*{12}{c}}
\toprule
\multirow{2}{*}{Dataset}
& \multirow{2}{*}{Model}
& \multirow{2}{*}{Method}
& \multicolumn{4}{c}{IID}
& \multicolumn{4}{c}{Dirichlet}
& \multicolumn{4}{c}{Shards} \\
\cmidrule(lr){4-7}\cmidrule(lr){8-11}\cmidrule(lr){12-15}
& & 
& RES$\downarrow$ & PQS$\uparrow$ & CES$\uparrow$ & UES$\uparrow$
& RES$\downarrow$ & PQS$\uparrow$ & CES$\uparrow$ & UES$\uparrow$
& RES$\downarrow$ & PQS$\uparrow$ & CES$\uparrow$ & UES$\uparrow$ \\
\thickmidrule

\multirow{14}{*}{\textbf{CIFAR10}}
& \multirow{7}{*}{MobileNetV2}
& FedAvg
& 0.9447 & \textbf{0.9897} & 1.0057 & 1.0535
& 0.9091 & 0.6461 & 1.0057 & 0.7147
& 0.9452 & 0.6276 & 1.0057 & 0.6677 \\
& & Ditto
& 0.9983 & 0.9667 & 1.0057 & 0.9738
& 0.9528 & 0.3682 & 1.0057 & 0.3886
& 0.9734 & 0.2253 & 1.0057 & 0.2328 \\
& & FedDF
& 0.9607 & 0.1377 & 1.0000 & 0.1433
& 0.9172 & 0.4909 & 1.0000 & 0.5352
& 0.9587 & 0.4385 & 1.0000 & 0.4574 \\
& & FedDistill
& 0.9338 & 0.6933 & \textbf{44.4882} & 33.0328
& 0.9044 & 0.8032 & \textbf{44.4882} & 39.5127
& 0.9419 & 0.8508 & \textbf{44.4882} & 40.1862 \\
& & FedMD
& 0.9363 & 0.7928 & \textbf{44.4882} & 37.6695
& 0.9387 & 0.8126 & \textbf{44.4882} & 38.5099
& 0.9356 & 0.9050 & \textbf{44.4882} & 43.0314 \\
& & Local-KD
& 0.9917 & 0.7219 & 1.0057 & 0.7321
& 1.0000 & 0.1479 & 1.0057 & 0.1487
& 1.0000 & 0.1850 & 1.0057 & 0.1860 \\
& & FedKD-NAS
& \textbf{0.7625} & 0.6940 & \textbf{44.4882} & \textbf{40.4934}
& \textbf{0.7092} & \textbf{1.0000} & \textbf{44.4882} & \textbf{62.7298}
& \textbf{0.7549} & \textbf{1.0000} & \textbf{44.4882} & \textbf{58.9359} \\

\cmidrule(lr){2-15}

& \multirow{7}{*}{ShuffleNetV2}
& FedAvg
& 0.9147 & 0.8035 & 1.0359 & 0.9100
& 0.9062 & 0.5677 & 1.0359 & 0.6489
& 0.9096 & 0.4503 & 1.0359 & 0.5129 \\
& & Ditto
& 0.9322 & 0.7315 & 1.0359 & 0.8128
& 0.9174 & 0.2738 & 1.0359 & 0.3091
& 0.9231 & 0.1724 & 1.0359 & 0.1935 \\
& & FedDF
& 0.9200 & 0.1877 & 1.0000 & 0.2041
& 0.9037 & 0.3217 & 1.0000 & 0.3560
& 0.9105 & 0.4478 & 1.0000 & 0.4919 \\
& & FedDistill
& 0.9346 & 0.5657 & \textbf{7.2129} & 4.3659
& 0.9408 & 0.7481 & \textbf{7.2129} & 5.7358
& 0.9263 & 0.7080 & \textbf{7.2129} & 5.5132 \\
& & FedMD
& 0.9290 & 0.6941 & \textbf{7.2129} & 5.3890
& 0.9214 & 0.7259 & \textbf{7.2129} & 5.6826
& 0.9250 & 0.6440 & \textbf{7.2129} & 5.0217 \\
& & Local-KD
& 1.0000 & 0.6222 & 1.0359 & 0.6446
& 1.0000 & 0.1369 & 1.0359 & 0.1418
& 1.0000 & 0.2147 & 1.0359 & 0.2224 \\
& & FedKD-NAS
& \textbf{0.8540} & \textbf{1.0000} & \textbf{7.2129} & \textbf{8.4463}
& \textbf{0.8539} & \textbf{1.0000} & \textbf{7.2129} & \textbf{8.4466}
& \textbf{0.8744} & \textbf{1.0000} & \textbf{7.2129} & \textbf{8.2493} \\

\thickmidrule

\multirow{14}{*}{\textbf{CIFAR100}}
& \multirow{7}{*}{MobileNetV2}
& FedAvg
& 0.9503 & \textbf{1.0000} & 1.0537 & \textbf{1.1088}
& 0.8919 & \textbf{1.0000} & 1.0537 & \textbf{1.1814}
& 0.9376 & \textbf{0.9975} & 1.0537 & \textbf{1.1211} \\
& & Ditto
& 0.9802 & 0.8964 & 1.0537 & 0.9636
& 0.9418 & 0.5688 & 1.0537 & 0.6365
& 0.9896 & 0.4223 & 1.0537 & 0.4497 \\
& & FedDF
& 0.9651 & 0.1816 & 1.0000 & 0.1882
& 0.9031 & 0.1826 & 1.0000 & 0.2022
& 0.9411 & 0.1824 & 1.0000 & 0.1938 \\
& & FedDistill
& 0.9185 & 0.4204 & \textbf{4.9018} & 2.2434
& 0.8946 & 0.5019 & \textbf{4.9018} & 2.7502
& 0.9413 & 0.5310 & \textbf{4.9018} & 2.7651 \\
& & FedMD
& 0.9406 & 0.5676 & \textbf{4.9018} & 2.9579
& 0.8981 & 0.7018 & \textbf{4.9018} & 3.8307
& 0.9382 & 0.7534 & \textbf{4.9018} & 3.9363 \\
& & Local-KD
& 0.9992 & 0.7995 & 1.0537 & 0.8431
& 1.0000 & 0.6019 & 1.0537 & 0.6343
& 1.0000 & 0.7064 & 1.0537 & 0.7444 \\
& & FedKD-NAS
& \textbf{0.7568} & 0.7412 & \textbf{4.9018} & 4.8006
& \textbf{0.7023} & 0.7119 & \textbf{4.9018} & 4.9691
& \textbf{0.7527} & 0.7263 & \textbf{4.9018} & 4.7298 \\

\cmidrule(lr){2-15}

& \multirow{7}{*}{ShuffleNetV2}
& FedAvg
& 0.8949 & \textbf{0.9913} & \textbf{1.2845} & \textbf{1.4228}
& 0.8867 & \textbf{0.9659} & \textbf{1.2845} & \textbf{1.3991}
& 0.8997 & \textbf{1.0000} & \textbf{1.2845} & \textbf{1.4276} \\
& & Ditto
& 0.9202 & 0.8490 & \textbf{1.2845} & 1.1851
& 0.9097 & 0.4913 & \textbf{1.2845} & 0.6937
& 0.9223 & 0.3293 & \textbf{1.2845} & 0.4586 \\
& & FedDF
& 0.9054 & 0.1921 & 1.0000 & 0.2122
& 0.8939 & 0.2017 & 1.0000 & 0.2256
& 0.9044 & 0.2050 & 1.0000 & 0.2267 \\
& & FedDistill
& 0.9106 & 0.4156 & 1.1287 & 0.5152
& 0.9224 & 0.4180 & 1.1287 & 0.5115
& 0.9278 & 0.5088 & 1.1287 & 0.6191 \\
& & FedMD
& 0.9661 & 0.6009 & 1.1287 & 0.7021
& 0.9227 & 0.7181 & 1.1287 & 0.8784
& 0.9309 & 0.6932 & 1.1287 & 0.8405 \\
& & Local-KD
& 1.0000 & 0.8180 & \textbf{1.2845} & 1.0508
& 1.0000 & 0.6194 & \textbf{1.2845} & 0.7956
& 1.0000 & 0.7756 & \textbf{1.2845} & 0.9962 \\
& & FedKD-NAS
& \textbf{0.8452} & 0.7811 & 1.1287 & 1.0432
& \textbf{0.8468} & 0.9410 & 1.1287 & 1.2543
& \textbf{0.8582} & 0.9101 & 1.1287 & 1.1971 \\

\thickmidrule

\multirow{14}{*}{\textbf{FMNIST}}
& \multirow{7}{*}{LeNet5}
& FedAvg
& 0.8313 & 0.8208 & 1.0558 & 1.0426
& 0.8353 & 0.8209 & 1.0558 & 1.0377
& 0.8604 & 0.7660 & 1.0558 & 0.9400 \\
& & Ditto
& 0.8461 & 0.7244 & 1.0558 & 0.9039
& 0.8596 & 0.2782 & 1.0558 & 0.3417
& 0.8556 & 0.2322 & 1.0558 & 0.2866 \\
& & FedDF
& 0.8487 & 0.2289 & 1.0000 & 0.2697
& 0.8365 & 0.8327 & 1.0000 & 0.9954
& 0.8641 & 0.8742 & 1.0000 & 1.0116 \\
& & FedDistill
& 0.8154 & 0.4658 & \textbf{4.7268} & 2.7003
& \textbf{0.8115} & 0.9040 & \textbf{4.7268} & 5.2654
& \textbf{0.8402} & 0.9683 & \textbf{4.7268} & 5.4477 \\
& & FedMD
& \textbf{0.8106} & 0.7195 & \textbf{4.7268} & 4.1956
& 0.8188 & 0.9353 & \textbf{4.7268} & 5.3992
& 0.8406 & 0.8970 & \textbf{4.7268} & 5.0438 \\
& & Local-KD
& 1.0000 & 0.4341 & 1.0558 & 0.4583
& 1.0000 & 0.6462 & 1.0558 & 0.6823
& 1.0000 & 0.6506 & 1.0558 & 0.6870 \\
& & FedKD-NAS
& 0.9206 & \textbf{0.9796} & \textbf{4.7268} & \textbf{5.0295}
& 0.9257 & \textbf{0.9836} & \textbf{4.7268} & 5.0227
& 0.9154 & \textbf{0.9985} & \textbf{4.7268} & 5.1557 \\

\cmidrule(lr){2-15}

& \multirow{7}{*}{ResNet18}
& FedAvg
& \textbf{0.6903} & \textbf{0.9956} & 1.0357 & \textbf{1.4937}
& 0.8111 & 0.9166 & 1.0357 & 1.1703
& 0.7010 & 0.8040 & 1.0357 & 1.1879 \\
& & Ditto
& 0.7135 & 0.9012 & 1.0357 & 1.3081
& 0.8149 & 0.2792 & 1.0357 & 0.3548
& \textbf{0.6866} & 0.2443 & 1.0357 & 0.3685 \\
& & FedDF
& 0.7277 & 0.2124 & 1.0000 & 0.2919
& \textbf{0.8056} & 0.8053 & 1.0000 & 0.9997
& 0.7098 & 0.8716 & 1.0000 & 1.2280 \\
& & FedDistill
& 0.9661 & 0.7256 & \textbf{7.2587} & 5.4518
& 0.8118 & \textbf{0.9852} & \textbf{7.2587} & \textbf{8.8094}
& 0.8067 & \textbf{1.0000} & \textbf{7.2587} & 8.9977 \\
& & FedMD
& 0.8322 & 0.9499 & \textbf{7.2587} & \textbf{8.2856}
& 0.9727 & \textbf{1.0000} & \textbf{7.2587} & 7.4620
& 0.9579 & 0.9717 & \textbf{7.2587} & 7.3637 \\
& & Local-KD
& 0.8042 & 0.3655 & 1.0357 & 0.4707
& 0.9026 & 0.5175 & 1.0357 & 0.5938
& 0.7649 & 0.6970 & 1.0357 & 0.9437 \\
& & FedKD-NAS
& 0.7952 & 0.5931 & \textbf{7.2587} & 5.4142
& 0.8055 & 0.8857 & \textbf{7.2587} & 7.9815
& 0.7037 & 0.9255 & \textbf{7.2587} & \textbf{9.5467} \\

\thickmidrule

\multirow{14}{*}{\textbf{EMNIST}}
& \multirow{7}{*}{LeNet5}
& FedAvg
& 0.8174 & 0.9373 & \textbf{1.9095} & \textbf{2.1896}
& 0.8583 & 0.7623 & \textbf{1.9095} & \textbf{1.6958}
& 0.8459 & 0.9148 & \textbf{1.9095} & \textbf{2.0649} \\
& & Ditto
& 0.8218 & 0.7906 & \textbf{1.9095} & 1.8370
& 0.8615 & 0.2752 & \textbf{1.9095} & 0.6100
& 0.8375 & 0.2934 & \textbf{1.9095} & 0.6689 \\
& & FedDF
& 0.8280 & 0.2080 & 1.2925 & 0.3247
& 0.8595 & 0.5692 & 1.2925 & 0.8560
& 0.8640 & 0.7934 & 1.2925 & 1.1869 \\
& & FedDistill
& \textbf{0.8075} & 0.5541 & 1.0000 & 0.6862
& \textbf{0.8416} & 0.7266 & 1.0000 & 0.8634
& \textbf{0.8261} & 0.8706 & 1.0000 & 1.0539 \\
& & FedMD
& 0.8116 & 0.7801 & 1.0000 & 0.9612
& 0.8668 & 0.8408 & 1.0000 & 0.9700
& 0.8305 & 0.9302 & 1.0000 & 1.1201 \\
& & Local-KD
& 1.0000 & 0.5578 & \textbf{1.9095} & 1.0650
& 1.0000 & 0.4731 & \textbf{1.9095} & 0.9033
& 1.0000 & 0.7432 & \textbf{1.9095} & 1.4191 \\
& & FedKD-NAS
& 0.9040 & \textbf{0.9836} & 1.0000 & 1.0881
& 0.9450 & \textbf{1.0000} & 1.0000 & 1.0582
& 0.9239 & \textbf{1.0000} & 1.0000 & 1.0823 \\

\cmidrule(lr){2-15}

& \multirow{7}{*}{ResNet18}
& FedAvg
& 0.7638 & \textbf{1.0000} & \textbf{1.3121} & \textbf{1.7180}
& \textbf{0.8046} & \textbf{0.9640} & \textbf{1.3121} & \textbf{1.5721}
& \textbf{0.5784} & 0.9531 & \textbf{1.3121} & \textbf{2.1622} \\
& & Ditto
& \textbf{0.7617} & 0.6808 & \textbf{1.3121} & 1.1728
& 0.8048 & 0.2961 & \textbf{1.3121} & 0.4827
& 0.5846 & 0.2946 & \textbf{1.3121} & 0.6611 \\
& & FedDF
& 0.7837 & 0.4802 & 1.0000 & 0.6127
& 0.8051 & 0.9140 & 1.0000 & 1.1352
& 0.5928 & \textbf{0.9956} & 1.0000 & 1.6793 \\
& & FedDistill
& 0.8121 & 0.1762 & 1.0509 & 0.2280
& 0.9697 & 0.8468 & 1.0509 & 0.9178
& 0.9645 & 0.9207 & 1.0509 & 1.0032 \\
& & FedMD
& 0.8871 & 0.4788 & 1.0509 & 0.5672
& 0.8120 & 0.9537 & 1.0509 & 1.2343
& 0.8650 & 0.9802 & 1.0509 & 1.1908 \\
& & Local-KD
& 0.8811 & 0.6005 & \textbf{1.3121} & 0.8942
& 0.8969 & 0.4729 & \textbf{1.3121} & 0.6919
& 0.6319 & 0.7897 & \textbf{1.3121} & 1.6397 \\
& & FedKD-NAS
& 0.9720 & 0.4135 & 1.0509 & 0.4471
& 0.8757 & 0.9442 & 1.0509 & 1.1332
& \textbf{0.5781} & 0.9293 & 1.0509 & 1.6891 \\

\thickmidrule

\multirow{14}{*}{\textbf{MNIST}}
& \multirow{7}{*}{LeNet5}
& FedAvg
& 0.8315 & 0.5852 & 1.0558 & 0.7431
& 0.8587 & 0.6577 & 1.0558 & 0.8087
& 0.8543 & 0.6683 & 1.0558 & 0.8259 \\
& & Ditto
& 0.8440 & 0.5766 & 1.0558 & 0.7213
& 0.8648 & 0.0797 & 1.0558 & 0.0973
& 0.8676 & 0.0722 & 1.0558 & 0.0878 \\
& & FedDF
& 0.8529 & 0.0504 & 1.0000 & 0.0591
& 0.8618 & 0.4991 & 1.0000 & 0.5792
& 0.8622 & 0.5280 & 1.0000 & 0.6124 \\
& & FedDistill
& \textbf{0.8181} & 0.5086 & \textbf{4.7268} & 2.9385
& \textbf{0.8291} & 0.6787 & \textbf{4.7268} & 3.8692
& \textbf{0.8164} & 0.7152 & \textbf{4.7268} & 4.1406 \\
& & FedMD
& 0.8212 & 0.5122 & \textbf{4.7268} & 2.9484
& 0.8324 & 0.6601 & \textbf{4.7268} & 3.7482
& 0.8305 & 0.6923 & \textbf{4.7268} & 3.9400 \\
& & Local-KD
& 1.0000 & 0.3913 & 1.0558 & 0.4132
& 1.0000 & 0.5913 & 1.0558 & 0.6243
& 1.0000 & 0.6210 & 1.0558 & 0.6557 \\
& & FedKD-NAS
& 0.9167 & \textbf{1.0000} & \textbf{4.7268} & \textbf{5.1565}
& 0.9321 & \textbf{1.0000} & \textbf{4.7268} & \textbf{5.0712}
& 0.9396 & \textbf{1.0000} & \textbf{4.7268} & \textbf{5.0307} \\

\cmidrule(lr){2-15}

& \multirow{7}{*}{ResNet18}
& FedAvg
& \textbf{0.7063} & \textbf{1.0000} & 1.0357 & 1.4664
& \textbf{0.7851} & 0.9103 & 1.0357 & 1.2008
& 0.8489 & 0.8933 & 1.0357 & 1.0898 \\
& & Ditto
& 0.7175 & 0.8455 & 1.0357 & 1.2204
& 0.7889 & 0.2377 & 1.0357 & 0.3120
& 0.8562 & 0.2175 & 1.0357 & 0.2631 \\
& & FedDF
& 0.7117 & 0.5155 & 1.0000 & 0.7244
& 0.8091 & 0.6633 & 1.0000 & 0.8198
& 0.8751 & 0.7200 & 1.0000 & 0.8228 \\
& & FedDistill
& 0.7499 & 0.7989 & \textbf{7.2587} & 7.7328
& 0.9631 & 0.9029 & \textbf{7.2587} & 6.8047
& \textbf{0.8241} & 0.9517 & \textbf{7.2587} & \textbf{8.3831} \\
& & FedMD
& 0.9514 & 0.7891 & \textbf{7.2587} & 6.0202
& 0.8364 & 0.8777 & \textbf{7.2587} & 7.6169
& 0.9727 & 0.9257 & \textbf{7.2587} & 6.9082 \\
& & Local-KD
& 0.7992 & 0.0439 & 1.0357 & 0.0569
& 0.8812 & 0.4434 & 1.0357 & 0.5212
& 0.9423 & 0.6461 & 1.0357 & 0.7101 \\
& & FedKD-NAS
& 0.7095 & 0.9508 & \textbf{7.2587} & \textbf{9.7278}
& 0.8087 & \textbf{1.0000} & \textbf{7.2587} & \textbf{8.9755}
& 0.8733 & \textbf{1.0000} & \textbf{7.2587} & 8.3114 \\

\bottomrule
\end{tabular}%
\end{adjustbox}
\end{table*}

\subsubsection{Convergence Quality and Robustness to Statistical Heterogeneity (PQS, $Acc$ drop)}\label{sec:Exp:Results:PQS}
Beyond accuracy, PQS captures convergence speed and stability. On CIFAR10, FedKD-NAS reaches 1.0000 under Dirichlet and Shards with MobileNetV2, versus FedMD at 0.8126 and 0.9050, and achieves 1.0000 across all three distributions with ShuffleNetV2. Although it trails FedAvg under IID with MobileNetV2 (0.6940 vs. 0.9897), this gap disappears under heterogeneity. The same pattern holds on smaller benchmarks, with PQS near 1.0000 under challenging non-IID settings and increasing on FMNIST from 0.9796 to 0.9985, confirming a systematic rather than dataset-specific advantage (\tablename~\ref{tab:all_combined_quality}). 

Turning to robustness, \figurename~\ref{fig:cifar10_acc_drop}, which visualizes the CIFAR10 robustness analysis by showing for each method the final accuracy under Dirichlet and Shards partitioning relative to its IID reference accuracy (the horizontal black markers denote IID accuracy, while the dashed vertical lines quantify the accuracy drop from IID to each non-IID setting),  shows that FedKD-NAS is the only method combining high accuracy with minimal degradation as heterogeneity increases. Across MobileNetV2 and ShuffleNetV2, it drops by only ${\approx}3\%$ from IID to Dirichlet and ${\approx}4\%$ to Shards, while other methods degrade by 20\% on average, and Ditto and Local-KD collapse by more than 35-43\% under Shards. The same pattern holds on other benchmarks: on EMNIST (LeNet5), FedKD-NAS drops by only 1.3\% to Dirichlet and 2.9\% to Shards, compared with Ditto's ${\approx}36\%$ and FedAvg's ${\approx}6.4\%$ drops.
    \begin{tcolorbox}[colback=gray!5!white,colframe=gray!50!black!50,title=Takeaway]
    FedKD-NAS achieves perfect or near-perfect PQS across most dataset/model settings, with $+$10-19\% gains over the strongest baselines, and ranks 1\textsuperscript{st} in robustness, with only ${\approx}3$-$4\%$ accuracy drop from IID to non-IID settings.
    \end{tcolorbox}
On FMNIST and MNIST, FedKD-NAS again shows the smallest accuracy drop in the distillation family across non-IID regimes (\tablename~\ref{tab:all_combined_raw}).

\subsubsection{Communication Efficiency ($Comm$, CES)}\label{sec:Exp:Results:Comm}
\figurename~\ref{fig:cifar10_comm} highlights the underlying communication split on CIFAR10. Aggregation-based approaches (FedAvg, Ditto, Local-KD) and FedDF exchange full model parameters every round (${\approx}71.6$~MB for MobileNetV2,
${\approx}11.3$~MB for ShuffleNetV2), whereas FedDistill, FedMD, and FedKD-NAS communicate only logit vectors (${\approx}1.62$~MB), yielding a structural \textbf{$44\times$ reduction} for MobileNetV2 and \textbf{$7\times$} for
ShuffleNetV2 invariant to the data distribution. The same pattern holds for MNIST
and FMNIST, where distillation methods reduce $Comm$ from ${\approx}8.6$~MB to
${\approx}1.9$~MB (${\approx}4.5\times$, \tablename~\ref{tab:all_combined_raw}).
These communication savings translate into high CES, which is notable on CIFAR10, where FedKD-NAS, FedDistill, and FedMD all attain CES~$= 44.49$ (MobileNetV2) and
CES~$= 7.21$ (ShuffleNetV2), vs.\ CES~$\approx 1$ for aggregation-based methods
(\tablename~\ref{tab:all_combined_quality}). However, EMNIST presents a significant inversion of this pattern, as the logit payload for 47 classes (${\approx}13.2$~MB on LeNet5) exceeds the model weight size (${\approx}8.9$~MB), so aggregation-based methods achieve lower communication cost than distillation methods (CES~$= 1.91$ vs.\ 1.00, \tablename~\ref{tab:all_combined_quality}), a finding with an important practical implication explaining that the bandwidth savings of logit-based distillation are architecture- and label-space-dependent, disappearing when the number of classes is large relative to model size, with the crossover identifiable analytically before training when $C \cdot |\mathcal{D}^{\mathrm{pub}}|$ bits per round exceeds $P$. Within the logit-based family on CIFAR10, where all three distillation methods share identical CES, FedKD-NAS uniquely converts its communication budget into both higher predictive quality and lower resource use, reaching PQS~$= 1.0000$ vs.\ 0.8032 for FedDistill and 0.8126 for FedMD under Dirichlet with the largest bubble ($\propto 1/\text{RES}$) in \figurename~\ref{fig:cifar10_bubble}; even on EMNIST, where the CES structure
is inverted, FedKD-NAS retains the highest PQS among all distillation methods
under Dirichlet (1.000) and Shards (1.000) with LeNet5 (\tablename~\ref{tab:all_combined_quality}), confirming that its predictive advantage within the distillation family holds regardless of the communication cost structure.

\begin{tcolorbox}[colback=gray!5!white,colframe=gray!50!black!50,title=Takeaway]
FedKD-NAS attains the same structural $44\times$ (MobileNetV2) and $7\times$ (ShuffleNetV2) communication reduction as the other logit-based methods on CIFAR10, while uniquely converting that budget into superior predictive quality. The distillation bandwidth advantage is label-space-dependent and reverses when class count is large relative to model size.

\end{tcolorbox}

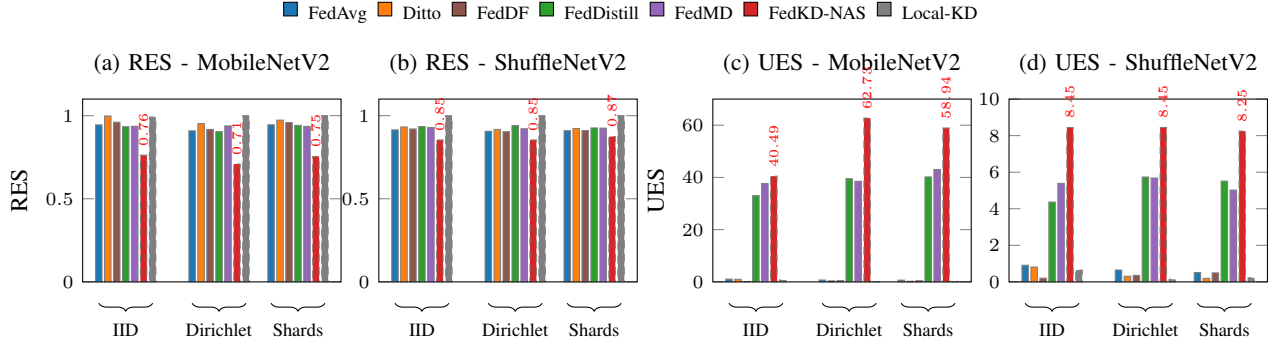
\begin{figure*}[t]
\centering

\begin{tabular}{cccccccccccccc}
  \raisebox{0.2ex}{\tikz\draw[fill=FedAvgColor,    draw=black](0,0)rectangle(0.18,0.18);} & \scriptsize FedAvg     &
  \raisebox{0.2ex}{\tikz\draw[fill=DittoColor,     draw=black](0,0)rectangle(0.18,0.18);} & \scriptsize Ditto      &
  \raisebox{0.2ex}{\tikz\draw[fill=FedDFColor,     draw=black](0,0)rectangle(0.18,0.18);} & \scriptsize FedDF      &
  \raisebox{0.2ex}{\tikz\draw[fill=FedDistillColor,draw=black](0,0)rectangle(0.18,0.18);} & \scriptsize FedDistill &
  \raisebox{0.2ex}{\tikz\draw[fill=FedMDColor,     draw=black](0,0)rectangle(0.18,0.18);} & \scriptsize FedMD      &
  \raisebox{0.2ex}{\tikz\draw[fill=FedKDNASColor,  draw=black](0,0)rectangle(0.18,0.18);} & \scriptsize FedKD-NAS  &
  \raisebox{0.2ex}{\tikz\draw[fill=LocalKDColor,   draw=black](0,0)rectangle(0.18,0.18);} & \scriptsize Local-KD   \\
\end{tabular}
\vspace{0.6em}

\makebox[0pt][c]{%
\begin{tikzpicture}
\begin{groupplot}[
    group style={
        group size     = 4 by 1,
        horizontal sep = 0.5cm,
    },
    width  = 5cm,
    height = 4.0cm,
    xmin=-1.2, xmax=6.0,
    xtick=\empty,
    tick align=inside,
    clip=false,
    yticklabel style  = {font=\scriptsize},
    label style       = {font=\small},
    title style       = {font=\small},
    enlarge x limits  = false,
    every axis plot/.append style={mark=none},
]
\nextgroupplot[
    title  = {(a) RES - MobileNetV2},
    ylabel = {RES},
    ymin=0, ymax=1.1,
    bar width=0.18,
]
\AddThreeRegimePlotShift{RES}{0}{\MobiIID}{\MobiDir}{\MobiSha}{FedAvgStyle}    {-0.75}{1.85}{4.05}
\AddThreeRegimePlotShift{RES}{1}{\MobiIID}{\MobiDir}{\MobiSha}{DittoStyle}     {-0.50}{2.10}{4.30}
\AddThreeRegimePlotShift{RES}{2}{\MobiIID}{\MobiDir}{\MobiSha}{FedDFStyle}     {-0.25}{2.35}{4.55}
\AddThreeRegimePlotShift{RES}{3}{\MobiIID}{\MobiDir}{\MobiSha}{FedDistillStyle}{ 0.00}{2.60}{4.80}
\AddThreeRegimePlotShift{RES}{4}{\MobiIID}{\MobiDir}{\MobiSha}{FedMDStyle}     { 0.25}{2.85}{5.05}
\AddThreeRegimePlotShift{RES}{5}{\MobiIID}{\MobiDir}{\MobiSha}{FedKDNASStyle}  { 0.50}{3.10}{5.30}
\AddThreeRegimePlotShift{RES}{6}{\MobiIID}{\MobiDir}{\MobiSha}{LocalKDStyle}   { 0.75}{3.35}{5.55}
\DrawRegimeBrackets

\nextgroupplot[
    title  = {(b) RES - ShuffleNetV2},
    ymin=0, ymax=1.1,
    bar width=0.18,
]
\AddThreeRegimePlotShift{RES}{0}{\ShufIID}{\ShufDir}{\ShufSha}{FedAvgStyle}    {-0.75}{1.85}{4.05}
\AddThreeRegimePlotShift{RES}{1}{\ShufIID}{\ShufDir}{\ShufSha}{DittoStyle}     {-0.50}{2.10}{4.30}
\AddThreeRegimePlotShift{RES}{2}{\ShufIID}{\ShufDir}{\ShufSha}{FedDFStyle}     {-0.25}{2.35}{4.55}
\AddThreeRegimePlotShift{RES}{3}{\ShufIID}{\ShufDir}{\ShufSha}{FedDistillStyle}{ 0.00}{2.60}{4.80}
\AddThreeRegimePlotShift{RES}{4}{\ShufIID}{\ShufDir}{\ShufSha}{FedMDStyle}     { 0.25}{2.85}{5.05}
\AddThreeRegimePlotShift{RES}{5}{\ShufIID}{\ShufDir}{\ShufSha}{FedKDNASStyle}  { 0.50}{3.10}{5.30}
\AddThreeRegimePlotShift{RES}{6}{\ShufIID}{\ShufDir}{\ShufSha}{LocalKDStyle}   { 0.75}{3.35}{5.55}
\DrawRegimeBrackets

\nextgroupplot[
    title  = {(c) UES - MobileNetV2},
    ylabel = {UES},
    ymin=0, ymax=70,
    xshift=0.5cm,
    bar width=0.18,
]
\AddThreeRegimePlotShift{UES}{0}{\MobiIID}{\MobiDir}{\MobiSha}{FedAvgStyle}    {-0.75}{1.85}{4.05}
\AddThreeRegimePlotShift{UES}{1}{\MobiIID}{\MobiDir}{\MobiSha}{DittoStyle}     {-0.50}{2.10}{4.30}
\AddThreeRegimePlotShift{UES}{2}{\MobiIID}{\MobiDir}{\MobiSha}{FedDFStyle}     {-0.25}{2.35}{4.55}
\AddThreeRegimePlotShift{UES}{3}{\MobiIID}{\MobiDir}{\MobiSha}{FedDistillStyle}{ 0.00}{2.60}{4.80}
\AddThreeRegimePlotShift{UES}{4}{\MobiIID}{\MobiDir}{\MobiSha}{FedMDStyle}     { 0.25}{2.85}{5.05}
\AddThreeRegimePlotShift{UES}{5}{\MobiIID}{\MobiDir}{\MobiSha}{FedKDNASStyle}  { 0.50}{3.10}{5.30}
\AddThreeRegimePlotShift{UES}{6}{\MobiIID}{\MobiDir}{\MobiSha}{LocalKDStyle}   { 0.75}{3.35}{5.55}
\DrawRegimeBrackets

\nextgroupplot[
    title  = {(d) UES - ShuffleNetV2},
    ymin=0, ymax=10,
    bar width=0.18,
]
\AddThreeRegimePlotShift{UES}{0}{\ShufIID}{\ShufDir}{\ShufSha}{FedAvgStyle}    {-0.75}{1.85}{4.05}
\AddThreeRegimePlotShift{UES}{1}{\ShufIID}{\ShufDir}{\ShufSha}{DittoStyle}     {-0.50}{2.10}{4.30}
\AddThreeRegimePlotShift{UES}{2}{\ShufIID}{\ShufDir}{\ShufSha}{FedDFStyle}     {-0.25}{2.35}{4.55}
\AddThreeRegimePlotShift{UES}{3}{\ShufIID}{\ShufDir}{\ShufSha}{FedDistillStyle}{ 0.00}{2.60}{4.80}
\AddThreeRegimePlotShift{UES}{4}{\ShufIID}{\ShufDir}{\ShufSha}{FedMDStyle}     { 0.25}{2.85}{5.05}
\AddThreeRegimePlotShift{UES}{5}{\ShufIID}{\ShufDir}{\ShufSha}{FedKDNASStyle}  { 0.50}{3.10}{5.30}
\AddThreeRegimePlotShift{UES}{6}{\ShufIID}{\ShufDir}{\ShufSha}{LocalKDStyle}   { 0.75}{3.35}{5.55}
\DrawRegimeBrackets

\end{groupplot}
\end{tikzpicture}%
}
\caption{RES~$\downarrow$ and UES~$\uparrow$ on CIFAR10 across IID, Dirichlet ($\alpha{=}0.1$), and Shards. FedKD-NAS achieves the lowest RES and highest UES across distributions and architectures, with UES increasing under heterogeneity.}
\label{fig:cifar10_res_ues}
\end{figure*}

\subsubsection{Client-Side Resource Consumption ($CPU$, $Mem$, RES)}
\label{sec:Exp:Results:resource}

The RES panels in \figurename~\ref{fig:cifar10_res_ues} (panels a-b) confirm that FedKD-NAS is the least resource-intensive approach on CIFAR10 across model architectures and data heterogeneity settings. With MobileNetV2, average CPU is ${\approx}34\%$ versus ${\approx}47\%$ for aggregation-based baselines (a 28\% relative reduction), and RAM is also lowest overall (971.9 / 911.7 / 947.5~MB under IID / Dirichlet / Shards, \tablename~\ref{tab:cifar10_combined_raw}), and FedKD-NAS again records the lowest CPU (${\approx}35\%$ vs.\ ${\approx}47\%$) on CIFAR100 (\tablename~\ref{tab:all_combined_raw}).
These savings translate into RES values of 0.7625 / 0.7092 / 0.7549 on CIFAR10 (MobileNetV2) versus 0.90-1.00 for competing baselines (\tablename~\ref{tab:all_combined_quality}, as reflected in the largest bubble within the high-CES distillation cluster across all three distributions in \figurename~\ref{fig:cifar10_bubble}. These gains arise from NAS-derived compact client architectures that reduce per-round training cost without sacrificing accuracy.

On EMNIST (LeNet5), the RES picture differs significantly from CIFAR10 due to the inverted CES structure. Because aggregation methods carry lower communication cost on this dataset, FedDistill achieves the best RES (0.8075~IID, 0.8416~Dirichlet, 0.8261~Shards, \tablename~\ref{tab:all_combined_quality}), rewarded by both low communication cost and competitive CPU usage. FedKD-NAS records higher RES (0.9040 / 0.9450 / 0.9239), penalized by its larger logit payload, however it recovers on ResNet18 EMNIST under Shards, achieving the best overall RES (0.5781) with the lowest CPU of all methods (42.5\%, \tablename~\ref{tab:all_combined_raw}).

On MNIST and FMNIST (LeNet5), FedDistill records the best RES because CPU differences are slightly small ($<5\%$) and the communication term dominates, yet FedKD-NAS remains the most accurate within that low-RES group, and on CIFAR100
it leads RES within the distillation family (0.757 / 0.702 / 0.753, \tablename~\ref{tab:all_combined_quality}).

\begin{tcolorbox}[colback=gray!5!white,colframe=gray!50!black!50,title=Takeaway]
FedKD-NAS ranks \textbf{1\textsuperscript{st}} on RES on CIFAR10 and CIFAR100 across all distributions, with a \textbf{28\% CPU reduction} (${\approx}34\%$ vs.\ ${\approx}47\%$) and RES of 0.71-0.76 vs.\ 0.90-1.00 for other methods, it ranks \textbf{2\textsuperscript{nd}} within the distillation family on MNIST and FMNIST wher FedDistill leads on RES while FedKD-NAS leads on accuracy, and recovers to \textbf{1\textsuperscript{st}} overall on ResNet18 EMNIST.

\end{tcolorbox}

\subsubsection{Unified Deployment Efficiency (UES)}\label{sec:Exp:Results:UES}

The UES panels in \figurename~\ref{fig:cifar10_res_ues} (panels c-d) provide the clearest deployment-level trade-off. On CIFAR10 with MobileNetV2, FedKD-NAS achieves UES of \textbf{40.49} (IID), \textbf{62.73} (Dirichlet), and \textbf{58.94} (Shards), outperforming the next best baseline by around $1.5\times$-$4\times$ (FedMD: 37.67 / 38.51 / 43.03, \tablename~\ref{tab:all_combined_quality}). Crucially, UES increases from IID to Dirichlet, confirming that FedKD-NAS becomes more competitive as data heterogeneity increases. This trend is further shown in \figurename~\ref{fig:cifar10_bubble}, where FedKD-NAS's bubble grows across all three distributions, indicating simultaneous improvement on the three dimensions (CES, PQS, and $1/	ext{RES}$), which no competing baseline replicates. With ShuffleNetV2, FedKD-NAS achieves scores of 8.45 / 8.45 / 8.25, while FedMD scores 5.39 / 5.68 / 5.02, giving FedKD-NAS an approximate 60\% advantage. The same pattern extends across all other benchmarks. On MNIST (LeNet5), for example, FedKD-NAS leads UES across all distributions (5.16 / 5.07 / 5.03, $>$$1.3\times$ above FedDistill and FedMD, \tablename~\ref{tab:all_combined_quality}). on FMNIST (LeNet5), it records the highest PQS and top UES under IID and Shards, with FedMD leading under Dirichlet only by a small margin due to marginally lower RES (\tablename~\ref{tab:all_combined_quality}, \figurename~\ref{fig:fmnist_bubble}). On EMNIST (LeNet5), the CES inversion raises FedAvg's UES above all distillation methods, yet within the distillation family FedKD-NAS leads UES under all three distributions, and on ResNet18 EMNIST under Shards, its UES of 1.6891 nearly matches FedAvg's 2.1622, substantially closing the gap compared to IID (0.4471 vs.\ 1.7180, \tablename~\ref{tab:all_combined_quality}, \figurename~\ref{fig:emnist_bubble}). And on CIFAR100, although FedAvg leads UES on raw-accuracy grounds, FedKD-NAS reaches ${\approx}4\times$ higher UES within the distillation family under Dirichlet (4.97 vs.\ 1.18) because its CPU and communication savings surpass FedAvg's accuracy margin in the UES
formulation (\tablename~\ref{tab:all_combined_quality}, \figurename~\ref{fig:cifar100_bubble}).

\begin{tcolorbox}[colback=gray!5!white,colframe=gray!50!black!50,title=Takeaway]
FedKD-NAS ranks \textbf{1\textsuperscript{st}} on UES in \textbf{31 of 35} configurations, with gains of \textbf{$1.5\times$-$4\times$} over the second best baseline on CIFAR10 MobileNetV2, \textbf{${\approx}60\%$} on ShuffleNetV2,
and a UES that \emph{grows} with heterogeneity ($+$55\% from IID to Dirichlet on MobileNetV2), confirming that NAS-driven client adaptation delivers better deployment-level efficiency under realistic, heterogeneous federated conditions.

\end{tcolorbox}

\begin{table}[!htbp]
\centering
\caption{HAR results at the final communication round. We therefore report UES$^\star$ = PQS $\cdot$ CES as a resource-free unified efficiency indicator.}
\label{tab:har_results}
\scriptsize
\resizebox{\columnwidth}{!}{
\begin{tabular}{lcccccc}
\toprule
Algorithm & Acc $\uparrow$ & Loss $\downarrow$ & Comm (MB) $\downarrow$ & PQS $\uparrow$ & CES $\uparrow$ & UES$^\star$ $\uparrow$ \\
\midrule
FedAvg       & 0.7313 & 0.6046 & 2.588 & 0.081 & 1.000  & 0.08 \\
FedMD        & 0.7864 & 0.5235 & 0.176 & 0.275 & 14.724 & 4.04 \\
FedDistill   & 0.8206 & 0.4681 & 0.176 & 0.398 & 14.724 & 5.86 \\
\textbf{FedKD-NAS} & \textbf{0.9442} & \textbf{0.1632} & \textbf{0.176} & \textbf{1.000} & \textbf{14.724} & \textbf{14.72} \\
\bottomrule
\end{tabular}
}
\end{table}
\begin{table}[!t]
\centering
\caption{Client-side efficiency metrics on real devices (1 RockPi, 1 Jetson Nano, 3 Raspberry Pi; K=5). Communication cost is computed as $(\text{server\_comm\_in}+\text{server\_comm\_out})$. Since test/val loss was not available in the exported logs, PQS is computed using an accuracy-only proxy:$\text{PQS}=0.7\cdot \text{Acc}_{\text{norm}} + 0.3$.}
\label{tab:real_devices_metrics}
\scriptsize
\resizebox{\columnwidth}{!}{
\begin{tabular}{lcccccccc}
\toprule
Algorithm & Acc $\uparrow$ & CPU & RAM (MB) & Comm (MB) &
RES $\downarrow$ & PQS $\uparrow$ & CES $\uparrow$ & UES $\uparrow$ \\
\midrule
FedAvg     & 0.8093 & 92.64 & 392.80 & 4.25 & 0.979 & 0.300 & 1.000 & 0.31 \\
FedMD      & 0.8194 & 91.82 & 397.76 & 2.29 & 0.981 & 0.379 & 1.854 & 0.72 \\
\textbf{FedKD-NAS} & \textbf{0.8991} & \textbf{90.53} & 409.99 & \textbf{2.29} &
0.989 & \textbf{1.000} & \textbf{1.854} & \textbf{1.88} \\
\bottomrule
\end{tabular}
}
\end{table}

\subsection{Experimental Results on Real-World Datasets and Devices}\label{sec:Exp:ResultsOnReal}
We evaluate FedKD-NAS under realistic deployment conditions using both a large-scale real-world dataset (~\ref{sec:Exp:ResultsOnCASA}) and heterogeneous edge devices (~\ref{sec:Exp:ResultsOnRealDevices}). Unlike synthetic benchmarks, these settings capture natural statistical heterogeneity, resource variability, and system-level constraints.

\subsubsection{Human activity recognition dataset (CASA)}\label{sec:Exp:ResultsOnCASA}
\tablename~\ref{tab:har_results} reports performance on CASA, a real-world activity recognition benchmark characterised by naturally heterogeneous sensor streams. As resource traces were unavailable, we additionally report UES$^\star = \text{PQS} \cdot \text{CES}$ as a resource-free unified efficiency indicator. FedKD-NAS achieves the best accuracy by a \textbf{wide margin} of 0.9442 vs.\ 0.8206 for FedDistill ($+15\%$), 0.7864 for FedMD ($+20\%$), and 0.7313 for FedAvg ($+29\%$), together with the lowest loss (0.1632) and the minimum communication cost (0.176~MB, tied with FedMD and FedDistill).
It reaches the maximum PQS (1.000) and the highest UES$^\star$ (\textbf{14.72}), nearly $2.5\times$ above FedDistill (5.86) and $183\times$ above FedAvg (0.08). These findings are especially notable because CASA is inherently non-IID and imbalanced. The
$+15\%$ accuracy gain over the second-best distillation baseline confirms that NAS-driven adaptation generalises beyond vision tasks to time-series and sensor domains. FedAvg's UES$^\star$ of 0.08 further underlines the limitations of parameter aggregation on heterogeneous non-vision data.

\subsubsection{Real device experiments}\label{sec:Exp:ResultsOnRealDevices}
\tablename~\ref{tab:real_devices_metrics} summarizes a real device deployment on five heterogeneous edge devices (1~RockPi, 1~Jetson Nano, 3~Raspberry~Pi boards). FedKD-NAS attains the highest validation accuracy (\textbf{0.8991}), exceeding FedMD (0.8194) by by \textbf{$+8\%$} and FedAvg (0.8093) by \textbf{$+11\%$}, despite varying device capabilities. It records the lowest CPU usage (\textbf{90.53\%} vs.\ 91.82\% for FedMD and 92.64\% for FedAvg), consistent with NAS-derived compact models providing CPU savings under real-device effects. Communication cost is 2.29~MB, matching FedMD and representing a ${\approx}1.86\times$ reduction over FedAvg (4.25~MB), confirming the practicality of logit-based communication on physical hardware. The only metric where FedKD-NAS does not lead is RAM (409.99~MB vs.\ 397.76~MB for FedMD), a modest ${\approx}4\%$ overhead for maintaining locally adapted model architectures on-device. This cost is more than justified by overall deployment efficiency. FedKD-NAS reaches UES~$=$ \textbf{1.88} versus 0.72 for FedMD and 0.31 for FedAvg ($2.6\times$ and $6\times$ improvements). The ordering of methods on real devices matches simulation throughout, supporting the validity of the composite metrics and reinforcing FedKD-NAS as a practical solution for heterogeneous IoT and edge FL.

\begin{tcolorbox}[colback=gray!5!white,colframe=gray!50!black!50,title=Takeaway]
FedKD-NAS ranks \textbf{1\textsuperscript{st}} on all metrics across both real-world evaluations: on CASA it gains \textbf{$+$15\%} accuracy over FedDistill with a UES$^\star$ \textbf{$2.5\times$} higher, confirming generalisation to time-series and sensor domains, while on the physical hardware it delivers \textbf{$+$8\%} accuracy over FedMD, the lowest CPU (90.53\%), a \textbf{$1.86\times$} communication reduction over FedAvg, and a UES \textbf{$2.6\times$} above FedMD and \textbf{$6\times$}
above FedAvg, at the cost of only a ${\approx}4\%$ RAM overhead.
\end{tcolorbox}

\begin{figure*}[htbp]
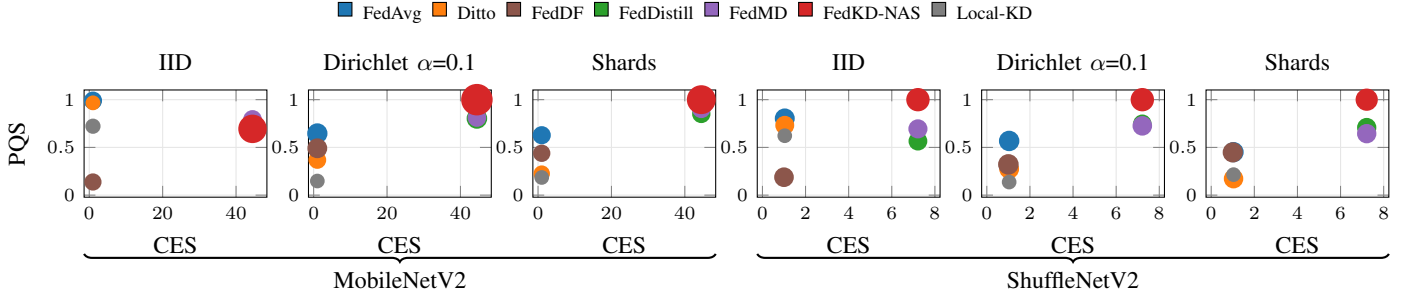

\centering

\begin{tabular}{cccccccccccccc}
  \raisebox{0.2ex}{\tikz\draw[fill=FedAvgColor,    draw=black](0,0)rectangle(0.18,0.18);} & \scriptsize FedAvg     &
  \raisebox{0.2ex}{\tikz\draw[fill=DittoColor,     draw=black](0,0)rectangle(0.18,0.18);} & \scriptsize Ditto      &
  \raisebox{0.2ex}{\tikz\draw[fill=FedDFColor,     draw=black](0,0)rectangle(0.18,0.18);} & \scriptsize FedDF      &
  \raisebox{0.2ex}{\tikz\draw[fill=FedDistillColor,draw=black](0,0)rectangle(0.18,0.18);} & \scriptsize FedDistill &
  \raisebox{0.2ex}{\tikz\draw[fill=FedMDColor,     draw=black](0,0)rectangle(0.18,0.18);} & \scriptsize FedMD      &
  \raisebox{0.2ex}{\tikz\draw[fill=FedKDNASColor,  draw=black](0,0)rectangle(0.18,0.18);} & \scriptsize FedKD-NAS  &
  \raisebox{0.2ex}{\tikz\draw[fill=LocalKDColor,   draw=black](0,0)rectangle(0.18,0.18);} & \scriptsize Local-KD   \\
\end{tabular}%

\caption{CES vs.\ PQS trade-off on CIFAR10 (bubble area $\propto 1/\text{RES}$, i.e.\ larger bubbles indicate lower resource usage). \emph{Left:} MobileNetV2. \emph{Right:} ShuffleNetV2. Methods split into two clusters along the CES axis: aggregation-based methods (left, CES $\approx$ 1) and logit-based distillation methods (right, high CES). FedKD-NAS occupies the top-right corner with the highest CES, highest PQS, and the largest bubble within the distillation cluster.}
\label{fig:cifar10_bubble}
\end{figure*}
\begin{figure*}[htbp]
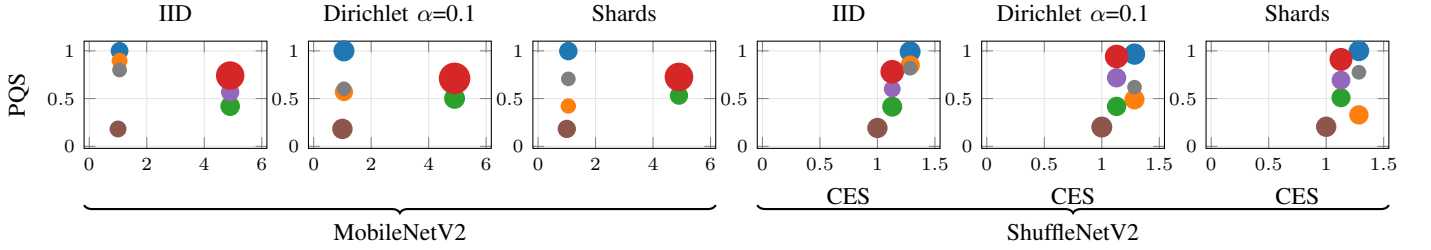

\centering
\caption{CES vs.\ PQS trade-off on CIFAR100 (bubble area $\propto 1/\text{RES}$).
  \emph{Top:} MobileNetV2. \emph{Bottom:} ShuffleNetV2. FedAvg occupies the top-left cluster (high PQS, low CES) under IID, while FedKD-NAS leads the high-CES cluster with the largest bubble. Under non-IID distributions, FedKD-NAS progressively closes the PQS gap with FedAvg while maintaining its CES and
  resource efficiency advantages.}
  \label{fig:cifar100_bubble}
\end{figure*}
\begin{figure*}[htbp]
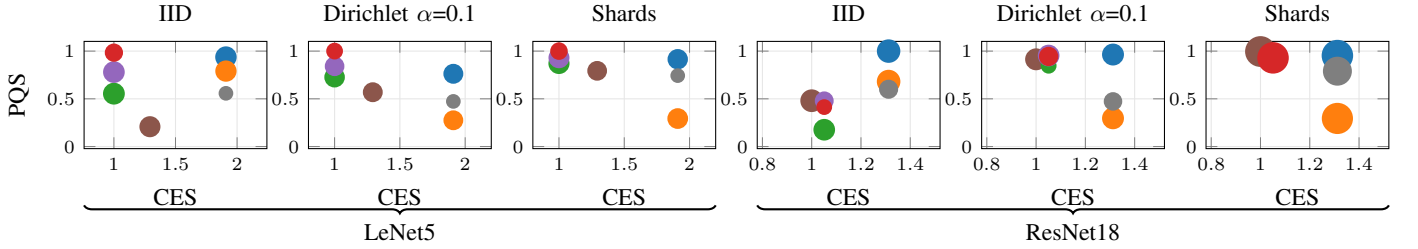

\centering

\caption{CES vs.\ PQS trade-off on EMNIST (bubble area $\propto 1/\text{RES}$).
  \emph{Top:} LeNet5. \emph{Bottom:} ResNet18. On LeNet5, FedAvg/Ditto/Local-KD occupy a higher CES tier than distillation methods due to the architecture-specific Comm ratio on this dataset. On ResNet18, FedKD-NAS leads PQS under Shards with the largest bubble in the high-CES group.}
  \label{fig:emnist_bubble}
\end{figure*}
\begin{figure*}[htbp]
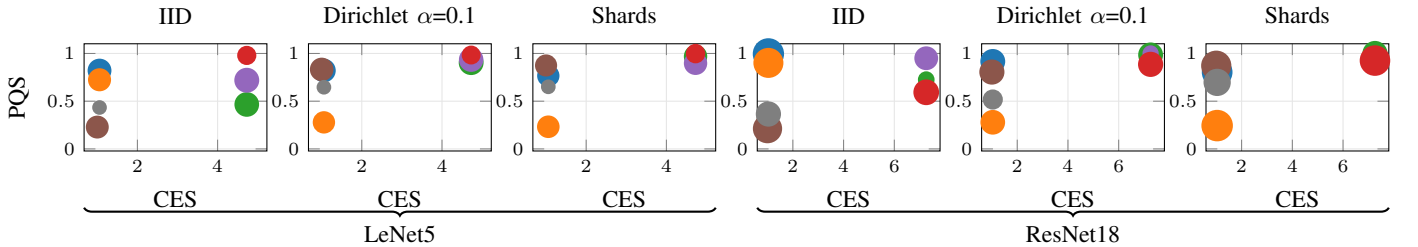

\centering
\caption{CES vs.\ PQS trade-off on FMNIST (bubble area $\propto 1/\text{RES}$).
  \emph{Top:} LeNet5. \emph{Bottom:} ResNet18. On LeNet5, FedKD-NAS leads PQS in the high-CES cluster across all distributions. On ResNet18, FedKD-NAS and FedMD compete closely under Dirichlet, with FedKD-NAS holding the larger bubble (lower RES) in both cases.}
  \label{fig:fmnist_bubble}
\end{figure*}
\begin{figure*}[htbp]
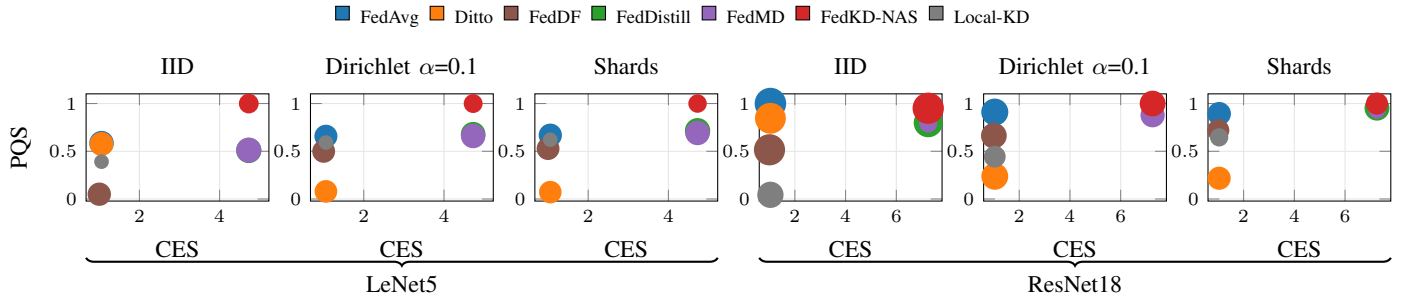

\centering
\begin{tabular}{cccccccccccccc}
  \raisebox{0.2ex}{\tikz\draw[fill=FedAvgColor,    draw=black](0,0)rectangle(0.18,0.18);} & \scriptsize FedAvg     &
  \raisebox{0.2ex}{\tikz\draw[fill=DittoColor,     draw=black](0,0)rectangle(0.18,0.18);} & \scriptsize Ditto      &
  \raisebox{0.2ex}{\tikz\draw[fill=FedDFColor,     draw=black](0,0)rectangle(0.18,0.18);} & \scriptsize FedDF      &
  \raisebox{0.2ex}{\tikz\draw[fill=FedDistillColor,draw=black](0,0)rectangle(0.18,0.18);} & \scriptsize FedDistill &
  \raisebox{0.2ex}{\tikz\draw[fill=FedMDColor,     draw=black](0,0)rectangle(0.18,0.18);} & \scriptsize FedMD      &
  \raisebox{0.2ex}{\tikz\draw[fill=FedKDNASColor,  draw=black](0,0)rectangle(0.18,0.18);} & \scriptsize FedKD-NAS  &
  \raisebox{0.2ex}{\tikz\draw[fill=LocalKDColor,   draw=black](0,0)rectangle(0.18,0.18);} & \scriptsize Local-KD   \\
\end{tabular}%

\caption{CES vs.\ PQS trade-off on MNIST (bubble area $\propto 1/\text{RES}$).
  \emph{Top:} LeNet5. \emph{Bottom:} ResNet18. FedKD-NAS occupies the top-right
  position in the high-CES cluster under non-IID distributions. Under IID with ResNet18,
  FedAvg leads on PQS but falls in the low-CES cluster, while FedKD-NAS maintains
  high CES with competitive PQS.}
  \label{fig:mnist_bubble}
\end{figure*}

\section{Discussion}\label{sec:discussion}

The empirical evidence reported in Section~\ref{sec:Exp} establishes that \textbf{FedKD-NAS achieves the most favorable joint performance–communication–resource trade-off of any evaluated method, and this advantage strengthens as statistical heterogeneity increases.} Three complementary mechanisms explain this outcome. First, NAS-based client adaptation customizes each device's model architecture to its local data distribution, directly reducing client drift. Second, logit-based distillation eliminates full model exchange while enabling richer server-side aggregation than raw labels alone. Third, jointly optimizing architecture search and distillation achieves a better balance between local specialization and global consensus than either pure aggregation or fixed-architecture distillation can provide.

\subsection{Impact of Data Heterogeneity}\label{sec:discussion:hetero}

The most theoretically significant finding is not that FedKD-NAS outperforms the other methods on average, but that its predictive quality consistently improves as heterogeneity increases, a behavior absent from all competing methods and unexpected from the perspective of standard convergence analysis. This pattern is explained by two interacting mechanisms grounded in the theoretical framework of Section~\ref{sec:convergence}.
Assumption~A4 combined with Assumption~A5 and Lemma~\ref{lem:ema} jointly ensure that the smoothed distillation targets remain stable across rounds even as local distributions diverge sharply, providing a coherent supervisory signal that aggregation-based methods cannot replicate because they accumulate gradient conflicts proportional to distributional divergence. Assumption~A6, now derived explicitly from Algorithm~2 via Proposition~1,
ensures that the architecture selection error remains bounded regardless of local distribution shift. Beyond basic stability, greater distributional diversity actively enriches the NAS search signal, enabling the controller to discover model architectures that generalise across a broader range of local data landscapes rather than overfitting to any single distribution.
Heterogeneity thus functions as a beneficial constructive bias for architecture search, a property structurally unavailable to fixed-architecture distillation methods, and one that helps explain why FedKD-NAS's PQS saturates at the top of its peer group precisely under the most challenging non-IID distributions across MNIST, FMNIST, and EMNIST.
The sharp deterioration of Local-KD under heterogeneity provides a useful ablation-like signal confirming that server-coordinated distillation targets, not compact architectures alone, are necessary for this robustness. Without the EMA-stabilised signal formalised in Lemma~\ref{lem:ema}, purely local training amplifies rather than absorbs distributional divergence.

\subsection{Performance versus Communication Efficiency}
\label{sec:discussion:perf}

The communication results contribute a finding to the broader federated distillation literature that has received insufficient attention. Prior works on logit-based distillation, including
FedDF~\cite{lin2020ensemble}, FedDistill~\cite{sattler2021fedaux}, and FedMD~\cite{li2019fedmd} have consistently reported communication savings relative to weight-sharing methods but without characterizing the conditions under which this advantage breaks down. The EMNIST results demonstrate that when the class count is large relative to model size, the logit payload exceeds the model payload and the communication advantage reverses entirely, with aggregation-based methods becoming the more communication-efficient family. This is a structural limitation of all logit-based approaches and motivates future investigation into compressed logit representations through
sparsification, top-$k$ selection, or quantized outputs as a necessary complement to the distillation paradigm in high-class-count settings.

Within the settings where the communication advantage holds, the results expose
a further distinction that pure communication metrics obscure. FedKD-NAS, FedDistill, and FedMD share identical CES on CIFAR10 and CIFAR100, yet FedKD-NAS consistently converts this shared communication budget into substantially higher predictive quality, providing empirical evidence that model architecture adaptation contributes value beyond what distillation alone achieves. This finding extends to non-vision domains, as confirmed by the CASA results, where FedKD-NAS surpasses FedDistill and FedMD by substantial accuracy margins on a naturally heterogeneous sensor dataset. Even on EMNIST, where the CES is inverted, and FedKD-NAS no longer occupies the high-CES cluster, it retains the highest PQS among all distillation methods under non-IID distributions, confirming that its predictive superiority is robust to changes in the communication cost structure.

\subsection{Client-Side Resource Efficiency}\label{sec:discussion:resource}

The resource results reveal a systematic gap in how FL methods are typically evaluated. Aggregation-based methods appear computationally lightweight because their per-round CPU overhead is low in absolute terms, but this impression is misleading along two dimensions. First, their high communication cost represents a substantial energy expenditure that CPU measurements alone do not capture. Second, their accuracy degrades sharply under non-IID conditions, meaning that the energy invested per unit of useful prediction quality increases dramatically with heterogeneity. FedKD-NAS avoids both problems simultaneously, achieving lower CPU usage, lower communication cost, and higher accuracy under the conditions most representative of real deployments. The EMNIST LeNet5 exception, where FedDistill achieves the best RES due to the communication inversion discussed above, is revealing rather than damaging, it demonstrates that the RES metric is sensitive and correctly penalizes FedKD-NAS when its communication
cost genuinely exceeds that of competing methods. On ResNet18 EMNIST under Shards, where the communication imbalance is smaller, FedKD-NAS recovers the best overall RES while simultaneously achieving the lowest CPU of all methods, illustrating that the resource efficiency advantage is real and recoverable as soon as the communication cost structure permits. Local-KD, which avoids communication cost entirely but achieves poor accuracy, yields the lowest UES among all methods across CIFAR10 configurations, confirming that resource savings without corresponding performance gains provide no deployability benefit. These observations collectively argue that evaluating FL methods on accuracy alone, or on accuracy and communication volume in isolation, is insufficient for deployment decisions on resource-constrained hardware, and that a composite metric jointly capturing all three dimensions is both necessary and discriminative.

\subsection{Hyperparameter Sensitivity and Component Contribution}
\label{sec:discussion:sensitivity}
FedKD-NAS is governed by three scalar hyperparameters whose convergence roles
are made explicit by Theorem~\ref{thm:main}. The distillation weight~$\alpha$ controls the relative contribution of the KD objective and, since term~(c) of bound~(\ref{eq:main_bound}) scales as $\alpha^2 L_Z^2(\Delta^{(r)}_{\widetilde{Z}})^2$, increasing $\alpha$
amplifies sensitivity to any residual target drift. The EMA smoothing coefficient~$\gamma$ counteracts this by driving $\Delta^{(r)}_{\widetilde{Z}}$ to decay geometrically at rate $(1-\gamma)$ per Lemma~\ref{lem:ema}, so that high $\alpha$ without correspondingly high
$\gamma$ amplifies noise rather than knowledge, while $\gamma \to 1$ trades
adaptation speed for stability. The teacher interpolation factor~$\beta$ suppresses erratic round-to-round shifts in the server aggregate under high heterogeneity, while smaller values
allow tighter tracking of client evolution in near-IID regimes. Because the three parameters are coupled, the stable operating region is approximately characterised by $\alpha\gamma < 1/(L_Z\sqrt{K})$, ensuring that terms~(c) and~(d) remain dominated by term~(a) and the $O(T^{-1/2})$ rate holds in practice. A systematic empirical sweep varying each parameter individually around the defaults ($\alpha{=}0.5$, $\beta{=}0.5$, $\gamma{=}0.9$) is an important item of future work that would provide concrete tuning guidelines for deployments
where the logit payload or round budget differs from the settings used here.
Beyond these global parameters, the existing results provide clear attribution signals for the four co-designed components. The NAS controller is responsible for the PQS gains over FedDistill and FedMD, which share identical CES but apply fixed architectures, with the gap widening on CASA and under Shards partitioning, where distributional diversity is highest.
The KD objective produces the communication reduction relative to weight-sharing methods, since replacing logit uplinks with parameter uplinks recovers FedAvg-family communication costs while retaining NAS adaptation. The server Teacher anchor provides the stability margin that makes distillation beneficial under non-IID data, as confirmed by the accuracy gap between FedKD-NAS and Local-KD widening monotonically with heterogeneity across all six benchmarks in accordance with Lemma~\ref{lem:ema}.
EMA smoothing controls round-to-round target variance so without it, term~(c) loses its geometric decay factor and $\Delta^{(r)}_{\widetilde{Z}}$ fluctuates as sharply as the raw client logit aggregate. Taken together, these signals support a contribution hierarchy in which NAS delivers the primary accuracy benefit, the Teacher anchors the primary stability
benefit, KD the primary communication benefit, and EMA modulates all three.

\subsection{Carbon Footprint and Green Federated Learning}\label{sec:green_fl}

The Green AI literature~\cite{schwartz2020green,strubell2019energy} promotes for joint reporting of model accuracy and computational cost, arguing that efficiency must be treated as a first-class metric alongside performance. FedKD-NAS implements this principle in the federated setting by design. Concretely, assuming a deployment of 10,000 edge clients running 100 training rounds on CIFAR10 MobileNetV2, FedKD-NAS's structural communication reduction translates to approximately 2.8~TB of transmission savings across the federation per experiment. Since wireless transmission energy scales approximately linearly with transmitted data~\cite{perrucci2011survey}, this represents a direct and unconditional energy saving that holds across all data distributions. On battery-powered IoT sensor deployments similar to CASA, the per-round communication reduction is particularly significant as energy budget rather than bandwidth cost is the primary constraint. Combined with the per-round CPU reduction confirmed on both
simulated and physical hardware, and with compact NAS-derived architectures that continue to reduce inference FLOPs and memory accesses after training completes, FedKD-NAS's environmental advantages are structural rather than secondary and compound over the full lifecycle of long-lived edge deployments. The one limitation is that the present analysis relies on proxy metrics rather than absolute energy
measurements. Future work should complement these comparisons with hardware power monitors or software profiling tools such as CodeCarbon~\cite{courty2023codecarbon} and  PowerJoular~\cite{noureddine2022powerjoular} to enable rigorous carbon accounting across training and inference and to position federated NAS research within the broader Green AI reporting framework.

\subsection{Pareto Optimality and Deployment Implications}\label{sec:discussion:deployment}

Real federated deployments are characterized by three conditions that co-occur by design rather than by coincidence. Data is heterogeneous across clients because devices belong to different users with different behaviors, located in different environments, and collected data from different sensor modalities. Devices are resource-constrained because the economic case for federation rests on aggregating cheap edge hardware rather than centralizing expensive compute. Class counts are typically moderate because most practical classification tasks, activity recognition, device fault detection, operate over tens to hundreds of categories rather than thousands. FedKD-NAS is specifically designed for this intersection, and the empirical evidence confirms that its advantages are strongest precisely in achieving the highest UES in 31 out of 35 evaluated settings, leads every reported metric simultaneously on CASA, and on physical hardware achieves a $6\times$ UES advantage over FedAvg.

The four exceptions to this pattern, all occurring under IID conditions or the EMNIST class-count inversion, are not anomalies but informative boundary conditions that define the deployment scope. When data is IID and task complexity is high relative to available model capacity, the global aggregation advantage of sharing a high-capacity model outweighs the adaptation benefit of NAS, and FedAvg is the appropriate choice. When class count is large enough relative to model size that the logit payload exceeds the model payload, aggregation-based methods achieve lower communication cost and FedAvg leads on
UES. These conditions are identifiable in advance from dataset statistics and model specifications, meaning we can make an informed method selection before training begins rather than discovering suboptimal performance after deployment. The modest RAM overhead of approximately $3\%$ for maintaining locally adapted model architecture descriptors on-device is the only unconditional cost of FedKD-NAS regardless of setting, and it is well justified by the accuracy and efficiency gains it enables under the heterogeneous settings that characterize the majority of real federated deployments. Taken together, these results demonstrate that FL evaluation must adopt a joint performance-communication-resource perspective, and that the UES provides a principled, reproducible, and discriminative framework for making that assessment.

\subsection{Key Trade-Offs in Federated Deployment}
\label{sec:discussion:tradeoffs}

The accuracy-communication relationship is not monotone across settings and reveals two qualitatively distinct regimes. In five of the six benchmarks (83\%), FedKD-NAS simultaneously achieves higher accuracy and lower communication volume than all aggregation-based methods. Hence, the trade-off is absent and distillation is Pareto-dominant along both dimensions. In contrast, on EMNIST the logit payload per round exceeds the model payload because the class count ($C{=}47$) is large relative to the model parameter count $P$, aggregation-based methods re-enter the communication Pareto frontier, and a genuine accuracy-communication trade-off emerges whose resolution depends on the deployment constraint, where the bandwidth-limited deployments on EMNIST-scale tasks should prefer FedAvg or FedProx. In contrast, accuracy-limited deployments should accept the higher logit traffic of FedKD-NAS. Critically, the crossover is analytically identifiable before training, since distillation achieves lower communication cost than aggregation precisely when $C \cdot |\mathcal{D}^{\mathrm{pub}}|$ bits per round is smaller than $P$, allowing practitioners to make an informed method selection from the task specification and model size alone.
The efficiency-accuracy frontier, characterised by the UES results in \tablename~\ref{tab:all_combined_quality}, reveals three distinct positions. Local-KD anchors the efficiency end by eliminating all communication cost and minimising CPU usage, but its accuracy collapses under non-IID data to the lowest UES of any method, confirming that resource savings without commensurate accuracy provide no deployment value.
FedAvg occupies the middle of the frontier under IID conditions, achieving competitive accuracy at moderate communication cost, but its efficiency advantage erodes as the rounds required for convergence grow with distributional divergence.
FedKD-NAS occupies the high-UES end under the heterogeneous conditions representative of real deployments, combining NAS-adapted accuracy, logit-based communication savings, and compact-model CPU reductions, with the $6\times$ UES advantage over FedAvg on physical hardware reported in \tablename~\ref{tab:real_devices_metrics}, confirming that this frontier position is not a simulation artefact.
The only scenario where an alternative frontier position may be preferred is the IID, high-class-count setting where FedAvg exploits its communication advantage; in all other scenarios the frontier strongly favours FedKD-NAS, and the UES provides the principled, reproducible metric needed to make that assessment rigorously.

\section{Limitations and Future Directions}
\label{sec:limitations}

Although FedKD-NAS demonstrates strong performance across seven benchmarks and achieves the best or joint-best UES in 31 out of 35 evaluated configurations, several practical limitations must be acknowledged before real-world deployment at scale.

The current protocol assumes access to a small public reference dataset $D^{\text{pub}}$ for logit exchange and distillation target aggregation. Performance may degrade when $D^{\text{pub}}$ exhibits significant domain mismatch with clients' private data, a concern particularly relevant in medical imaging or industrial sensing domains where no publicly available in-domain dataset exists. A promising remedy is to replace $D^{\text{pub}}$ with synthetic reference data generated via class-conditional sampling, feature-matching generation, or collaborative privacy-preserving generative models.

The one systematic setting where FedKD-NAS does not lead on raw accuracy, namely CIFAR100 under IID conditions, reflects a structural gap in the current NAS search space. A shared high-capacity aggregated model can leverage the full global dataset more effectively than locally adapted compact architectures under these conditions. Future work should investigate hierarchical NAS strategies that dynamically expand architecture capacity as task complexity grows. A related limitation emerges on EMNIST, where the 47-class logit payload exceeds model size and inverts the usual communication advantage of distillation-based methods. Future work should combine communication-efficient federated optimization~\cite{chaimaa2025fedsparq} with the current framework to recover the bandwidth advantage in high-class-count settings without sacrificing the predictive quality gains demonstrated here.

On the privacy front, although exchanging only soft predictions inherently reduces the attack surface compared to gradient sharing methods, prediction vectors remain vulnerable to membership inference and model inversion attacks under strong adversarial threat models~\cite{shokri2017membership}. Future extensions will incorporate differentially private prediction sharing with calibrated noise and evaluate the accuracy-privacy trade-off across the heterogeneity regimes studied here. In open or semi-trusted federations, malicious clients may submit adversarial predictions to corrupt the aggregated teacher signal. While the exponential smoothing aggregation of Section~\ref{sec:FedKD-NAS:Archi:server_agg} provides empirical resilience, formal Byzantine robustness guarantees are absent. Future work will integrate coordinate-wise trimmed mean or geometric median aggregation rules and benchmark these under increasing fractions of malicious participants.

Finally, the Green FL analysis of Section~\ref{sec:green_fl} establishes FedKD-NAS's environmental advantages through proxy metrics but does not provide absolute carbon estimates. Future work should complement these comparisons with direct energy measurements via hardware power monitors or tools such as CodeCarbon~\cite{courty2023codecarbon} and PowerJoular~\cite{noureddine2022powerjoular}, enabling rigorous carbon accounting across the full training and inference lifecycle and positioning FedKD-NAS within the broader Green AI reporting framework~\cite{schwartz2020green, strubell2019energy}.

\section{Conclusion}
\label{sec:conclusion}

This paper identifies architectural inflexibility as a fundamental limitation of
existing federated learning methods under statistical and system heterogeneity.
To address this, we proposed FedKD-NAS, a framework that couples lightweight
client-side neural architecture search with server-coordinated knowledge
distillation. Each client dynamically adapts its model capacity to its local data
distribution, while global alignment is maintained through prediction-level
communication alone, eliminating dependence on shared model architectures and full weight transmission.

To enable multi-objective evaluation beyond accuracy, we introduced four composite
metrics (the Resource Efficiency Score (RES), Performance Quality Score (PQS), Communication Efficiency Score (CES), and Unified Efficiency Score (UES)) that jointly quantify predictive quality, communication cost, and computational resource usage. These metrics surface deployment-relevant differences that accuracy-only evaluations obscure, and we argued they should become standard reporting practice
in federated learning benchmarks.

We conducted extensive experiments across seven benchmarks (MNIST, FMNIST, EMNIST, CIFAR10, CIFAR100, CASA, and a real-device deployment), two backbone architectures per dataset, and three data partitioning regimes (IID, Dirichlet $\alpha{=}0.1$, and Shards) demonstrated that FedKD-NAS achieves the highest or joint-highest UES in \textbf{31 out of 35 evaluated configurations}. Under non-IID conditions, FedKD-NAS limits accuracy degradation compared to aggregation-based baselines, while simultaneously reducing communication volume and client CPU utilization to FedAvg. On the CASA benchmark, FedKD-NAS achieves a $183\times$ UES gap over FedAvg, confirming that its advantages generalize beyond vision to time-series and sensor domains with natural cross-subject heterogeneity. These gains are validated on physical hardware with heterogeneous edge nodes, confirming that FedKD-NAS achieved the highest accuracy, lowest CPU utilization, and a UES $6\times$ higher than FedAvg. A formal convergence analysis provides theoretical grounding for the observed training stability, with the bounded teacher drift assumption and NAS approximation bound jointly explaining why FedKD-NAS maintains and often
improves convergence quality as heterogeneity increases.

Critically, none of these improvements arises from increased communication or additional client-side computation. They emerge from more effective resource utilisation through adaptive architecture selection and communication-efficient distillation, making FedKD-NAS the only evaluated method that is simultaneously
\emph{communication-green}, \emph{computation-green}, and \emph{convergence-efficient}. As federated learning scales to large deployments of energy-constrained IoT and edge devices, this triply-green property positions
FedKD-NAS as a natural candidate for sustainable AI systems where bandwidth,
compute, and carbon footprint are operational constraints alongside accuracy.

Future work will pursue five directions: replacing the public reference dataset with synthetic or collaboratively generated data to relax data availability assumptions, incorporating differentially private prediction sharing to provide formal privacy guarantees, extending robust aggregation to defend against Byzantine clients in open federations, complementing the proxy efficiency metrics introduced here with direct energy and carbon measurements to enable rigorous Green FL reporting under the standards advocated by the broader sustainable AI community and also integrating communication-efficient solution to recover the communication advantage of distillation.
\section*{Acknowledgment}
This work was funded in whole or in part by the Luxembourg National Research Fund (FNR) LightGridSEED Project, ref. C21/IS/16215802/LightGridSEED.

\appendix

\section{Proof of Lemma~\ref{lem:ema}}
\label{app:proof_lemma}

Subtracting consecutive iterates of the EMA recurrence~\eqref{eq:ema} gives
\begin{align}
    \tilde{\mathbf{Z}}^{(r)} - \tilde{\mathbf{Z}}^{(r-1)}
    & = \gamma\tilde{\mathbf{Z}}^{(r-1)} + (1-\gamma)\mathbf{Z}^{(r)} - \tilde{\mathbf{Z}}^{(r-1)} 
    \\
    & = (1-\gamma)\bigl(\mathbf{Z}^{(r)} - \tilde{\mathbf{Z}}^{(r-1)}\bigr),
\end{align}

establishing the first equality. To obtain the upper bound, we expand the gap $\mathbf{Z}^{(r)} - \tilde{\mathbf{Z}}^{(r-1)}$:
\[
\mathbf{Z}^{(r)} - \tilde{\mathbf{Z}}^{(r-1)}
= \bigl(\mathbf{Z}^{(r)} - \mathbf{Z}^{(r-1)}\bigr) + \gamma\bigl(\mathbf{Z}^{(r-1)} - \tilde{\mathbf{Z}}^{(r-2)}\bigr),
\]
and iterating under $\tilde{\mathbf{Z}}^{(0)} = \mathbf{Z}^{(0)}$ yields
$\|\mathbf{Z}^{(r)} - \tilde{\mathbf{Z}}^{(r-1)}\|_F \le \sum_{k=0}^{r-1}\gamma^k \Delta_Z^{(r-k)}$.
If $\Delta_Z^{(r)} \le \delta_Z$ for all $r$, the geometric series gives
$\Delta_{\tilde{Z}}^{(r)} \le (1-\gamma)\delta_Z \cdot \frac{1-\gamma^r}{1-\gamma} = (1-\gamma^r)\delta_Z \le \delta_Z$.~\hfill$\square$

\section{Proof of Lemma~\ref{lem:multiclient}}
\label{proof_multi}
By the independence of per-client logit noise across clients and the bias-variance decomposition:
\begin{align}
  \mathbb{E}\!\left[\left\|\mathbf{Z}^{(r)} - \mathbf{Z}^*\right\|_F^2\right]
  &= \mathbb{E}\!\left[\left\|\frac{1}{K}\sum_{k=1}^K
     \bigl(z_k^{(r)} - \mathbf{Z}^*\bigr)\right\|_F^2\right] \notag \\
  &= \frac{1}{K^2}\sum_{k=1}^K
     \mathbb{E}\!\left[\left\|z_k^{(r)} - \mathbf{Z}^*\right\|_F^2\right],
\end{align}
where cross-terms vanish by independence of client noise. Decomposing each $\|z_k^{(r)} - \mathbf{Z}^*\|_F^2$ into variance and squared bias and summing over $k$ gives~\eqref{eq:agg_error}. Substituting into the EMA bound of Lemma~\ref{lem:ema} via $\Delta_{\tilde{Z}}^{(r)} \leq (1-\gamma)\|\mathbf{Z}^{(r)} - \tilde{\mathbf{Z}}^{(r-1)}\|_F$ and taking expectations gives~\eqref{eq:smoothed_agg}.

\section{Discussion and Practical Implications}
\label{subsec:conv_discussion}

Theorem~\ref{thm:main} decomposes the convergence behavior of FedKD-NAS into four interpretable terms, each tied to a concrete aspect of the framework. We discuss the practical implications of each.

\paragraph{Optimization term (a)}
The first term $\mathcal{O}(\eta^{-1} T^{-1})$ is the standard nonconvex SGD rate and vanishes as $T \to \infty$ at the rate $\mathcal{O}(T^{-1/2})$ under the prescribed step size schedule. This term is identical to what one would obtain for a single-client nonconvex SGD baseline, confirming that FedKD-NAS does not introduce any overhead in the optimization term relative to local training.

\paragraph{Role of EMA smoothing ($\gamma$)}
Term (c) shows that the principal deviation from standard SGD stationarity is governed by the EMA drift $\mathbb{E}[(\Delta_{\tilde{Z}}^{(r)})^2]$. By Lemma~\ref{lem:ema}, a larger $\gamma$ (stronger smoothing) attenuates high-frequency target fluctuations and directly tightens this term, improving convergence stability, particularly under non-IID data distributions where client predictions may fluctuate significantly across rounds.

\paragraph{Role of teacher-client fusion ($\beta^{(r)}$)}
The raw target $\mathbf{Z}^{(r)} = \beta^{(r)}\mathbf{P}_T + (1-\beta^{(r)})\mathbf{P}_{\mathrm{agg}}^{(r)}$ interpolates between a stable teacher signal and an adaptive client consensus. A larger $\beta^{(r)}$ reduces the round-to-round variability of $\mathbf{Z}^{(r)}$, hence decreasing $\Delta_Z^{(r)}$ and, by Lemma~\ref{lem:ema}, $\Delta_{\tilde{Z}}^{(r)}$. This strengthens the bound in term (c) at the cost of reduced adaptivity to evolving client representations. The two hyperparameters $\gamma$ and $\beta^{(r)}$ thus provide complementary and theoretically grounded controls over the stability-adaptivity trade-off.

\paragraph{Effect of architecture heterogeneity and NAS (d)}
Unlike parameter-aggregation methods, FedKD-NAS is architecture-agnostic at the communication layer, exchanging only logits on $D^{\mathrm{pub}}$. Architecture variation is entirely isolated into term (d), $\sum_r \Delta_A^2$, which remains negligible when the NAS controller selects from a bounded lightweight family without frequent inter-round oscillations( which is consistent with the design of Algorithm~\ref{alg:nas}). Importantly, term (d) does not depend on the degree of data heterogeneity, confirming that architecture diversity does not amplify client drift in the FedKD-NAS framework.

\paragraph{Scalability to Large $K$}
Corollary~\ref{cor:multiclient} provides a formal scalability guarantee absent from a single-client analysis. Unlike FedAvg, whose gradient aggregation error grows with client heterogeneity independently of $K$, FedKD-NAS benefits from a variance reduction of $\mathcal{O}(1/K)$ in the distillation target as more clients participate. This is because logit averaging in prediction space is less susceptible to client drift than gradient averaging in parameter space: the softmax operator compresses inter-client divergence, so $\kappa_z^2 \leq \kappa^2$, where $\kappa^2$ is the gradient heterogeneity constant of FedAvg~\cite{li2020federated}. In practice, this implies that FedKD-NAS's convergence quality improves, or at worst remains stable, as the federation scales to larger $K$, a property empirically corroborated by the consistent UES gains observed under non-IID distributions across all evaluated datasets (Table~\ref{tab:all_combined_quality}).

\paragraph{Summary of Convergence Behavior}
Under standard smoothness and bounded-variance assumptions, FedKD-NAS converges to a first-order stationary point at rate $\mathcal{O}(T^{-1/2})$, up to additive terms capturing distillation target drift and architecture switching. These additional terms arise explicitly from the prediction-based communication mechanism and heterogeneous architecture selection.
Both effects are controlled by algorithmic design choices: EMA smoothing ($\gamma$) and teacher fusion ($\beta^{(r)}$) reduce target drift, while the constrained NAS search space bounds architecture variation. Consequently, convergence stability is not
incidental but directly tied to these mechanisms.

\section{Proof of Theorem~\ref{thm:main}}
\label{app:proof_main}

\begin{proof}
We prove the result in three steps.

\paragraph{Step 1: One-step descent inequality.}

Fix a client $i$ and round $r$, and consider one local SGD step
\[
\theta^+ = \theta - \eta g,
\]
where $g$ is an unbiased stochastic gradient of $\mathcal{L}_i^{(r)}$ evaluated at $\theta$.

By $L$-smoothness (Assumption A1),
\begin{equation}
\mathcal{L}_i^{(r)}(\theta^+) 
\le 
\mathcal{L}_i^{(r)}(\theta)
+ \langle \nabla \mathcal{L}_i^{(r)}(\theta), \theta^+ - \theta \rangle
+ \frac{L}{2} \|\theta^+ - \theta\|^2.
\end{equation}
Substituting $\theta^+ - \theta = -\eta g$ yields
\begin{equation}
\mathcal{L}_i^{(r)}(\theta^+)
\le 
\mathcal{L}_i^{(r)}(\theta)
- \eta \langle \nabla \mathcal{L}_i^{(r)}(\theta), g \rangle
+ \frac{L \eta^2}{2} \|g\|^2.
\end{equation}

Taking conditional expectation given $\theta$ and using unbiasedness (A2),
\[
\mathbb{E}[g \mid \theta] = \nabla \mathcal{L}_i^{(r)}(\theta),
\]
we obtain
\begin{equation}
\mathbb{E}\big[\mathcal{L}_i^{(r)}(\theta^+)\big]
\le
\mathbb{E}\big[\mathcal{L}_i^{(r)}(\theta)\big]
- \eta \|\nabla \mathcal{L}_i^{(r)}(\theta)\|^2
+ \frac{L\eta^2}{2} \mathbb{E}\big[\|g\|^2\big].
\end{equation}

Decomposing the second moment via bounded variance (A2),
\[
\mathbb{E}[\|g\|^2]
=
\|\nabla \mathcal{L}_i^{(r)}(\theta)\|^2
+ \mathbb{E}[\|g - \nabla \mathcal{L}_i^{(r)}(\theta)\|^2]
\le
\|\nabla \mathcal{L}_i^{(r)}(\theta)\|^2 + \sigma^2,
\]
gives
\begin{equation}
\mathbb{E}[\mathcal{L}_i^{(r)}(\theta^+)]
\le
\mathbb{E}[\mathcal{L}_i^{(r)}(\theta)]
- \eta\Big(1 - \frac{L\eta}{2}\Big)
\|\nabla \mathcal{L}_i^{(r)}(\theta)\|^2
+ \frac{L\eta^2}{2}\sigma^2.
\end{equation}

Using $\eta \le 1/L$ implies $1 - L\eta/2 \ge 1/2$, hence
\begin{equation}
\label{eq:descent_final}
\mathbb{E}[\mathcal{L}_i^{(r)}(\theta^+)]
\le
\mathbb{E}[\mathcal{L}_i^{(r)}(\theta)]
- \frac{\eta}{2}
\|\nabla \mathcal{L}_i^{(r)}(\theta)\|^2
+ \frac{L\eta^2}{2}\sigma^2.
\end{equation}

\paragraph{Step 2: Inter-round objective mismatch.}

Because the objective changes across rounds through both the
distillation target $\tilde{\mathbf{Z}}^{(r)}$ and the architecture $a_i^{(r)}$,
we bound the gradient discrepancy between consecutive rounds.

Adding and subtracting the KD gradient evaluated at the previous
target and architecture, and applying Assumptions A5 and A6,
we obtain
\begin{equation}
\|\nabla \mathcal{L}_i^{(r)}(\theta)
-
\nabla \mathcal{L}_i^{(r-1)}(\theta)\|^2
\le
2\alpha^2 L_Z^2
(\Delta_{\tilde{Z}}^{(r)})^2
+
2\Delta_A^2.
\end{equation}

This controls the perturbation induced by both target drift and
architecture switching.

\paragraph{Step 3: Telescoping over all updates.}

Summing inequality~\eqref{eq:descent_final}
over $t=0,\dots,E-1$ and $r=1,\dots,R$ yields
\begin{align}
\mathbb{E}[\mathcal{L}_i^{(R)}(\theta_{i,E}^{(R)})]
&\le
\mathbb{E}[\mathcal{L}_i^{(1)}(\theta_{i,0}^{(1)})]
- \frac{\eta}{2}
\sum_{r,t}
\mathbb{E}\big[\|\nabla \mathcal{L}_i^{(r)}(\theta_{i,t}^{(r)})\|^2\big]
\\
&\quad
+ \frac{L\eta^2}{2}\sigma^2 T
+ \sum_{r=2}^{R}
\big(
\alpha^2 L_Z^2 (\Delta_{\tilde{Z}}^{(r)})^2
+ \Delta_A^2
\big).
\end{align}

Using the lower bound assumption (A3),
\[
\mathcal{L}_i^{(R)}(\theta_{i,E}^{(R)})
\ge \mathcal{L}_{\inf},
\]
and rearranging terms gives
\begin{equation}
\begin{split}
\frac{1}{T}
\sum_{r,t}
\mathbb{E}\!\left[
\|\nabla \mathcal{L}_i^{(r)}(\theta_{i,t}^{(r)})\|^2
\right]
\le
& \frac{2(\mathbb{E}[\mathcal{L}_i^{(1)}(\theta_{i,0}^{(1)})]
-
\mathcal{L}_{\inf})}{\eta T}
+
L\eta\sigma^2
+ \\
& \frac{2}{T}
\sum_{r=2}^{R}
\Delta_A^2
+
2\alpha^2 L_Z^2
\frac{1}{T}
\sum_{r=2}^{R}
\mathbb{E}\!\left[(\Delta_{\tilde{Z}}^{(r)})^2\right.
\end{split}
\end{equation}

This establishes~\eqref{eq:main_bound}. 

Setting $\eta = \Theta(T^{-1/2})$ yields the rate in~\eqref{eq:rate}.
\end{proof}

\section{Additional Figures for All Datasets}
\label{app:additional_figures}

This appendix presents the complete set of evaluation figures for MNIST, FMNIST, EMNIST, and CIFAR100, following the same structure as the CIFAR10 analysis in Section~\ref{sec:Exp}.

Each dataset includes: (i) accuracy curves over 100 rounds for CIFAR100 and 30 rounds for the other datasets, with inset zoom; (ii) accuracy-drop plots; (iii) loss curves; (iv) CES vs.\ PQS bubble trade-off plots, where bubble area is proportional to $1/\mathrm{RES}$; (v) RES and UES bar charts; and (vi) communication-cost bar charts. Results are reported for both architectures per dataset (LeNet5 and ResNet18 for MNIST, FMNIST, and EMNIST; MobileNetV2 and ShuffleNetV2 for CIFAR100).

\begin{figure}[htbp]
\centering

\begin{tabular}{cccccccc}
  \raisebox{0.2ex}{\tikz\draw[fill=FedAvgColor,    draw=black](0,0)rectangle(0.18,0.18);} & \small FedAvg     &
  \raisebox{0.2ex}{\tikz\draw[fill=DittoColor,     draw=black](0,0)rectangle(0.18,0.18);} & \small Ditto      &
  \raisebox{0.2ex}{\tikz\draw[fill=FedDFColor,     draw=black](0,0)rectangle(0.18,0.18);} & \small FedDF      &
  \raisebox{0.2ex}{\tikz\draw[fill=FedDistillColor,draw=black](0,0)rectangle(0.18,0.18);} & \small FedDistill \\
  \raisebox{0.2ex}{\tikz\draw[fill=FedMDColor,     draw=black](0,0)rectangle(0.18,0.18);} & \small FedMD      &
  \raisebox{0.2ex}{\tikz\draw[fill=FedKDNASColor,  draw=black](0,0)rectangle(0.18,0.18);} & \small FedKD-NAS  &
  \raisebox{0.2ex}{\tikz\draw[fill=LocalKDColor,   draw=black](0,0)rectangle(0.18,0.18);} & \small Local-KD   \\
\end{tabular}
\vspace{0.6em}

\begin{tikzpicture}
\begin{axis}[
    name=mainaxis,
    width  = 0.9\linewidth,
    height = 3.5cm,
    xmin=-1.4, xmax=4.8,
    xtick=\empty,
    clip=false,
    axis y line       = left,
    axis x line       = bottom,
    ymin=0, ymax=14000000,
    ylabel            = {Comm.\ Cost (bytes)},
    ylabel style      = {font=\scriptsize},
    yticklabel style  = {font=\scriptsize},
    enlarge x limits  = false,
    every axis plot/.append style={mark=none},
    bar width=0.18,
]

\addplot+[FedAvgStyle]     coordinates {(-0.75,  8922592)};
\addplot+[DittoStyle]      coordinates {(-0.50,  8922592)};
\addplot+[FedDFStyle]      coordinates {(-0.25, 13181920)};
\addplot+[FedDistillStyle] coordinates {( 0.00,  1617920)};
\addplot+[FedMDStyle]      coordinates {( 0.25,  1617920)};
\addplot+[FedKDNASStyle]   coordinates {( 0.50,  1617920)};
\addplot+[LocalKDStyle]    coordinates {( 0.75,  8922592)};

\draw[decorate, decoration={brace, amplitude=4pt, mirror}, thin]
  ([yshift=-10pt]axis cs:-0.75,0) -- ([yshift=-10pt]axis cs:0.75,0)
  node[midway, below=5pt, font=\scriptsize]{LeNet5};

\end{axis}

\begin{axis}[
    at={(mainaxis.south west)},
    anchor=south west,
    width  = 0.9\linewidth,
    height = 3.5cm,
    xmin=-1.4, xmax=4.8,
    xtick=\empty,
    clip=false,
    axis y line       = right,
    axis x line       = none,
    ymin=0, ymax=19000000,
    ylabel            = {Comm.\ Cost (bytes)},
    ylabel style      = {font=\scriptsize},
    yticklabel style  = {font=\scriptsize},
    enlarge x limits  = false,
    every axis plot/.append style={mark=none},
    bar width=0.18,
]

\addplot+[FedAvgStyle]     coordinates {(2.45, 13645280)};
\addplot+[DittoStyle]      coordinates {(2.70, 13645280)};
\addplot+[FedDFStyle]      coordinates {(2.95, 17904608)};
\addplot+[FedDistillStyle] coordinates {(3.20,  1617920)};
\addplot+[FedMDStyle]      coordinates {(3.45,  1617920)};
\addplot+[FedKDNASStyle]   coordinates {(3.70,  1617920)};
\addplot+[LocalKDStyle]    coordinates {(3.95, 13645280)};

\draw[decorate, decoration={brace, amplitude=4pt, mirror}, thin]
  ([yshift=-10pt]axis cs:2.45,0) -- ([yshift=-10pt]axis cs:3.95,0)
  node[midway, below=5pt, font=\scriptsize]{ResNet18};

\end{axis}
\end{tikzpicture}

\caption{Communication cost per round on MNIST. For both LeNet5 and ResNet18, FedDistill, FedMD, and FedKD-NAS incur a substantially lower communication cost (1,617,920 bytes) compared to full-model exchange methods such as FedAvg, Ditto, and Local-KD. FedDF exhibits the highest communication cost among all methods.}
\label{fig:mnist_comm}
\end{figure}
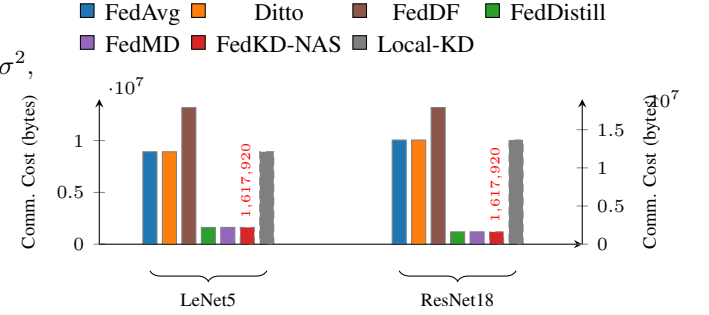
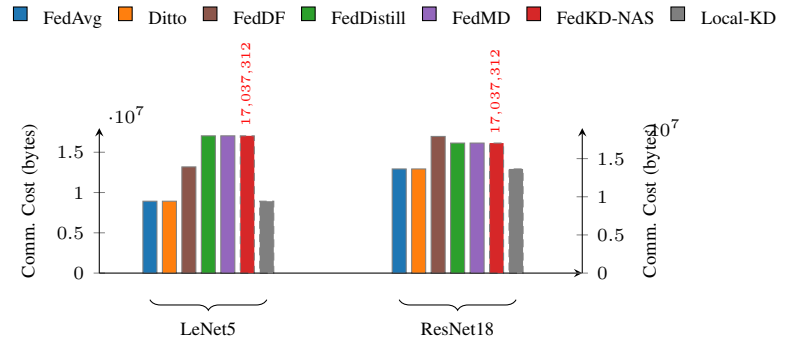
\begin{figure}[htbp]
\centering

\setlength{\tabcolsep}{3pt}
\begin{tabular}{@{}cccccccccccccc@{}}
  \raisebox{0.2ex}{\tikz\draw[fill=FedAvgColor,draw=black](0,0)rectangle(0.18,0.18);} & \scriptsize FedAvg &
  \raisebox{0.2ex}{\tikz\draw[fill=DittoColor,draw=black](0,0)rectangle(0.18,0.18);} & \scriptsize Ditto &
  \raisebox{0.2ex}{\tikz\draw[fill=FedDFColor,draw=black](0,0)rectangle(0.18,0.18);} & \scriptsize FedDF &
  \raisebox{0.2ex}{\tikz\draw[fill=FedDistillColor,draw=black](0,0)rectangle(0.18,0.18);} & \scriptsize FedDistill &
  \raisebox{0.2ex}{\tikz\draw[fill=FedMDColor,draw=black](0,0)rectangle(0.18,0.18);} & \scriptsize FedMD &
  \raisebox{0.2ex}{\tikz\draw[fill=FedKDNASColor,draw=black](0,0)rectangle(0.18,0.18);} & \scriptsize FedKD-NAS &
  \raisebox{0.2ex}{\tikz\draw[fill=LocalKDColor,draw=black](0,0)rectangle(0.18,0.18);} & \scriptsize Local-KD \\
\end{tabular}

\begin{tikzpicture}
\begin{axis}[
    name=mainaxis,
    width=0.9\linewidth,
    height=3.5cm,
    xmin=-1.4, xmax=4.8,
    xtick=\empty,
    clip=false,
    axis y line=left,
    axis x line=bottom,
    ymin=0, ymax=18000000,
    ylabel={Comm.\ Cost (bytes)},
    ylabel style={font=\scriptsize},
    yticklabel style={font=\scriptsize},
    enlarge x limits=false,
    every axis plot/.append style={mark=none},
    bar width=0.18,
]
\addplot+[FedAvgStyle]     coordinates {(-0.75,  8922592)};
\addplot+[DittoStyle]      coordinates {(-0.50,  8922592)};
\addplot+[FedDFStyle]      coordinates {(-0.25, 13181920)};
\addplot+[FedDistillStyle] coordinates {( 0.00, 17037312)};
\addplot+[FedMDStyle]      coordinates {( 0.25, 17037312)};
\addplot+[FedKDNASStyle]   coordinates {( 0.50, 17037312)};
\addplot+[LocalKDStyle]    coordinates {( 0.75,  8922592)};
\draw[decorate, decoration={brace, amplitude=4pt, mirror}, thin]
([yshift=-10pt]axis cs:-0.75,0) -- ([yshift=-10pt]axis cs:0.75,0)
node[midway, below=5pt, font=\scriptsize]{LeNet5};
\end{axis}

\begin{axis}[
    at={(mainaxis.south west)},
    anchor=south west,
    width=0.9\linewidth,
    height=3.5cm,
    xmin=-1.4, xmax=4.8,
    xtick=\empty,
    clip=false,
    axis y line=right,
    axis x line=none,
    ymin=0, ymax=19000000,
    ylabel={Comm.\ Cost (bytes)},
    ylabel style={font=\scriptsize},
    yticklabel style={font=\scriptsize},
    enlarge x limits=false,
    every axis plot/.append style={mark=none},
    bar width=0.18,
]
\addplot+[FedAvgStyle]     coordinates {(2.45, 13645280)};
\addplot+[DittoStyle]      coordinates {(2.70, 13645280)};
\addplot+[FedDFStyle]      coordinates {(2.95, 17904608)};
\addplot+[FedDistillStyle] coordinates {(3.20, 17037312)};
\addplot+[FedMDStyle]      coordinates {(3.45, 17037312)};
\addplot+[FedKDNASStyle]   coordinates {(3.70, 17037312)};
\addplot+[LocalKDStyle]    coordinates {(3.95, 13645280)};
\draw[decorate, decoration={brace, amplitude=4pt, mirror}, thin]
([yshift=-10pt]axis cs:2.45,0) -- ([yshift=-10pt]axis cs:3.95,0)
node[midway, below=5pt, font=\scriptsize]{ResNet18};
\end{axis}
\end{tikzpicture}

\caption{Communication cost per round on EMNIST. For LeNet5, FedAvg, Ditto, and Local-KD incur the lowest communication cost (8,922,592 bytes), followed by FedDF (13,181,920 bytes), while FedDistill, FedMD, and FedKD-NAS have the highest cost (17,037,312 bytes). For ResNet18, FedAvg, Ditto, and Local-KD again have the lowest cost (13,645,280 bytes), FedDistill/FedMD/FedKD-NAS require 17,037,312 bytes, and FedDF has the highest cost (17,904,608 bytes).}
\label{fig:emnist_comm}
\end{figure}
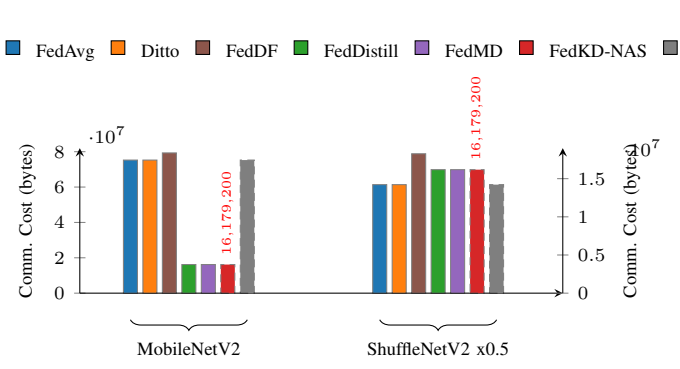
\begin{figure}[htbp]
\centering

\setlength{\tabcolsep}{3pt}
\begin{tabular}{@{}cccccccccccccc@{}}
  \raisebox{0.2ex}{\tikz\draw[fill=FedAvgColor,draw=black](0,0)rectangle(0.18,0.18);} & \scriptsize FedAvg &
  \raisebox{0.2ex}{\tikz\draw[fill=DittoColor,draw=black](0,0)rectangle(0.18,0.18);} & \scriptsize Ditto &
  \raisebox{0.2ex}{\tikz\draw[fill=FedDFColor,draw=black](0,0)rectangle(0.18,0.18);} & \scriptsize FedDF &
  \raisebox{0.2ex}{\tikz\draw[fill=FedDistillColor,draw=black](0,0)rectangle(0.18,0.18);} & \scriptsize FedDistill &
  \raisebox{0.2ex}{\tikz\draw[fill=FedMDColor,draw=black](0,0)rectangle(0.18,0.18);} & \scriptsize FedMD &
  \raisebox{0.2ex}{\tikz\draw[fill=FedKDNASColor,draw=black](0,0)rectangle(0.18,0.18);} & \scriptsize FedKD-NAS &
  \raisebox{0.2ex}{\tikz\draw[fill=LocalKDColor,draw=black](0,0)rectangle(0.18,0.18);} & \scriptsize Local-KD \\
\end{tabular}

\begin{tikzpicture}
\begin{axis}[
    name=mainaxis,
    width=0.9\linewidth,
    height=3.5cm,
    xmin=-1.4, xmax=4.8,
    xtick=\empty,
    clip=false,
    axis y line=left,
    axis x line=bottom,
    ymin=0, ymax=82000000,
    ylabel={Comm.\ Cost (bytes)},
    ylabel style={font=\scriptsize},
    yticklabel style={font=\scriptsize},
    enlarge x limits=false,
    every axis plot/.append style={mark=none},
    bar width=0.18,
]
\addplot+[FedAvgStyle]     coordinates {(-0.75, 75263104)};
\addplot+[DittoStyle]      coordinates {(-0.50, 75263104)};
\addplot+[FedDFStyle]      coordinates {(-0.25, 79307904)};
\addplot+[FedDistillStyle] coordinates {( 0.00, 16179200)};
\addplot+[FedMDStyle]      coordinates {( 0.25, 16179200)};
\addplot+[FedKDNASStyle]   coordinates {( 0.50, 16179200)};
\addplot+[LocalKDStyle]    coordinates {( 0.75, 75263104)};
\draw[decorate, decoration={brace, amplitude=4pt, mirror}, thin]
([yshift=-10pt]axis cs:-0.75,0) -- ([yshift=-10pt]axis cs:0.75,0)
node[midway, below=5pt, font=\scriptsize]{MobileNetV2};
\end{axis}

\begin{axis}[
    at={(mainaxis.south west)},
    anchor=south west,
    width=0.9\linewidth,
    height=3.5cm,
    xmin=-1.4, xmax=4.8,
    xtick=\empty,
    clip=false,
    axis y line=right,
    axis x line=none,
    ymin=0, ymax=19000000,
    ylabel={Comm.\ Cost (bytes)},
    ylabel style={font=\scriptsize},
    yticklabel style={font=\scriptsize},
    enlarge x limits=false,
    every axis plot/.append style={mark=none},
    bar width=0.18,
]
\addplot+[FedAvgStyle]     coordinates {(2.45, 14217344)};
\addplot+[DittoStyle]      coordinates {(2.70, 14217344)};
\addplot+[FedDFStyle]      coordinates {(2.95, 18262144)};
\addplot+[FedDistillStyle] coordinates {(3.20, 16179200)};
\addplot+[FedMDStyle]      coordinates {(3.45, 16179200)};
\addplot+[FedKDNASStyle]   coordinates {(3.70, 16179200)};
\addplot+[LocalKDStyle]    coordinates {(3.95, 14217344)};
\draw[decorate, decoration={brace, amplitude=4pt, mirror}, thin]
([yshift=-10pt]axis cs:2.45,0) -- ([yshift=-10pt]axis cs:3.95,0)
node[midway, below=5pt, font=\scriptsize]{ShuffleNetV2 x0.5};
\end{axis}
\end{tikzpicture}

\caption{Communication cost per round on CIFAR100 under the Dirichlet setting. Since communication cost depends on the exchanged model or logit payload rather than on per-round accuracy, one bar is shown per method and architecture. For MobileNetV2, FedKD-NAS, FedDistill, and FedMD incur a communication cost of 16,179,200 bytes, corresponding to an approximately $4.65\times$ reduction relative to FedAvg, Ditto, and Local-KD (75,263,104 bytes), while FedDF has the highest cost (79,307,904 bytes). For ShuffleNetV2 x0.5, the logit-based methods require 16,179,200 bytes, which is higher than FedAvg, Ditto, and Local-KD (14,217,344 bytes) but lower than FedDF (18,262,144 bytes).}
\label{fig:cifar100_comm}
\end{figure}
\begin{figure}
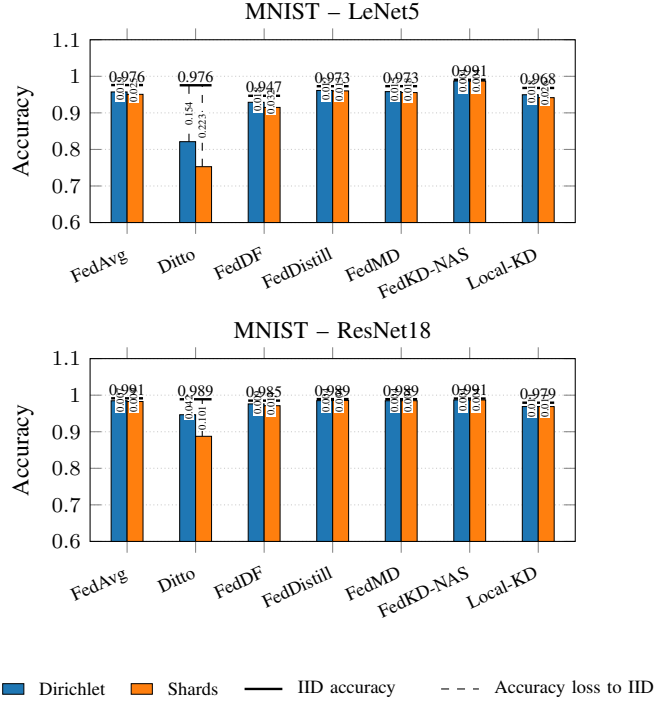

\centering

\caption{Accuracy drop from IID to non-IID on MNIST. FedKD-NAS and FedDistill exhibit the smallest degradation across both architectures. Ditto and Local-KD suffer the largest drops, consistent with CIFAR-10 findings.}
\label{fig:mnist_acc_drop}
\end{figure}
\begin{figure}
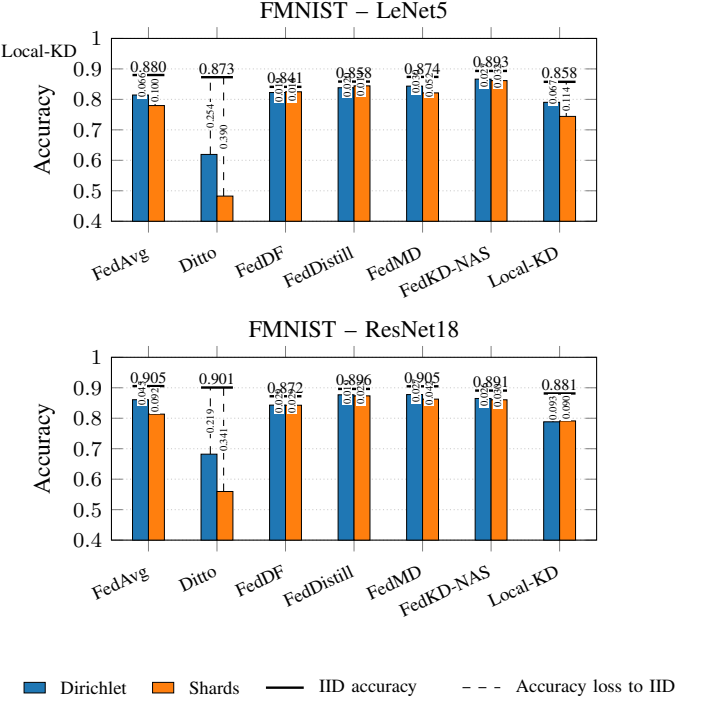

\centering
%
\caption{Accuracy drop from IID to non-IID on FMNIST. FedKD-NAS sustains minimal degradation on LeNet5 ($<4\%$). On ResNet18, Ditto collapses by $>34\%$ under Shards while FedKD-NAS and FedDistill remain stable.}
\label{fig:fmnist_acc_drop}
\end{figure}
\begin{figure}
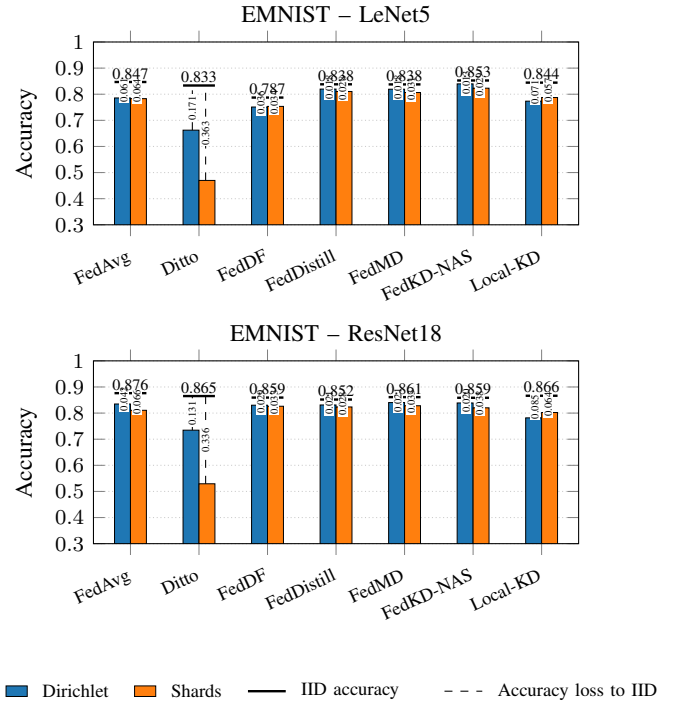

\centering
%
\caption{Accuracy drop from IID to non-IID on EMNIST. On LeNet5, FedDistill achieves the smallest drop while FedKD-NAS remains in the low-degradation group. On ResNet18, FedKD-NAS limits its drop to less than $5\%$ under Shards, whereas Ditto collapses by more than $35\%$.}
\label{fig:emnist_acc_drop}
\end{figure}
\begin{figure}
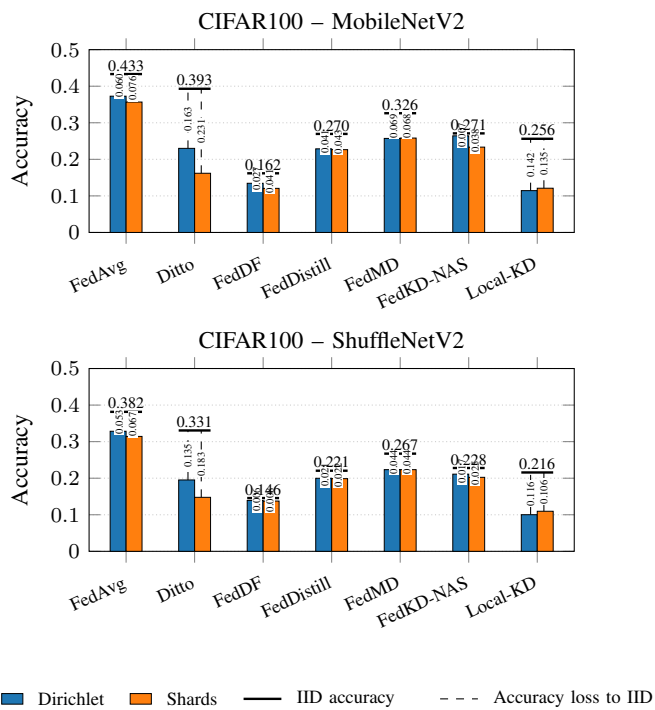

\centering
%
\caption{Accuracy drop from IID to non-IID on CIFAR100. All methods degrade substantially on this harder task. FedKD-NAS and FedAvg exhibit the most contained degradation, whereas Ditto and Local-KD again suffer the largest collapses.}
\label{fig:cifar100_acc_drop}
\end{figure}

\begin{figure*}[htbp]
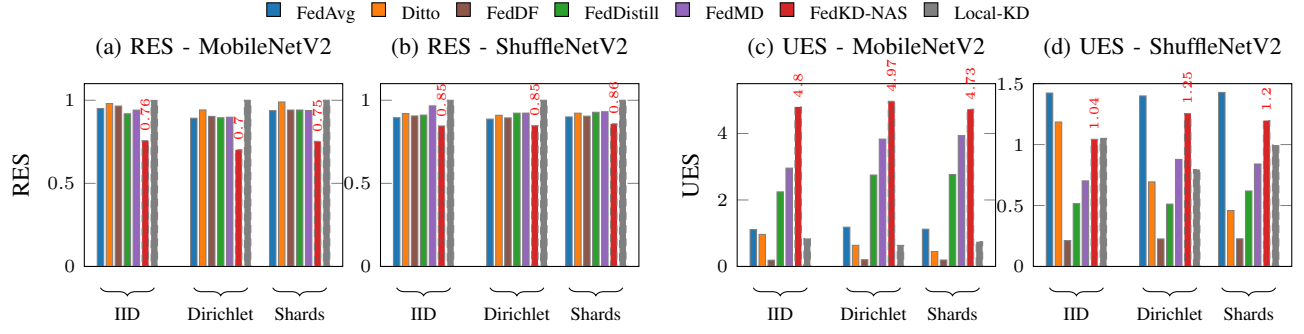

\centering

\setlength{\tabcolsep}{3pt}
%
%
}
\caption{RES ($\downarrow$) and UES ($\uparrow$) on CIFAR100. FedKD-NAS achieves the lowest RES on both architectures under all three data distributions. For UES, FedKD-NAS attains the highest values on MobileNetV2 across all three data distributions. On ShuffleNetV2, FedAvg has the highest overall UES, while FedKD-NAS remains the strongest method within the low-communication logit-based group, achieving 1.0432, 1.2543, and 1.1971 under IID, Dirichlet, and Shards.}
\label{fig:cifar100_res_ues}
\end{figure*}
\begin{figure*}[htbp]
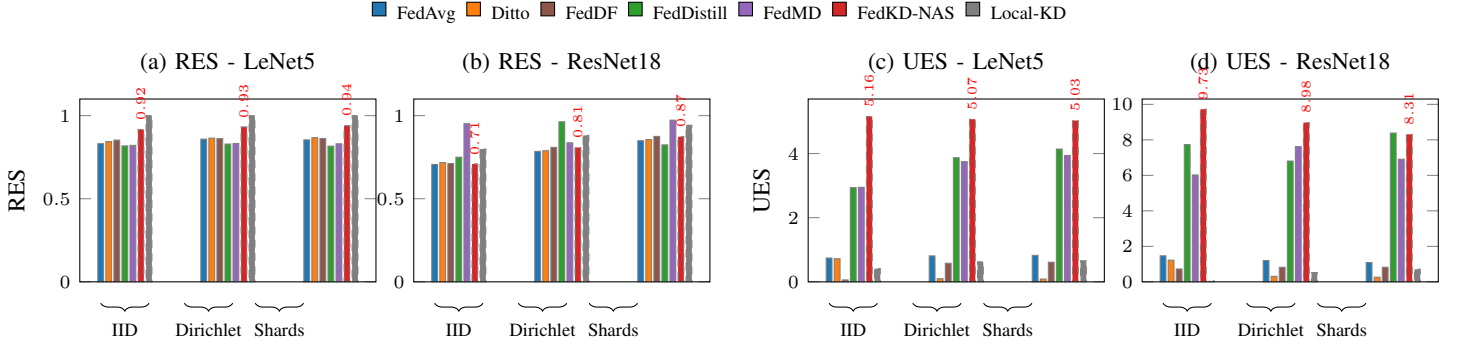

\centering

%
%
}
\caption{RES ($\downarrow$) and UES ($\uparrow$) on MNIST. On LeNet5, FedDistill consistently achieves the lowest RES across all data distributions, while FedKD-NAS attains the highest UES in IID, Dirichlet, and Shards, driven by its superior PQS and high CES. On ResNet18, FedAvg achieves the lowest RES under IID and Dirichlet, while FedDistill attains the lowest RES under Shards. For UES, FedKD-NAS consistently achieves the highest values across all distributions, with the largest gains under IID and Dirichlet, while FedDistill is competitive under Shards.}
\label{fig:mnist_res_ues}
\end{figure*}
\begin{figure*}[htbp]
\centering

\begin{tabular}{cccccccccccccc}
  \raisebox{0.2ex}{\tikz\draw[fill=FedAvgColor,    draw=black](0,0)rectangle(0.18,0.18);} & \scriptsize FedAvg     &
  \raisebox{0.2ex}{\tikz\draw[fill=DittoColor,     draw=black](0,0)rectangle(0.18,0.18);} & \scriptsize Ditto      &
  \raisebox{0.2ex}{\tikz\draw[fill=FedDFColor,     draw=black](0,0)rectangle(0.18,0.18);} & \scriptsize FedDF      &
  \raisebox{0.2ex}{\tikz\draw[fill=FedDistillColor,draw=black](0,0)rectangle(0.18,0.18);} & \scriptsize FedDistill &
  \raisebox{0.2ex}{\tikz\draw[fill=FedMDColor,     draw=black](0,0)rectangle(0.18,0.18);} & \scriptsize FedMD      &
  \raisebox{0.2ex}{\tikz\draw[fill=FedKDNASColor,  draw=black](0,0)rectangle(0.18,0.18);} & \scriptsize FedKD-NAS  &
  \raisebox{0.2ex}{\tikz\draw[fill=LocalKDColor,   draw=black](0,0)rectangle(0.18,0.18);} & \scriptsize Local-KD   \\
\end{tabular}

\vspace{0.6em}

\makebox[0pt][c]{%
\begin{tikzpicture}
\begin{groupplot}[
    group style={
        group size     = 4 by 1,
        horizontal sep = 0.5cm,
    },
    width  = 5.5cm,
    height = 4.0cm,
    xmin=-1.4, xmax=7.8,
    xtick=\empty,
    tick align=inside,
    clip=false,
    yticklabel style  = {font=\scriptsize},
    label style       = {font=\small},
    title style       = {font=\small},
    enlarge x limits  = false,
    every axis plot/.append style={mark=none},
]
\nextgroupplot[
    title  = {(a) RES - LeNet5},
    ylabel = {RES},
    ymin=0, ymax=1.1,
    bar width=0.18,
]
\AddThreeRegimePlotShift{RES}{0}{\MobiIIDFm}{\MobiDirFm}{\MobiShaFm}{FedAvgStyle}    {-0.75}{2.45}{5.65}
\AddThreeRegimePlotShift{RES}{1}{\MobiIIDFm}{\MobiDirFm}{\MobiShaFm}{DittoStyle}     {-0.50}{2.70}{5.90}
\AddThreeRegimePlotShift{RES}{2}{\MobiIIDFm}{\MobiDirFm}{\MobiShaFm}{FedDFStyle}     {-0.25}{2.95}{6.15}
\AddThreeRegimePlotShift{RES}{3}{\MobiIIDFm}{\MobiDirFm}{\MobiShaFm}{FedDistillStyle}{ 0.00}{3.20}{6.40}
\AddThreeRegimePlotShift{RES}{4}{\MobiIIDFm}{\MobiDirFm}{\MobiShaFm}{FedMDStyle}     { 0.25}{3.45}{6.65}
\AddThreeRegimePlotShift{RES}{5}{\MobiIIDFm}{\MobiDirFm}{\MobiShaFm}{FedKDNASStyle}  { 0.50}{3.70}{6.90}
\AddThreeRegimePlotShift{RES}{6}{\MobiIIDFm}{\MobiDirFm}{\MobiShaFm}{LocalKDStyle}   { 0.75}{3.95}{7.15}
\DrawRegimeBrackets

\nextgroupplot[
    title  = {(b) RES - ResNet18},
    ymin=0, ymax=1.1,
    bar width=0.18,
]
\AddThreeRegimePlotShift{RES}{0}{\ShufIIDFm}{\ShufDirFm}{\ShufShaFm}{FedAvgStyle}    {-0.75}{2.45}{5.65}
\AddThreeRegimePlotShift{RES}{1}{\ShufIIDFm}{\ShufDirFm}{\ShufShaFm}{DittoStyle}     {-0.50}{2.70}{5.90}
\AddThreeRegimePlotShift{RES}{2}{\ShufIIDFm}{\ShufDirFm}{\ShufShaFm}{FedDFStyle}     {-0.25}{2.95}{6.15}
\AddThreeRegimePlotShift{RES}{3}{\ShufIIDFm}{\ShufDirFm}{\ShufShaFm}{FedDistillStyle}{ 0.00}{3.20}{6.40}
\AddThreeRegimePlotShift{RES}{4}{\ShufIIDFm}{\ShufDirFm}{\ShufShaFm}{FedMDStyle}     { 0.25}{3.45}{6.65}
\AddThreeRegimePlotShift{RES}{5}{\ShufIIDFm}{\ShufDirFm}{\ShufShaFm}{FedKDNASStyle}  { 0.50}{3.70}{6.90}
\AddThreeRegimePlotShift{RES}{6}{\ShufIIDFm}{\ShufDirFm}{\ShufShaFm}{LocalKDStyle}   { 0.75}{3.95}{7.15}
\DrawRegimeBrackets

\nextgroupplot[
    title  = {(c) UES - LeNet5},
    ylabel = {UES},
    ymin=0, ymax=5.6,
    xshift=0.8cm,
    bar width=0.18,
]
\AddThreeRegimePlotShift{UES}{0}{\MobiIIDFm}{\MobiDirFm}{\MobiShaFm}{FedAvgStyle}    {-0.75}{2.45}{5.65}
\AddThreeRegimePlotShift{UES}{1}{\MobiIIDFm}{\MobiDirFm}{\MobiShaFm}{DittoStyle}     {-0.50}{2.70}{5.90}
\AddThreeRegimePlotShift{UES}{2}{\MobiIIDFm}{\MobiDirFm}{\MobiShaFm}{FedDFStyle}     {-0.25}{2.95}{6.15}
\AddThreeRegimePlotShift{UES}{3}{\MobiIIDFm}{\MobiDirFm}{\MobiShaFm}{FedDistillStyle}{ 0.00}{3.20}{6.40}
\AddThreeRegimePlotShift{UES}{4}{\MobiIIDFm}{\MobiDirFm}{\MobiShaFm}{FedMDStyle}     { 0.25}{3.45}{6.65}
\AddThreeRegimePlotShift{UES}{5}{\MobiIIDFm}{\MobiDirFm}{\MobiShaFm}{FedKDNASStyle}  { 0.50}{3.70}{6.90}
\AddThreeRegimePlotShift{UES}{6}{\MobiIIDFm}{\MobiDirFm}{\MobiShaFm}{LocalKDStyle}   { 0.75}{3.95}{7.15}
\DrawRegimeBrackets

\nextgroupplot[
    title  = {(d) UES - ResNet18},
    ymin=0, ymax=10.1,
    bar width=0.18,
]
\AddThreeRegimePlotShift{UES}{0}{\ShufIIDFm}{\ShufDirFm}{\ShufShaFm}{FedAvgStyle}    {-0.75}{2.45}{5.65}
\AddThreeRegimePlotShift{UES}{1}{\ShufIIDFm}{\ShufDirFm}{\ShufShaFm}{DittoStyle}     {-0.50}{2.70}{5.90}
\AddThreeRegimePlotShift{UES}{2}{\ShufIIDFm}{\ShufDirFm}{\ShufShaFm}{FedDFStyle}     {-0.25}{2.95}{6.15}
\AddThreeRegimePlotShift{UES}{3}{\ShufIIDFm}{\ShufDirFm}{\ShufShaFm}{FedDistillStyle}{ 0.00}{3.20}{6.40}
\AddThreeRegimePlotShift{UES}{4}{\ShufIIDFm}{\ShufDirFm}{\ShufShaFm}{FedMDStyle}     { 0.25}{3.45}{6.65}
\AddThreeRegimePlotShift{UES}{5}{\ShufIIDFm}{\ShufDirFm}{\ShufShaFm}{FedKDNASStyle}  { 0.50}{3.70}{6.90}
\AddThreeRegimePlotShift{UES}{6}{\ShufIIDFm}{\ShufDirFm}{\ShufShaFm}{LocalKDStyle}   { 0.75}{3.95}{7.15}
\DrawRegimeBrackets

\end{groupplot}
\end{tikzpicture}%
}
\caption{RES ($\downarrow$) and UES ($\uparrow$) on FMNIST. On LeNet5, FedMD achieves the lowest RES under IID, while FedDistill attains the lowest RES under Dirichlet and Shards; FedKD-NAS achieves the highest UES under IID, whereas FedMD and FedDistill lead under Dirichlet and Shards, respectively. On ResNet18, the lowest RES is achieved by FedAvg under IID, FedKD-NAS under Dirichlet, and Ditto under Shards. For UES on ResNet18, FedDistill is best under IID and Dirichlet, while FedKD-NAS achieves the highest UES under Shards.}
\label{fig:fmnist_res_ues}
\end{figure*}
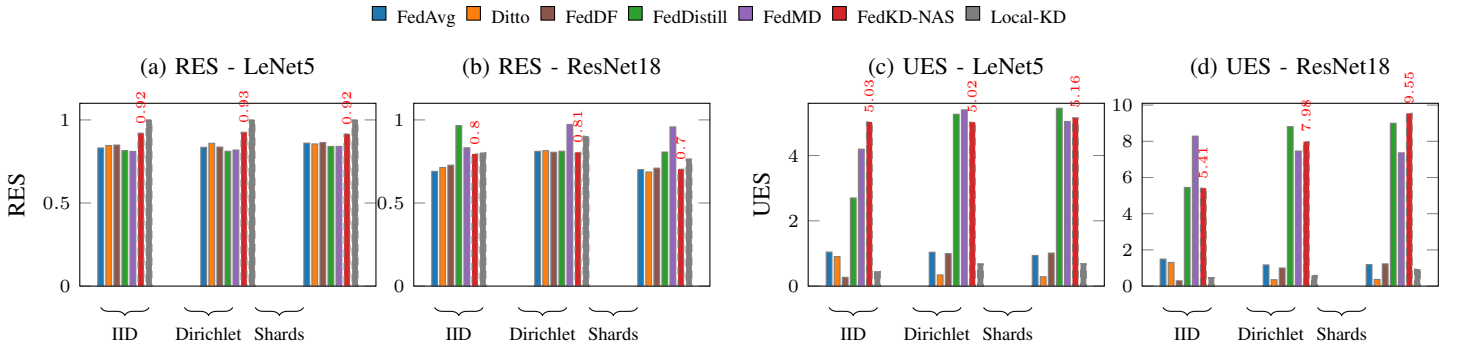
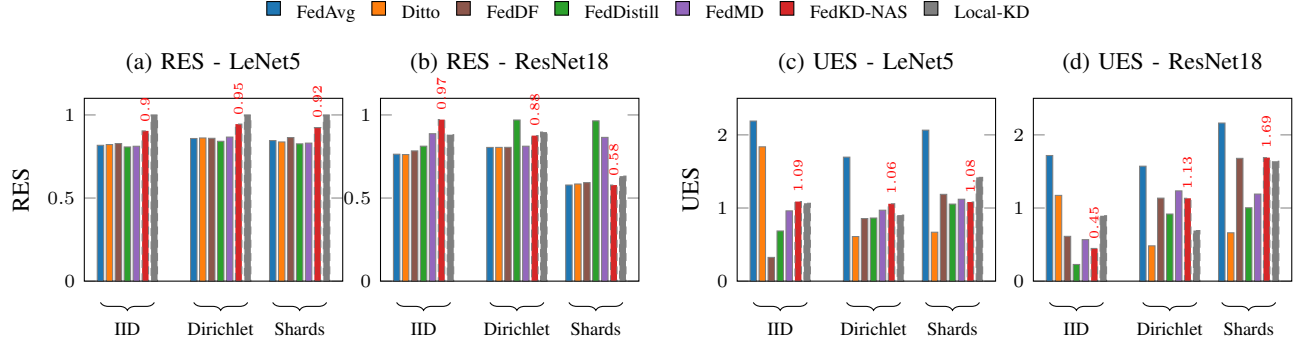
\begin{figure*}[htbp]
\centering

\setlength{\tabcolsep}{3pt}
\begin{tabular}{@{}cccccccccccccc@{}}
  \raisebox{0.2ex}{\tikz\draw[fill=FedAvgColor,draw=black](0,0)rectangle(0.18,0.18);} & \scriptsize FedAvg &
  \raisebox{0.2ex}{\tikz\draw[fill=DittoColor,draw=black](0,0)rectangle(0.18,0.18);} & \scriptsize Ditto &
  \raisebox{0.2ex}{\tikz\draw[fill=FedDFColor,draw=black](0,0)rectangle(0.18,0.18);} & \scriptsize FedDF &
  \raisebox{0.2ex}{\tikz\draw[fill=FedDistillColor,draw=black](0,0)rectangle(0.18,0.18);} & \scriptsize FedDistill &
  \raisebox{0.2ex}{\tikz\draw[fill=FedMDColor,draw=black](0,0)rectangle(0.18,0.18);} & \scriptsize FedMD &
  \raisebox{0.2ex}{\tikz\draw[fill=FedKDNASColor,draw=black](0,0)rectangle(0.18,0.18);} & \scriptsize FedKD-NAS &
  \raisebox{0.2ex}{\tikz\draw[fill=LocalKDColor,draw=black](0,0)rectangle(0.18,0.18);} & \scriptsize Local-KD \\
\end{tabular}

\vspace{0.6em}

\makebox[0pt][c]{%
\begin{tikzpicture}
\begin{groupplot}[
    group style={group size=4 by 1, horizontal sep=0.5cm},
    width=5cm,
    height=4.0cm,
    xmin=-1.2, xmax=6.0,
    xtick=\empty,
    tick align=inside,
    clip=false,
    yticklabel style={font=\scriptsize},
    label style={font=\small},
    title style={font=\small},
    enlarge x limits=false,
    every axis plot/.append style={mark=none},
]
\nextgroupplot[title={(a) RES - LeNet5}, ylabel={RES}, ymin=0, ymax=1.1, bar width=0.18]
\AddThreeRegimePlotShift{RES}{0}{\MobiIIDEm}{\MobiDirEm}{\MobiShaEm}{FedAvgStyle}{-0.75}{1.85}{4.05}
\AddThreeRegimePlotShift{RES}{1}{\MobiIIDEm}{\MobiDirEm}{\MobiShaEm}{DittoStyle}{-0.50}{2.10}{4.30}
\AddThreeRegimePlotShift{RES}{2}{\MobiIIDEm}{\MobiDirEm}{\MobiShaEm}{FedDFStyle}{-0.25}{2.35}{4.55}
\AddThreeRegimePlotShift{RES}{3}{\MobiIIDEm}{\MobiDirEm}{\MobiShaEm}{FedDistillStyle}{0.00}{2.60}{4.80}
\AddThreeRegimePlotShift{RES}{4}{\MobiIIDEm}{\MobiDirEm}{\MobiShaEm}{FedMDStyle}{0.25}{2.85}{5.05}
\AddThreeRegimePlotShift{RES}{5}{\MobiIIDEm}{\MobiDirEm}{\MobiShaEm}{FedKDNASStyle}{0.50}{3.10}{5.30}
\AddThreeRegimePlotShift{RES}{6}{\MobiIIDEm}{\MobiDirEm}{\MobiShaEm}{LocalKDStyle}{0.75}{3.35}{5.55}
\DrawRegimeBrackets

\nextgroupplot[title={(b) RES - ResNet18}, ymin=0, ymax=1.1, bar width=0.18]
\AddThreeRegimePlotShift{RES}{0}{\ShufIIDEm}{\ShufDirEm}{\ShufShaEm}{FedAvgStyle}{-0.75}{1.85}{4.05}
\AddThreeRegimePlotShift{RES}{1}{\ShufIIDEm}{\ShufDirEm}{\ShufShaEm}{DittoStyle}{-0.50}{2.10}{4.30}
\AddThreeRegimePlotShift{RES}{2}{\ShufIIDEm}{\ShufDirEm}{\ShufShaEm}{FedDFStyle}{-0.25}{2.35}{4.55}
\AddThreeRegimePlotShift{RES}{3}{\ShufIIDEm}{\ShufDirEm}{\ShufShaEm}{FedDistillStyle}{0.00}{2.60}{4.80}
\AddThreeRegimePlotShift{RES}{4}{\ShufIIDEm}{\ShufDirEm}{\ShufShaEm}{FedMDStyle}{0.25}{2.85}{5.05}
\AddThreeRegimePlotShift{RES}{5}{\ShufIIDEm}{\ShufDirEm}{\ShufShaEm}{FedKDNASStyle}{0.50}{3.10}{5.30}
\AddThreeRegimePlotShift{RES}{6}{\ShufIIDEm}{\ShufDirEm}{\ShufShaEm}{LocalKDStyle}{0.75}{3.35}{5.55}
\DrawRegimeBrackets

\nextgroupplot[title={(c) UES - LeNet5}, ylabel={UES}, ymin=0, ymax=2.5, xshift=0.8cm, bar width=0.18]
\AddThreeRegimePlotShift{UES}{0}{\MobiIIDEm}{\MobiDirEm}{\MobiShaEm}{FedAvgStyle}{-0.75}{1.85}{4.05}
\AddThreeRegimePlotShift{UES}{1}{\MobiIIDEm}{\MobiDirEm}{\MobiShaEm}{DittoStyle}{-0.50}{2.10}{4.30}
\AddThreeRegimePlotShift{UES}{2}{\MobiIIDEm}{\MobiDirEm}{\MobiShaEm}{FedDFStyle}{-0.25}{2.35}{4.55}
\AddThreeRegimePlotShift{UES}{3}{\MobiIIDEm}{\MobiDirEm}{\MobiShaEm}{FedDistillStyle}{0.00}{2.60}{4.80}
\AddThreeRegimePlotShift{UES}{4}{\MobiIIDEm}{\MobiDirEm}{\MobiShaEm}{FedMDStyle}{0.25}{2.85}{5.05}
\AddThreeRegimePlotShift{UES}{5}{\MobiIIDEm}{\MobiDirEm}{\MobiShaEm}{FedKDNASStyle}{0.50}{3.10}{5.30}
\AddThreeRegimePlotShift{UES}{6}{\MobiIIDEm}{\MobiDirEm}{\MobiShaEm}{LocalKDStyle}{0.75}{3.35}{5.55}
\DrawRegimeBrackets

\nextgroupplot[title={(d) UES - ResNet18}, ymin=0, ymax=2.5, bar width=0.18]
\AddThreeRegimePlotShift{UES}{0}{\ShufIIDEm}{\ShufDirEm}{\ShufShaEm}{FedAvgStyle}{-0.75}{1.85}{4.05}
\AddThreeRegimePlotShift{UES}{1}{\ShufIIDEm}{\ShufDirEm}{\ShufShaEm}{DittoStyle}{-0.50}{2.10}{4.30}
\AddThreeRegimePlotShift{UES}{2}{\ShufIIDEm}{\ShufDirEm}{\ShufShaEm}{FedDFStyle}{-0.25}{2.35}{4.55}
\AddThreeRegimePlotShift{UES}{3}{\ShufIIDEm}{\ShufDirEm}{\ShufShaEm}{FedDistillStyle}{0.00}{2.60}{4.80}
\AddThreeRegimePlotShift{UES}{4}{\ShufIIDEm}{\ShufDirEm}{\ShufShaEm}{FedMDStyle}{0.25}{2.85}{5.05}
\AddThreeRegimePlotShift{UES}{5}{\ShufIIDEm}{\ShufDirEm}{\ShufShaEm}{FedKDNASStyle}{0.50}{3.10}{5.30}
\AddThreeRegimePlotShift{UES}{6}{\ShufIIDEm}{\ShufDirEm}{\ShufShaEm}{LocalKDStyle}{0.75}{3.35}{5.55}
\DrawRegimeBrackets
\end{groupplot}
\end{tikzpicture}%
}
\caption{RES ($\downarrow$) and UES ($\uparrow$) on EMNIST. On LeNet5, FedDistill consistently achieves the lowest RES across all data distributions, while FedAvg attains the highest UES because of its strong accuracy and higher CES. On ResNet18, Ditto achieves the lowest RES under IID, FedAvg under Dirichlet, and FedKD-NAS under Shards. For UES, FedAvg consistently achieves the highest values across all distributions and both architectures.}
\label{fig:emnist_res_ues}
\end{figure*}

\begin{figure*}[t]
  \centering
    \begin{subfigure}[t]{0.98\textwidth}
    \centering
    \includegraphics[width=\textwidth]{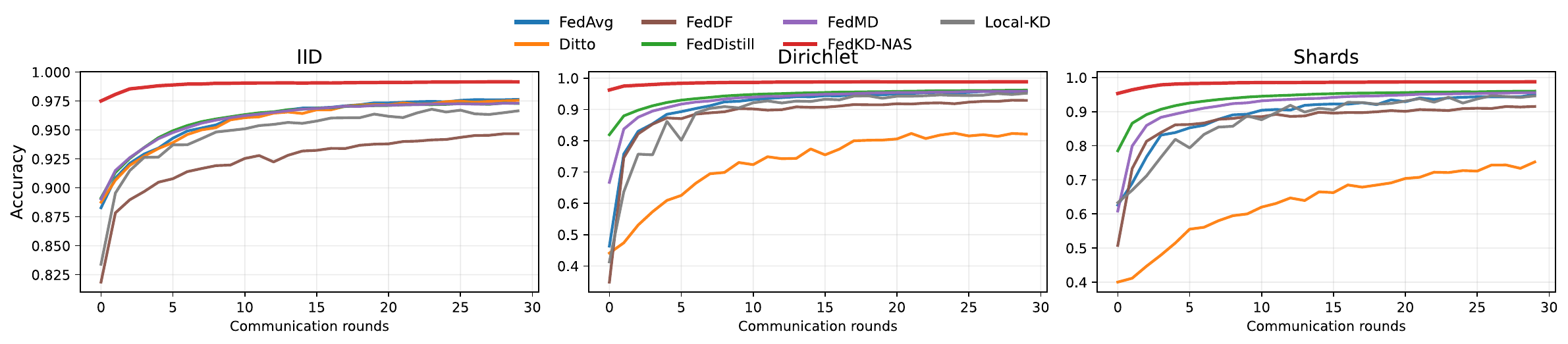}
    \caption{LeNet5}
  \end{subfigure}

  \vspace{0.5em}

 \begin{subfigure}[t]{0.98\textwidth}
    \centering
    \begin{adjustbox}{clip,trim=0 0 0 20pt}
      \includegraphics[width=\textwidth]{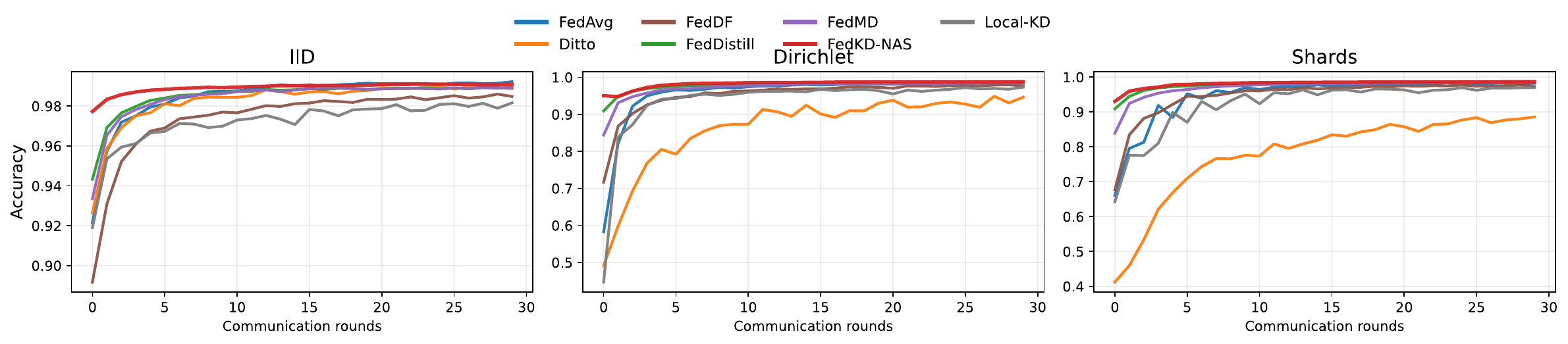}
    \end{adjustbox}
    \caption{ResNet18}
  \end{subfigure}
    \caption{Accuracy curves over 100 rounds on MNIST under IID, Dirichlet ($\alpha = 0.1$), and Shards. \emph{Top:} LeNet5. \emph{Bottom:} ResNet18. FedKD-NAS (red) achieves the highest final accuracy under all distributions with LeNet5. Under IID with ResNet18, FedAvg leads at round 100, but FedKD-NAS overtakes it under non-IID conditions.}
  \label{fig:mnist_acc_curves}
\end{figure*}
\begin{figure*}[t]
\centering
\begin{subfigure}[t]{0.98\textwidth}
    \centering
    \includegraphics[width=\textwidth]{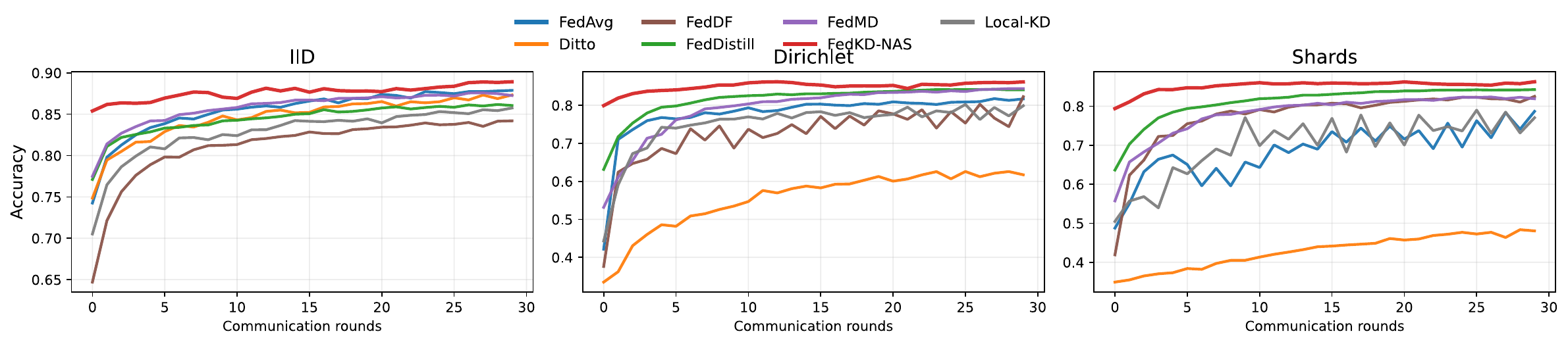}
    \caption{LeNet5}
  \end{subfigure}

  \vspace{0.5em}

 \begin{subfigure}[t]{0.98\textwidth}
    \centering
    \begin{adjustbox}{clip,trim=0 0 0 20pt}
      \includegraphics[width=\textwidth]{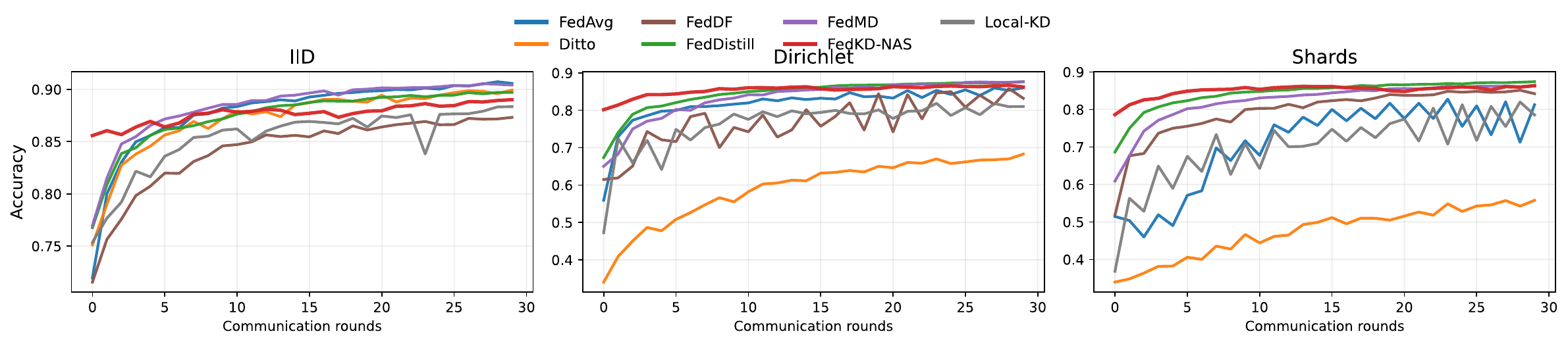}
    \end{adjustbox}
    \caption{ResNet18}
  \end{subfigure}
\caption{Accuracy curves over 100 rounds on FMNIST. FedKD-NAS leads under all distributions with LeNet5. On ResNet18, FedAvg and FedMD are competitive under IID and Dirichlet, respectively, but FedKD-NAS achieves the best accuracy under Shards on both architectures.}
\label{fig:fmnist_acc_curves}
\end{figure*}

\begin{figure*}[htbp]
\centering
\begin{subfigure}[t]{0.98\textwidth}
    \centering
    \includegraphics[width=\textwidth]{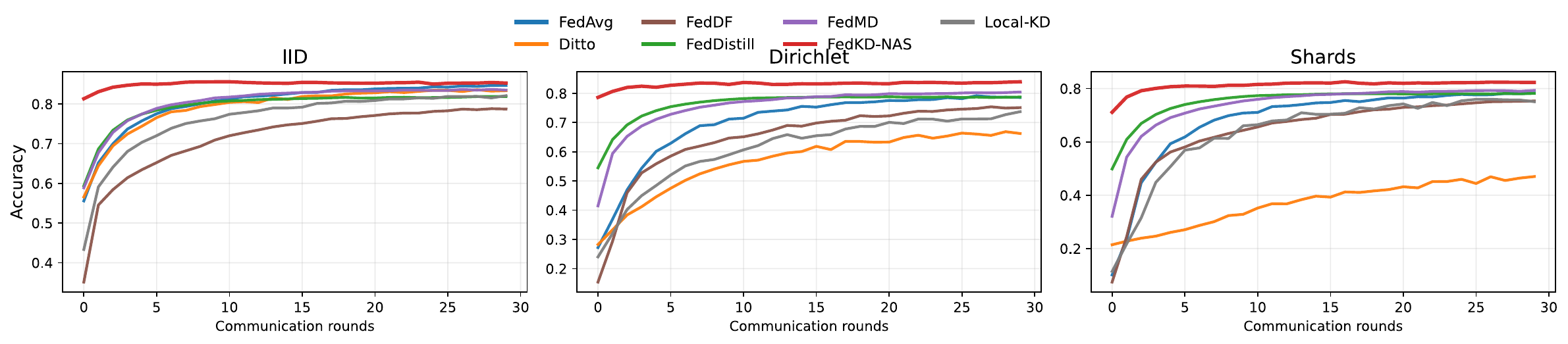}
    \caption{LeNet5}
  \end{subfigure}

  \vspace{0.5em}

 \begin{subfigure}[t]{0.98\textwidth}
    \centering
    \begin{adjustbox}{clip,trim=0 0 0 20pt}
      \includegraphics[width=\textwidth]{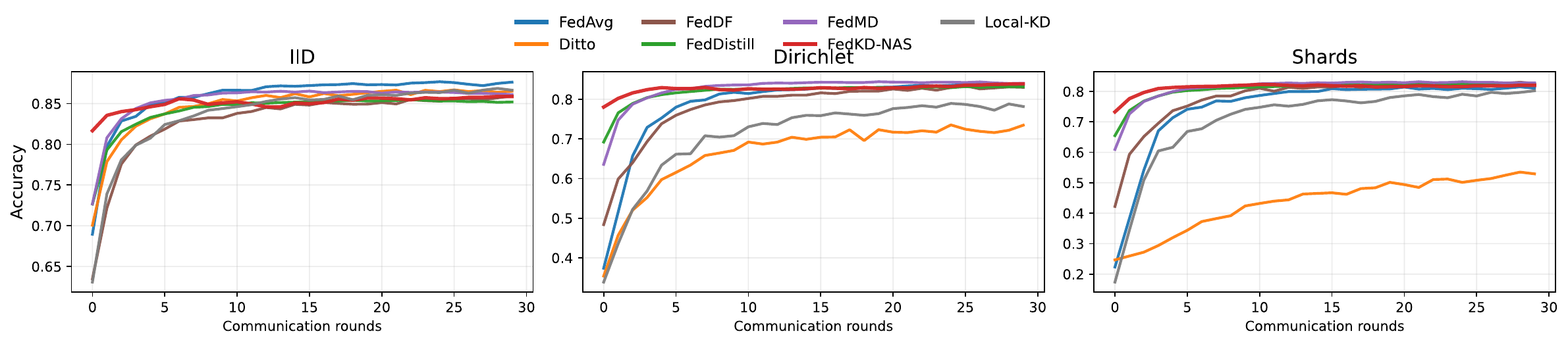}
    \end{adjustbox}
    \caption{ResNet18}
  \end{subfigure}
\caption{Accuracy curves over 100 rounds on EMNIST. EMNIST is a 47-class benchmark in which heterogeneity creates a significant accuracy spread across methods. FedKD-NAS maintains the highest accuracy under non-IID distributions on LeNet5. On ResNet18, FedAvg leads under IID, but FedKD-NAS and FedMD remain competitive under Shards.}
\label{fig:emnist_acc_curves}
\end{figure*}
\begin{figure*}[htbp]
\centering
\begin{subfigure}[t]{0.98\textwidth}
    \centering
    \includegraphics[width=\textwidth]{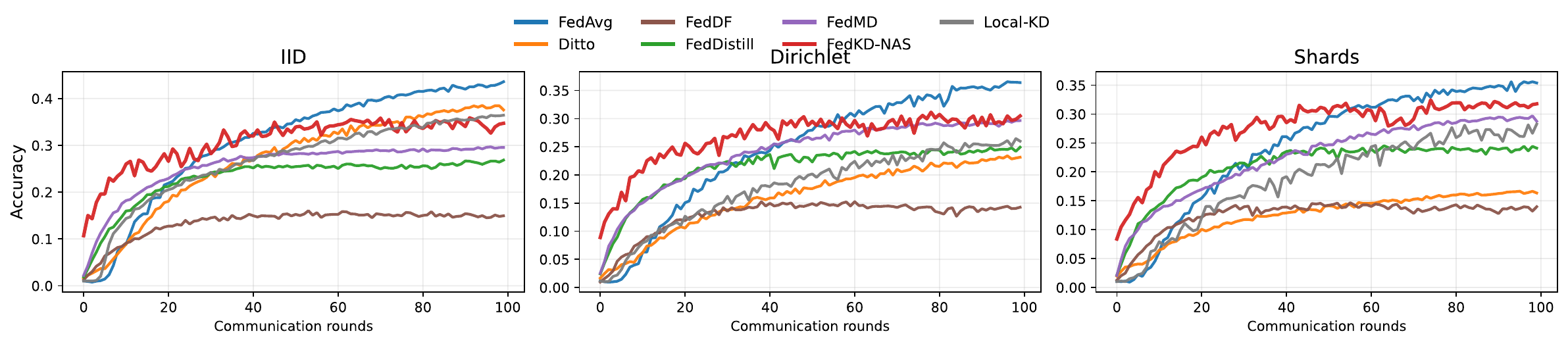}
    \caption{MobileNetV2}
  \end{subfigure}

  \vspace{0.5em}

 \begin{subfigure}[t]{0.98\textwidth}
    \centering
    \begin{adjustbox}{clip,trim=0 0 0 20pt}
      \includegraphics[width=\textwidth]{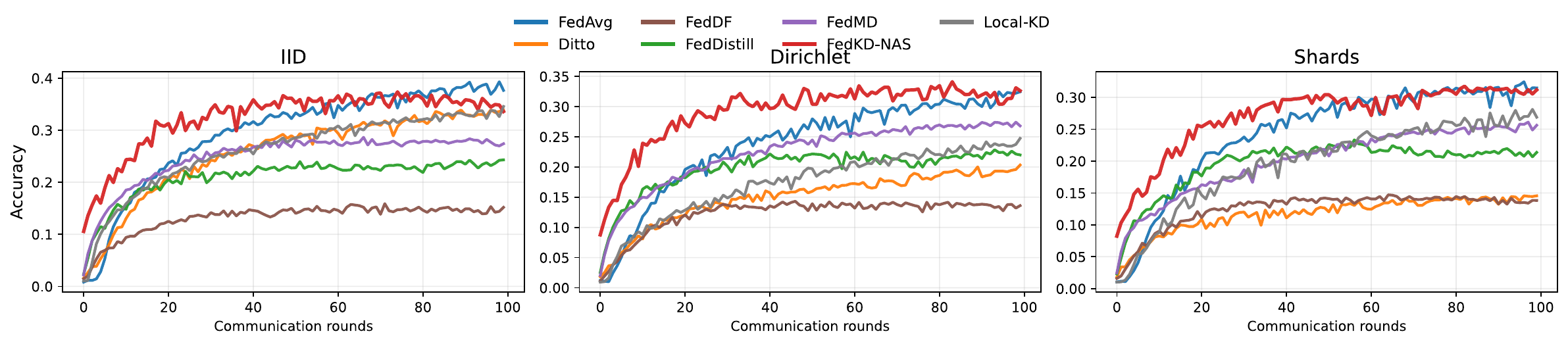}
    \end{adjustbox}
    \caption{ShuffleNetV2}
  \end{subfigure}
\caption{Accuracy curves over 100 rounds on CIFAR100. CIFAR100 is the most challenging benchmark because it contains 100 classes. FedAvg achieves the highest final accuracy under IID with both architectures, but FedKD-NAS converges more stably and shows smaller degradation under non-IID distributions.}
\label{fig:cifar100_acc_curves}
\end{figure*}

\begin{table*}[htbp]
\centering
\caption{Raw results across MNIST, FMNIST, EMNIST, and CIFAR100 under three data distributions (IID, Dirichlet, Shards): performance and resource usage metrics for LeNet5 and ResNet18 over 30 communication rounds for MNIST/FMNIST/EMNIST, MobileNetV2 and ShuffleNetV2 over 100 communication rounds for CIFAR100. Accuracy is higher-better ($\uparrow$); loss, CPU, RAM, and communication cost are lower-better ($\downarrow$).}
\label{tab:all_combined_raw}

\setlength{\tabcolsep}{2.0pt}
\renewcommand{\arraystretch}{0.82}
\small
\begin{adjustbox}{max width=\textwidth, max totalheight=0.9\textheight, keepaspectratio}
\begin{tabular}{lll*{15}{c}}
\toprule
\multirow{2}{*}{Dataset}
& \multirow{2}{*}{Model}
& \multirow{2}{*}{Method}
& \multicolumn{5}{c}{IID}
& \multicolumn{5}{c}{Dirichlet}
& \multicolumn{5}{c}{Shards} \\
\cmidrule(lr){4-8}\cmidrule(lr){9-13}\cmidrule(lr){14-18}
& & 
& Acc$\uparrow$ & Loss$\downarrow$ & CPU$\downarrow$ & RAM$\downarrow$ & Comm$\downarrow$
& Acc$\uparrow$ & Loss$\downarrow$ & CPU$\downarrow$ & RAM$\downarrow$ & Comm$\downarrow$
& Acc$\uparrow$ & Loss$\downarrow$ & CPU$\downarrow$ & RAM$\downarrow$ & Comm$\downarrow$ \\
\midrule

\multirow{14}{*}{\textbf{MNIST}}
& \multirow{7}{*}{LeNet5}
& FedAvg
& 0.9760 & 0.0731 & 43.3281 & \textbf{605.5195} & 8\,618\,304
& 0.9572 & 0.1400 & 43.4335 & 608.8672 & 8\,618\,304
& 0.9507 & 0.1668 & 43.2207 & 603.5156 & 8\,618\,304 \\
& & Ditto
& 0.9755 & 0.0735 & 43.7189 & 619.1250 & 8\,618\,304
& 0.8212 & 0.1473 & 43.3855 & 619.1406 & 8\,618\,304
& 0.7528 & 0.1762 & 43.3826 & 621.1250 & 8\,618\,304 \\
& & FedDF
& 0.9467 & 0.1716 & 43.6166 & 635.0313 & 9\,099\,584
& 0.9289 & 0.2625 & 43.3273 & 615.4023 & 9\,099\,584
& 0.9149 & 0.2955 & 42.6526 & 624.7773 & 9\,099\,584 \\
& & FedDistill
& 0.9729 & 0.0950 & \textbf{43.1896} & 586.4063 & \textbf{1\,925\,120}
& 0.9613 & 0.1343 & 43.0710 & \textbf{569.0898} & \textbf{1\,925\,120}
& 0.9601 & \textbf{0.1336} & \textbf{40.7937} & \textbf{585.0000} & \textbf{1\,925\,120} \\
& & FedMD
& 0.9730 & 0.0929 & 43.2285 & 590.7354 & \textbf{1\,925\,120}
& 0.9581 & 0.1421 & 43.1042 & 573.6250 & \textbf{1\,925\,120}
& 0.9555 & 0.1482 & 43.1458 & \textbf{568.4648} & \textbf{1\,925\,120} \\
& & Local-KD
& 0.9681 & 0.1719 & 47.7265 & 801.9434 & 8\,618\,304
& 0.9496 & 0.2372 & 46.6961 & 773.4072 & 8\,618\,304
& 0.9416 & 0.2261 & 47.1313 & 762.4521 & 8\,618\,304 \\
& & FedKD-NAS
& \textbf{0.9906} & \textbf{0.0288} & 43.9906 & 731.0547 & \textbf{1\,925\,120}
& \textbf{0.9871} & \textbf{0.0391} & 43.6538 & 718.7305 & \textbf{1\,925\,120}
& \textbf{0.9868} & 0.0424 & 43.9561 & 721.7002 & \textbf{1\,925\,120} \\

\cmidrule(lr){2-18}

& \multirow{7}{*}{ResNet18}
& FedAvg
& \textbf{0.9914} & \textbf{0.0264} & 44.5184 & \textbf{660.1582} & 13\,492\,544
& 0.9847 & 0.0460 & 44.6725 & \textbf{640.3623} & 13\,492\,544
& 0.9827 & 0.0553 & 45.1473 & \textbf{635.5342} & 13\,492\,544 \\
& & Ditto
& 0.9888 & 0.0265 & 44.7035 & 685.6455 & 13\,492\,544
& 0.9465 & 0.0482 & 44.0447 & 661.5205 & 13\,492\,544
& 0.8875 & 0.0567 & 43.9930 & 668.9219 & 13\,492\,544 \\
& & FedDF
& 0.9851 & 0.0433 & 44.7817 & 667.5684 & 13\,973\,824
& 0.9761 & 0.0759 & 45.2540 & 676.7383 & 13\,973\,824
& 0.9712 & 0.0944 & 45.3302 & 677.0088 & 13\,973\,824 \\
& & FedDistill
& 0.9893 & 0.0360 & 44.0701 & 792.8594 & \textbf{1\,925\,120}
& 0.9851 & 0.0487 & 44.0502 & 1015.0391 & \textbf{1\,925\,120}
& 0.9855 & 0.0472 & 44.0880 & 612.1982 & \textbf{1\,925\,120} \\
& & FedMD
& 0.9891 & 0.0357 & 43.1523 & 1372.1016 & \textbf{1\,925\,120}
& 0.9848 & 0.0535 & 44.3172 & 752.1445 & \textbf{1\,925\,120}
& 0.9849 & 0.0514 & 44.7044 & 855.3242 & \textbf{1\,925\,120} \\
& & Local-KD
& 0.9794 & 0.1805 & 47.7973 & 821.1201 & 13\,492\,544
& 0.9689 & 0.2056 & 47.5593 & 773.7686 & 13\,492\,544
& 0.9688 & 0.1673 & 47.2858 & 756.7031 & 13\,492\,544 \\
& & FedKD-NAS
& 0.9908 & 0.0277 & \textbf{42.0451} & 739.9141 & \textbf{1\,925\,120}
& \textbf{0.9869} & \textbf{0.0382} & \textbf{43.7023} & 709.0547 & \textbf{1\,925\,120}
& \textbf{0.9869} & \textbf{0.0411} & \textbf{43.4466} & 708.1055 & \textbf{1\,925\,120} \\

\midrule

\multirow{14}{*}{\textbf{FMNIST}}
& \multirow{7}{*}{LeNet5}
& FedAvg
& 0.8799 & \textbf{0.3360} & 43.4961 & 606.0078 & 8\,618\,304
& 0.8143 & 0.4946 & 42.4676 & 608.6680 & 8\,618\,304
& 0.7797 & 0.6261 & 43.5863 & 606.6172 & 8\,618\,304 \\
& & Ditto
& 0.8728 & 0.3369 & 44.1388 & 619.1250 & 8\,618\,304
& 0.6193 & 0.4773 & 43.9290 & 622.6719 & 8\,618\,304
& 0.4826 & 0.5856 & \textbf{42.2044} & 621.7266 & 8\,618\,304 \\
& & FedDF
& 0.8411 & 0.4404 & 43.3163 & 637.0703 & 9\,099\,584
& 0.8225 & 0.5167 & \textbf{40.3196} & 646.8320 & 9\,099\,584
& 0.8249 & 0.5622 & 43.2100 & 618.4180 & 9\,099\,584 \\
& & FedDistill
& 0.8584 & 0.4309 & 43.2469 & 584.7461 & \textbf{1\,925\,120}
& 0.8382 & 0.4679 & 42.5425 & \textbf{569.7227} & \textbf{1\,925\,120}
& 0.8444 & \textbf{0.4533} & 43.2056 & \textbf{581.7617} & \textbf{1\,925\,120} \\
& & FedMD
& 0.8737 & 0.3580 & \textbf{43.1240} & \textbf{579.0479} & \textbf{1\,925\,120}
& 0.8435 & \textbf{0.4426} & 43.2162 & 569.9258 & \textbf{1\,925\,120}
& 0.8213 & 0.5010 & 43.1608 & 583.1875 & \textbf{1\,925\,120} \\
& & Local-KD
& 0.8576 & 0.4741 & 47.9304 & 802.6133 & 8\,618\,304
& 0.7903 & 0.8212 & 47.0624 & 792.3389 & 8\,618\,304
& 0.7441 & 0.8116 & 46.8584 & 767.1650 & 8\,618\,304 \\
& & FedKD-NAS
& \textbf{0.8933} & 0.3606 & 44.4097 & 734.1006 & \textbf{1\,925\,120}
& \textbf{0.8664} & 0.4682 & 43.0446 & 742.1797 & \textbf{1\,925\,120}
& \textbf{0.8616} & 0.4556 & 43.3393 & 695.0117 & \textbf{1\,925\,120} \\

\cmidrule(lr){2-18}

& \multirow{7}{*}{ResNet18}
& FedAvg
& \textbf{0.9055} & 0.2578 & \textbf{43.0534} & \textbf{641.4414} & 13\,492\,544
& 0.8608 & 0.3817 & 45.4318 & \textbf{656.1436} & 13\,492\,544
& 0.8135 & 0.5095 & 45.0078 & \textbf{650.2842} & 13\,492\,544 \\
& & Ditto
& 0.9008 & \textbf{0.2541} & 44.3868 & 666.1387 & 13\,492\,544
& 0.6821 & 0.3788 & 44.7068 & 678.5410 & 13\,492\,544
& 0.5596 & 0.4958 & 43.4206 & 657.4619 & 13\,492\,544 \\
& & FedDF
& 0.8722 & 0.3588 & 45.4739 & 673.7041 & 13\,973\,824
& 0.8431 & 0.4615 & 43.9577 & 675.8506 & 13\,973\,824
& 0.8427 & 0.5043 & 44.5871 & 688.7148 & 13\,973\,824 \\
& & FedDistill
& 0.8962 & 0.3431 & 44.6656 & 1330.6758 & \textbf{1\,925\,120}
& 0.8770 & 0.3673 & 44.8635 & 669.1680 & \textbf{1\,925\,120}
& \textbf{0.8735} & \textbf{0.4038} & 44.0357 & 987.8984 & \textbf{1\,925\,120} \\
& & FedMD
& 0.9048 & 0.2880 & 44.6696 & 974.1602 & \textbf{1\,925\,120}
& \textbf{0.8777} & \textbf{0.3525} & 44.8709 & 986.7305 & \textbf{1\,925\,120}
& 0.8627 & 0.4096 & \textbf{43.1560} & 1454.8672 & \textbf{1\,925\,120} \\
& & Local-KD
& 0.8814 & 0.4429 & 47.9144 & 809.5762 & 13\,492\,544
& 0.7883 & 0.7695 & 47.4576 & 794.4297 & 13\,492\,544
& 0.7912 & 0.6712 & 47.1259 & 770.8984 & 13\,492\,544 \\
& & FedKD-NAS
& 0.8908 & 0.3761 & 44.1912 & 889.0664 & \textbf{1\,925\,120}
& 0.8649 & 0.4567 & \textbf{42.4085} & 707.9326 & \textbf{1\,925\,120}
& 0.8606 & 0.4766 & 43.7047 & 698.2813 & \textbf{1\,925\,120} \\

\midrule

\multirow{14}{*}{\textbf{EMNIST}}
& \multirow{7}{*}{LeNet5}
& FedAvg
& 0.8468 & \textbf{0.4745} & 42.6728 & 652.0977 & \textbf{8\,922\,592}
& 0.7857 & 0.6529 & 43.7307 & 657.7305 & \textbf{8\,922\,592}
& 0.7831 & 0.6694 & 42.8195 & 655.1875 & \textbf{8\,922\,592} \\
& & Ditto
& 0.8334 & 0.4808 & 42.2911 & 667.2500 & \textbf{8\,922\,592}
& 0.6624 & 0.6526 & 43.3166 & 670.4258 & \textbf{8\,922\,592}
& 0.4699 & 0.6707 & \textbf{41.2970} & 668.7773 & \textbf{8\,922\,592} \\
& & FedDF
& 0.7872 & 0.6844 & 42.1660 & 680.6250 & 13\,181\,920
& 0.7510 & 0.8199 & 43.2494 & 668.1758 & 13\,181\,920
& 0.7534 & 0.8499 & 42.6477 & 689.0342 & 13\,181\,920 \\
& & FedDistill
& 0.8380 & 0.5298 & \textbf{41.7678} & 657.0859 & 13\,181\,920
& 0.8197 & 0.6260 & \textbf{42.8149} & \textbf{646.1416} & 13\,181\,920
& 0.8105 & 0.6879 & 41.7719 & 684.0078 & 13\,181\,920 \\
& & FedMD
& 0.8376 & 0.5281 & 42.0817 & \textbf{646.6758} & 13\,181\,920
& 0.8192 & 0.6262 & 43.0333 & 646.5859 & 13\,181\,920
& 0.8065 & 0.6931 & 42.4014 & \textbf{638.3789} & 13\,181\,920 \\
& & Local-KD
& 0.8444 & 0.5565 & 43.2037 & 673.4023 & \textbf{8\,922\,592}
& 0.7734 & 0.7825 & 43.1476 & 667.0234 & \textbf{8\,922\,592}
& 0.7874 & 0.7238 & 43.2311 & 661.0234 & \textbf{8\,922\,592} \\
& & FedKD-NAS
& \textbf{0.8526} & 0.4867 & 42.0546 & 652.1289 & 13\,181\,920
& \textbf{0.8396} & \textbf{0.5986} & 42.8504 & 651.5547 & 13\,181\,920
& \textbf{0.8231} & \textbf{0.6560} & 42.2949 & 654.3633 & 13\,181\,920 \\

\cmidrule(lr){2-18}

& \multirow{7}{*}{ResNet18}
& FedAvg
& \textbf{0.8763} & \textbf{0.3706} & 44.4480 & \textbf{696.5771} & \textbf{13\,645\,280}
& 0.8348 & \textbf{0.4922} & 43.9550 & \textbf{715.5703} & \textbf{13\,645\,280}
& 0.8107 & 0.5698 & 44.2024 & \textbf{705.5752} & \textbf{13\,645\,280} \\
& & Ditto
& 0.8652 & 0.3714 & \textbf{43.2755} & 721.0303 & \textbf{13\,645\,280}
& 0.7346 & 0.4987 & 43.7218 & 721.6416 & \textbf{13\,645\,280}
& 0.5293 & 0.5685 & 44.6558 & 714.8604 & \textbf{13\,645\,280} \\
& & FedDF
& 0.8594 & 0.4198 & 44.3327 & 746.9258 & 17\,904\,608
& 0.8302 & 0.5255 & \textbf{43.0420} & 738.4873 & 17\,904\,608
& 0.8262 & \textbf{0.5582} & 44.7857 & 759.8770 & 17\,904\,608 \\
& & FedDistill
& 0.8519 & 0.6311 & 44.2358 & 816.6895 & 17\,037\,312
& 0.8311 & 0.7115 & 43.1805 & 1095.8750 & 17\,037\,312
& 0.8235 & 0.7230 & 43.6043 & 3281.4414 & 17\,037\,312 \\
& & FedMD
& 0.8611 & 0.5152 & 44.8281 & 979.8066 & 17\,037\,312
& \textbf{0.8403} & 0.5821 & 43.7457 & 736.7539 & 17\,037\,312
& \textbf{0.8281} & 0.5976 & 44.0089 & 2600.5156 & 17\,037\,312 \\
& & Local-KD
& 0.8664 & 0.6044 & 47.2418 & 905.0605 & \textbf{13\,645\,280}
& 0.7817 & 0.9182 & 45.9689 & 869.8691 & \textbf{13\,645\,280}
& 0.8022 & 1.1139 & 46.9388 & 865.8916 & \textbf{13\,645\,280} \\
& & FedKD-NAS
& 0.8586 & 0.5039 & 44.5918 & 1187.3555 & 17\,037\,312
& 0.8388 & 0.5819 & 43.4029 & 884.5527 & 17\,037\,312
& 0.8202 & 0.6755 & \textbf{42.5365} & 820.6143 & 17\,037\,312 \\

\midrule

\multirow{14}{*}{\textbf{CIFAR100}}
& \multirow{7}{*}{MobileNetV2}
& FedAvg
& \textbf{0.4332} & \textbf{2.1744} & 46.8684 & 1095.1211 & 75\,263\,104
& \textbf{0.3727} & \textbf{2.3779} & 47.0038 & 1145.4902 & 75\,263\,104
& \textbf{0.3567} & 2.4574 & 47.1546 & 1080.7998 & 75\,263\,104 \\
& & Ditto
& 0.3930 & 2.1742 & 47.0402 & 1162.7871 & 75\,263\,104
& 0.2300 & 2.4760 & 47.1608 & 1285.4453 & 75\,263\,104
& 0.1619 & \textbf{2.4368} & 47.1846 & 1207.7324 & 75\,263\,104 \\
& & FedDF
& 0.1618 & 3.5913 & 47.1966 & 1122.4121 & 79\,307\,904
& 0.1345 & 3.9060 & 47.1475 & 1173.5918 & 79\,307\,904
& 0.1206 & 4.0082 & 47.0975 & 1090.9590 & 79\,307\,904 \\
& & FedDistill
& 0.2695 & 2.7560 & 47.0136 & 1046.8477 & \textbf{1\,617\,920}
& 0.2287 & 2.7414 & 47.1488 & 1123.7461 & \textbf{1\,617\,920}
& 0.2265 & 2.8167 & 46.9418 & 1064.4805 & \textbf{1\,617\,920} \\
& & FedMD
& 0.3261 & 2.3841 & 46.7092 & 1055.6328 & \textbf{1\,617\,920}
& 0.2572 & 2.5601 & 46.7717 & 1112.7402 & \textbf{1\,617\,920}
& 0.2582 & 2.5828 & 46.9478 & 1047.8672 & \textbf{1\,617\,920} \\
& & Local-KD
& 0.2562 & 2.8606 & 46.9869 & 1178.0068 & 75\,263\,104
& 0.1144 & 4.6756 & 47.0660 & 1373.6992 & 75\,263\,104
& 0.1210 & 4.5290 & 47.1403 & 1211.3145 & 75\,263\,104 \\
& & FedKD-NAS
& 0.2710 & 2.7131 & \textbf{34.0793} & \textbf{957.1016} & \textbf{1\,617\,920}
& 0.2643 & 2.6500 & \textbf{35.5990} & \textbf{928.3594} & \textbf{1\,617\,920}
& 0.2334 & 2.7392 & \textbf{34.8734} & \textbf{957.5625} & \textbf{1\,617\,920} \\

\cmidrule(lr){2-18}

& \multirow{7}{*}{ShuffleNetV2}
& FedAvg
& \textbf{0.3815} & 2.4420 & 42.8852 & \textbf{829.6660} & 14\,217\,344
& \textbf{0.3284} & \textbf{2.6018} & 42.7763 & \textbf{845.8994} & 14\,217\,344
& \textbf{0.3143} & \textbf{2.6649} & 43.0894 & \textbf{831.2324} & 14\,217\,344 \\
& & Ditto
& 0.3306 & \textbf{2.3711} & \textbf{42.5637} & 888.8955 & 14\,217\,344
& 0.1952 & 2.6512 & \textbf{42.5654} & 899.9932 & 14\,217\,344
& 0.1477 & 2.7679 & \textbf{42.6513} & 887.4346 & 14\,217\,344 \\
& & FedDF
& 0.1457 & 3.7029 & 43.0044 & 848.3096 & 18\,262\,144
& 0.1392 & 3.8697 & 42.7728 & 861.2412 & 18\,262\,144
& 0.1375 & 3.8998 & 42.9304 & 844.3828 & 18\,262\,144 \\
& & FedDistill
& 0.2205 & 2.9419 & 42.8959 & 862.1084 & \textbf{1\,617\,920}
& 0.1998 & 2.9552 & 42.7038 & 884.7432 & \textbf{1\,617\,920}
& 0.1988 & 2.9825 & 42.8770 & 864.6064 & \textbf{1\,617\,920} \\
& & FedMD
& 0.2672 & 2.7210 & 42.7184 & 866.7246 & \textbf{1\,617\,920}
& 0.2234 & 2.8654 & 42.8040 & 863.6289 & \textbf{1\,617\,920}
& 0.2234 & 2.8728 & 42.8805 & 860.7461 & \textbf{1\,617\,920} \\
& & Local-KD
& 0.2157 & 2.9675 & 43.8116 & 989.7852 & 14\,217\,344
& 0.1002 & 4.9991 & 43.7808 & 1075.9805 & 14\,217\,344
& 0.1096 & 4.7778 & 43.9632 & 1014.0508 & 14\,217\,344 \\
& & FedKD-NAS
& 0.2279 & 2.9385 & 33.8331 & 906.9492 & \textbf{1\,617\,920}
& 0.2111 & 2.9354 & 35.0074 & 890.7393 & \textbf{1\,617\,920}
& 0.2024 & 2.9687 & 34.4866 & 902.2451 & \textbf{1\,617\,920} \\

\bottomrule
\end{tabular}
\end{adjustbox}

\end{table*}

\end{document}